\documentclass[11pt]{article}

\usepackage{arxiv}

\usepackage[utf8]{inputenc} 
\usepackage[T1]{fontenc}    
\usepackage{url}            
\usepackage{booktabs}       
\usepackage{amsfonts}       
\usepackage{nicefrac}       
\usepackage{microtype}      

\usepackage[colorlinks = true, citecolor = blue, bookmarks=false]{hyperref}
\usepackage[utf8]{inputenc}   
\usepackage[T1]{fontenc}      

\usepackage{comment}

\usepackage[tbtags]{amsmath}
\usepackage{amsthm}
\usepackage{bm}
\allowdisplaybreaks
\usepackage{amssymb,mathrsfs}
\usepackage{amsfonts}
\usepackage{upgreek}
\usepackage{geometry}
\usepackage[numbers]{natbib}
\usepackage{graphicx}
\usepackage{float}
\usepackage{color}

\usepackage{stmaryrd}

\usepackage[inline]{enumitem}
\usepackage{url}

\usepackage{graphicx} 
\usepackage{tikz}
\usepackage{wrapfig}
\usepackage{pgfplots}
\usepackage{xcolor}
\usepackage{bbm}

\usepackage{ifthen}
\usepackage{xargs}

\usepackage[textwidth=1.8cm]{todonotes}

\usepackage{aliascnt}
\usepackage{thmtools}
\usepackage{cleveref,nameref}
\usepackage{autonum}
\makeatletter
\newtheorem{theorem}{Theorem}
\crefname{theorem}{theorem}{Theorems}
\Crefname{Theorem}{Theorem}{Theorems}
\newaliascnt{llemma}{theorem}
\newtheorem{llemma}[theorem]{Lemma}
\aliascntresetthe{llemma}
\crefname{llemma}{lemma}{lemmas}
\Crefname{LLemma}{Lemma}{Lemmas}

\newaliascnt{pproposition}{theorem}
\newtheorem{pproposition}[theorem]{Proposition}
\aliascntresetthe{pproposition}
\crefname{pproposition}{proposition}{propositions}
\Crefname{Pproposition}{Proposition}{Propositions}

\newtheorem{example}[theorem]{Example}
\crefname{example}{example}{examples}
\Crefname{Example}{Example}{Examples}
\Crefname{assumption}{\textbf{H}\hspace{-3pt}}{\textbf{H}\hspace{-3pt}}
\crefname{assumption}{\textbf{H}}{\textbf{H}}

\newtheorem{assumptionPCA}{\textbf{HP}\hspace{-3pt}}
\Crefname{assumptionPCA}{\textbf{HP}\hspace{-3pt}}{\textbf{HP}\hspace{-3pt}}
\crefname{assumptionPCA}{\textbf{HP}}{\textbf{HP}}

\newtheorem{assumptionR}{\textbf{R}\hspace{-3pt}}
\Crefname{assumptionR}{\textbf{R}\hspace{-3pt}}{\textbf{R}\hspace{-3pt}}
\crefname{assumptionR}{\textbf{R}}{\textbf{R}}

\newtheorem{assumptionA}{\textbf{A}\hspace{-3pt}}
\Crefname{assumptionA}{\textbf{A}\hspace{-3pt}}{\textbf{A}\hspace{-3pt}}
\crefname{assumptionA}{\textbf{A}}{\textbf{A}}

\newtheorem{assumptionMD}{\textbf{MD}\hspace{-3pt}}
\Crefname{assumptionMD}{\textbf{MD}\hspace{-3pt}}{\textbf{MD}\hspace{-3pt}}
\crefname{assumptionMD}{\textbf{MD}}{\textbf{MD}}

\newtheorem{assumptionMA}{\textbf{MA}\hspace{-3pt}}
\Crefname{assumptionMA}{\textbf{MA}\hspace{-3pt}}{\textbf{MA}\hspace{-3pt}}
\crefname{assumptionMA}{\textbf{MA}}{\textbf{MA}}
\renewcommand{\bar}[1]{\overline{#1}}

\DeclareFontFamily{U}{matha}{\hyphenchar\font45}
\DeclareFontShape{U}{matha}{m}{n}{
      <5> <6> <7> <8> <9> <10> gen * matha
      <10.95> matha10 <12> <14.4> <17.28> <20.74> <24.88> matha12
      }{}
\DeclareSymbolFont{matha}{U}{matha}{m}{n}
\DeclareFontSubstitution{U}{matha}{m}{n}

\DeclareFontFamily{U}{mathx}{\hyphenchar\font45}
\DeclareFontShape{U}{mathx}{m}{n}{
      <5> <6> <7> <8> <9> <10>
      <10.95> <12> <14.4> <17.28> <20.74> <24.88>
      mathx10
      }{}
\DeclareSymbolFont{mathx}{U}{mathx}{m}{n}
\DeclareFontSubstitution{U}{mathx}{m}{n}

\DeclareMathDelimiter{\vvvert}{0}{matha}{"7E}{mathx}{"17}

\usepackage{cancel}

\newcommand{\calBLmu}{{\calBLmu}}
\newcommand{\calBLnu}{{\calBLnu}}
\newcommand{\calBLkappa}{{\calBLkappa}}

\newcommand{\ie}{\textit{i.e.}}

\newcommand{\calA}{\mathcal{A}}

\newcommand{\calF}{\mathcal{F}}

\newcommand{\bbE}{\mathbb{E}}

\newcommand{\rmD}{\mathrm{D}}

\newcommand{\rmO}{\mathrm{O}}

\newcommand{\rmT}{\mathrm{T}}

\newcommand{\rmI}{\mathrm{I}}

\newcommand{\spanD}{\mathrm{Span}}

\newcommand{\bfA}{\mathbf{A}}

\newcommand{\Qbf}{\mathbf{Q}}
\newcommand{\Rbf}{\mathbf{R}}

\newcommand{\bfU}{\mathbf{U}}
\newcommand{\bfV}{\mathbf{V}}

\newcommand{\bfSigma}{\boldsymbol{\Sigma}}
\newcommand{\Sigmabf}{\boldsymbol{\Sigma}}

\def\msc{\mathsf{C}}

\def\msg{\mathsf{G}}

\def\msx{\mathsf{X}}

\def\mcbb{\mathcal{B}}  

\def\mcx{\mathcal{X}}

\def\mcf{\mathcal{F}}

\def\rset{\mathbb{R}}

\def\nset{\mathbb{N}}
\def\nsets{\mathbb{N}^*}

\def\rmd{\mathrm{d}}

\newcommand{\abs}[1]{\left\vert #1 \right\vert}
\newcommand{\absLigne}[1]{\vert #1 \vert}
\newcommand{\tvnorm}[1]{\| #1 \|_{\mathrm{TV}}}

\newcommand{\Vnorm}[2][1=V]{\| #2 \|_{#1}}

\newcommandx{\psr}[3][3=]{\left\langle#1,#2 \right\rangle_{#3}}
\newcommandx{\normr}[2][2=]{ \left\Vert#1 \right\Vert_{#2}}
\newcommandx{\psrLigne}[3][3=]{\langle#1,#2 \rangle_{#3}}
\newcommandx{\normrLigne}[2][2=]{ \Vert#1 \Vert_{#2}}

\newcommandx{\norm}[2][1=]{\ifthenelse{\equal{#1}{}}{\left\Vert #2 \right\Vert}{\left\Vert #2 \right\Vert^{#1}}}
\newcommand{\normLigne}[2][1=]{\ifthenelse{\equal{#1}{}}{\Vert #2 \Vert}{\Vert #2\Vert^{#1}}}

\newcommand{\parenthese}[1]{\left(#1 \right)}
\newcommand{\parentheseLigne}[1]{(#1 )}
\newcommand{\parentheseDeux}[1]{\left[ #1 \right]}
\newcommand{\parentheseDeuxLigne}[1]{[ #1 ]}
\newcommand{\defEns}[1]{\left\lbrace #1 \right\rbrace }
\newcommand{\defEnsLigne}[1]{\lbrace #1 \rbrace }

\newcommand{\proba}[1]{\mathbb{P}\left( #1 \right)}
\newcommand{\PP}{\mathbb{P}}

\newcommand\probaMarkovTilde[2][2=]
{\ifthenelse{\equal{#2}{}}{{\widetilde{\mathbb{P}}_{#1}}}{\widetilde{\mathbb{P}}_{#1}\left[ #2\right]}}

\newcommand{\PE}{\bbE} 
\newcommand{\expe}[1]{\PE \left[ #1 \right]}

\newcommand{\expeLigne}[1]{\PE [ #1 ]}

\newcommand{\plusinfty}{+\infty}
\def\ie{\textit{i.e.}}

\def\eqsp{\;}
\renewcommand{\iint}[2]{\{ #1,\ldots,#2\}}
\newcommand{\coint}[1]{\left[#1\right)}

\newcommand{\ooint}[1]{\left(#1\right)}
\newcommand{\ccint}[1]{\left[#1\right]}
\newcommand{\cointLigne}[1]{[#1)}

\newcommand\sequence[3][2=,3=]
{\ifthenelse{\equal{#3}{}}{\ensuremath{\{ #1_{#2}\}}}{\ensuremath{\{ #1_{#2}, \eqsp #2 \in #3 \}}}}
\newcommand\sequenceD[3][2=,3=]
{\ifthenelse{\equal{#3}{}}{\ensuremath{\{ #1_{#2}\}}}{\ensuremath{( #1)_{ #2 \in #3} }}}

\newcommand\sequenceDouble[4][3=,4=]
{\ifthenelse{\equal{#3}{}}{\ensuremath{\{ (#1_{#3},#2_{#3}) \}}}{\ensuremath{\{  (#1_{#3},#2_{#3}), \eqsp #3 \in #4 \}}}}

\newcommand{\wrt}{with respect to}

\def\iid{i.i.d.}

\def\eg{e.g.}
\def\Id{\operatorname{Id}}
\def\Idd{\operatorname{I}_d}

\def\rmD{\mathrm{D}}

\def\bigO{\mathcal{O}}

\def\trace{\operatorname{Tr}}

\def\sphere{\mathbb{S}}
\newcommand{\1}{\mathbbm{1}}

\def\restriction#1#2{\mathchoice
              {\setbox1\hbox{${\displaystyle #1}_{\scriptstyle #2}$}
              \restrictionaux{#1}{#2}}
              {\setbox1\hbox{${\textstyle #1}_{\scriptstyle #2}$}
              \restrictionaux{#1}{#2}}
              {\setbox1\hbox{${\scriptstyle #1}_{\scriptscriptstyle #2}$}
              \restrictionaux{#1}{#2}}
              {\setbox1\hbox{${\scriptscriptstyle #1}_{\scriptscriptstyle #2}$}
              \restrictionaux{#1}{#2}}}
\def\restrictionaux#1#2{{#1\,\smash{\vrule height .8\ht1 depth .85\dp1}}_{\,#2}}

\def\Ret{\mathrm{Ret}}

\def\parallelTransport{\mathrm{T}}
\def\distT{\rho_{\Theta}}
\def\metricM{\mathrm{g}}
\def\grad{\mathrm{grad}\,}
\def\Hess{\mathrm{Hess}\,}

\def\rmD{\mathrm{D}}

\def\noise{e}
\def\hnoise{\hat{e}}
\def\Exp{\mathrm{Exp}}
\def\planT{\rmT}
\def\Cut{\mathrm{Cut}}

\def\grassmann{\mathrm{Gr}}
\def\stiefel{\mathrm{St}}
\def\transpose{\top}
\def\scrl{\mathscr{L}}
\def\scrlinf{\scrl}
\def\scrlU{\scrl^{(1)}}
\def\scrlD{\scrl^{(2)}}
\def\trho{\tilde{\rho}}

\def\metric{\metricM}

\def\cu{\overline{c}}
\def\cl{\underline{c}}

\def\sigmaZ{\sigma_0^2}
\def\sigmaU{\sigma_1^2}
\def\bfb{\mathbf{b}}
\def\ttheta{\tilde{\theta}}

\def\Gammabf{\boldsymbol{\Gamma}}
\def\bupeta{\bar{\upeta}}

\newcommand{\beq}{\begin{equation}}
\newcommand{\eeq}{\end{equation}}

\def\varespilon{\varepsilon}
\def\bw{\mathrm{b}_w}
\newcommand{\txts}{\textstyle}
\def\moment{\mathrm{m}}
\def\hinfty{h_{\infty}}
\def\binfty{b_{\infty}}

\def\pca{\mathrm{pca}}
\def\rhoH{\rho_{\mathrm{H}}}
\def\Mbf{\mathbf{M}}

\usepackage{caption}
\usepackage{subfig}

\title{Convergence Analysis of Riemannian Stochastic Approximation Schemes}

\author{Alain Durmus \\
    Centre Borelli, UMR 9010\\
	\'Ecole Normale Supérieure Paris-Saclay \\
    \texttt{alain.durmus@ens-paris-saclay.fr}\\
	\And 
    \textbf{Pablo Jim\'enez} \\
	CMAP, UMR 7641\\
	\'Ecole Polytechnique \\
	\texttt{pablo.jimenez-moreno@polytechnique.edu} \\
	\And\\
    \textbf{\'Eric Moulines}\\
	CMAP, UMR 7641\\
	\'Ecole Polytechnique\\
	\texttt{eric.moulines@polytechnique.edu}\\
	\And \\
    \textbf{Salem Said}\\
	Laboratoire IMS, UMR 5218\\
	CNRS, Universit\'e de Bordeaux\\
	\texttt{salem.said@u-bordeaux.fr}\\
	\And \\
    \textbf{Hoi-To Wai}\\
	Department of SEEM,\\
	The Chinese University of Hong Kong\\
	\texttt{htwai@se.cuhk.edu.hk}
}

\begin{document}
\maketitle

\begin{abstract}%
  This paper analyzes the convergence for a class of Riemannian stochastic approximation (SA) schemes, designed to tackle optimization problems on Riemannian manifolds. We consider SA schemes which use either the exponential map of the considered manifold or retraction maps  used as a proxy for the exponential map. The relaxation to retraction maps is of considerable interest since in many problems  the computation of the exponential map  is either costly or even untractable.
Our results are derived under mild assumptions. First, our results  are {global} in the sense that we do not assume iterates to be a-priori bounded.  Second, we allow the presence of a bias in   SA schemes. To be more specific, we consider the case where the mean-field function can only be estimated up to a bias, and/or the case in which the samples are drawn from a controlled Markov chain. Third, the assumptions on  retractions required to ensure convergence of the related SA schemes are weak and are satisfied in many settings. We illustrate our findings on applications to principal component analysis and Riemannian barycenter problems. 
\end{abstract}
%

%

\section{Introduction}
This paper is concerned with the root finding problem on a smooth Riemannian manifold $\Theta$:\vspace{-.1cm}
\begin{equation} \label{eq:opt}
\text{find~~~$\theta \in \Theta$~~~satisfying~~~$h(\theta) = 0_{\theta}$, \quad where~~~$h(\theta) = \int_{\msx} H_{\theta}(x)\rmd \pi_{\theta}(x)$} \eqsp,\vspace{-.1cm}
\end{equation}
such that $h : \Theta \to \planT \Theta$ is a smooth vector field,
called the mean vector field, for any $\theta\in \Theta$,
$\pi_{\theta}$ is a distribution over $(\msx,\mcx)$ and
$H : \Theta \times \msx \to \planT \Theta$ is a bimeasurable function
satisfying for any $\theta \in \Theta$,
$H_{\theta} : \msx \to \planT_{\theta} \Theta$.
This framework includes stochastic optimization problems related to
a smooth (but possibly non-convex even in the geodesic sense) objective function $f: \Theta \rightarrow \rset^*$, with $h = \grad f$,
where $\grad$ is the Riemannian gradient operator. Problem \eqref{eq:opt} arises in many applications such as principal component analysis (PCA) and the computation of geometric barycenter \cite{oja:1992, arnaudon:2013}.

Our objective is to study Riemannian stochastic approximation (SA) schemes to compute an approximate solution of \eqref{eq:opt}.
In the case where the geodesic curves on $\Theta$ can be explicitly computed,
the Riemannian SA scheme  is defined through the recursion: for $n \in \nset$,\vspace{-.1cm}
\begin{equation} \label{eq:sa_exp}
  \theta_{n+1} = \Exp_{\theta_{n}}\defEns{\,\upeta_{n+1}H_{\theta_{n}}(X_{n+1}) } \eqsp,\vspace{-.1cm}
\end{equation}
where $\Exp_{\theta_n}$ is the exponential map at $\theta_n$, $(X_n)_{n \in\nsets}$ is a sequence of \emph{independent} random
variables in $(\msx,\mcx)$ such that $X_n$ has distribution
$\pi_{\theta_n}$ for any $n\in\nsets$, and
$(\upeta_{n})_{n \in \nsets}$ is a sequence of positive stepsizes.
The recursion \eqref{eq:sa_exp} is a natural extension
of the Robbins-Monro algorithm originally applied on the Euclidean
setting \cite{robbins:monro:1951}, where $\Exp_\theta(y) = \theta+y$ for any $\theta, y \in \rset^d$.

A number of results have been reported for Riemannian SA \eqref{eq:sa_exp} under the
stochastic gradient (SG) setting
where $h = \grad f$ \cite{hosseini2020recent}.
The pioneering work of \cite{bonnabel2013stochastic} established the asymptotic
convergence of \eqref{eq:sa_exp} to a critical point using martingale
techniques adapted from \cite{benveniste:metivier:priouret:1990,kushner:yin:2003,borkar2009stochastic}. Later, \cite{sra:2016,zhang2016riemannian,zhang2018r}
provided non-asymptotic analysis under
the assumption that the objective function $f$ is geodesically convex
($\metricM$-convex), and that all the
iterates remains in a compact subset of $\Theta$. 
Another related work is \cite{flammarion:2018} which studied
 a \emph{retraction} based
averaging scheme and
 showed a central limit theorem for \eqref{eq:sa_exp}.
Analysis of Riemannian SA for non-convex objective has been addressed in
\cite{hosseini2019alternative}. However, they require a
strong assumption on the retraction map, the objective function and
the manifold.  We also mention that a few other papers deal with deterministic optimization on Riemannian manifold and the
convergence to local minimum in non-convex settings \cite{boumal2018global,sun2019escaping,criscitiello2019escapingsaddles}, or applying Nesterov's acceleration \cite{ahn2020nesterov} for $\metricM$-convex functions.
To deal with relatively general conditions on
the mean-field $h$, \cite{shah:2019} analyzed stochastic recursion
schemes adapting the well-known ODE method in
\cite{benveniste:metivier:priouret:1990,kushner:yin:2003} to the
Riemannian setting to show asymptotic convergence where $\Theta$ is compact.

The present paper considers several
relaxations of the scheme defined by \eqref{eq:sa_exp}. First, we assume that $(X_n)_{n \in\nsets}$ is a Markov chain whose transition is controlled
by the current value of the parameter.
Second, we deal with the use of retraction operators in place of the
exponential map in the recursion \eqref{eq:sa_exp}.
At last, we allow for the non-SG setting when $h \neq \grad f$ as well as asymptotic errors in the numerical evaluations of $H$, which results in a potentially
biased vector field in \eqref{eq:sa_exp}.
Note that the non-SG setting is common for SA: examples include, among many others, online Expectation Maximization algorithms
\cite{cappe:moulines:2009} and policy gradient \cite{baxter:bartlett:2001}.
In this sense, our work is in line with recent studies of SA in the
Euclidean setting for minimizing non-convex objective function in the
controlled Markov chain setting under  relaxed conditions on the
mean-field $h$, see
\eg~\cite{ghadimi:lan:2013,bottou2018optimization,karimi:2019} and the
references therein.
Our contributions are:
 \begin{itemize}[noitemsep,leftmargin=6mm]
 \item
We perform a {global} convergence analysis of a \emph{biased} geodesic Riemannian SA scheme without assuming a bounded domain for the iterates $(\theta_n)_{n \in \nset}$ nor $\metricM$-convexity. Under this relaxed setting, the iterates of the geodesic SA scheme satisfy $\expeLigne{\normrLigne{ h(\theta_{I_n})}[\theta_{I_n}]^2}  = {\cal O}(b_\infty + \log n / \sqrt{n})$, where $b_\infty$ is the asymptotic bias of the vector field $H$ and $I_n$ is a discrete random variable supported on $\{1,\ldots, n\}$. We cover settings when the noise is a sequence of martingale increments or stems from a controlled Markovian dynamics, see \Cref{th:martingale} \& \ref{th:markov}. 
\item
We consider Riemannian SA schemes with first or second order retraction. Note that retraction operators have met great success in practice (see \eg~\cite{malick:2012}) since they are computationally cheaper to
evaluate. Similar to the geodesic scheme, our methods find an ${\cal O}(b_\infty + \log n / \sqrt{n})$-stationary point in $n$ iterations in expectation, see \Cref{prop:retmartingale} \& \ref{th:retmarkov}.
We illustrate on several examples that the required conditions on the retraction function hold.
\item
 We consider example applications on PCA and geometric barycenter, 
 for which
  we show that the required convergence conditions are satisfied.
\end{itemize}

\paragraph{Notations}
For any two sequences of real numbers $(u_n)_{n
\in \nset}$ and $(v_n)_{n \in \nset}$, we write $u_n=\bigO(v_n)$ when
there exist $n_0\in \nset$ and $M_0 \in \rset_+$ such that for any
$n\geq n_0$, $\absLigne{u_n} \leq M_0 v_n$. If $u_n=\bigO(v_n)$ and
$v_n=\bigO(u_n)$, we write $u_n=\Omega(v_n)$.
If $\mu = \mu^+ - \mu^-$ is a signed measure over a measurable space
$(E,\calA)$, where $\mu^+$ and $\mu^-$ are non-negative measures, called
the positive and negative part of $\mu$, define the total variation
$\tvnorm{\mu} = \mu^+(E) + \mu^-(E)$.
 For any $d \in \nsets$, $\Idd$ is the identity
matrix of size $d \times d$.
We denote the tangent
space of $\Theta$ at $\theta$ by $\planT_\theta \Theta$ and its tangent
bundle $\planT \Theta$. If $\Theta_0$ and $\Theta_1$ are two smooth
manifolds, for any smooth function $f:\Theta_0 \to \Theta_1$, we denote
its derivative by $\rmD f : \rmT \Theta_0 \to \rmT \Theta_1$. The
Riemannian metric on $\Theta$ is denoted by $\metricM$ but for ease of
notation and if there is no risk of confusion, for any $\theta\in
\Theta$, $u,v \in \planT_{\theta}\Theta$, we should denote
$\metric_{\theta}(u,v) = \psr{u}{v}[\theta]$ and $\metric_{\theta}(u,u)
= \normr{u}[\theta]^2 $. Let $\nabla$ be the Levi-Civita connection of
the metric $\metricM$ on $\Theta$ (see \Cref{app:levicivita}). 
consider: $\parallelTransport_{t_0 t_1}^{\upgamma}:
\planT_{\upgamma(t_0)} \Theta \to \planT_{\upgamma(t_1)} \Theta $
stands for the parallel transport associated to the Levi-Civita
connection along a curve $\upgamma : \rmI \to \Theta$ from
$\upgamma(t_0)$ to $\upgamma(t_1)$.  In the interest of space, we leave
detailed definitions and generalities on Riemannian geometry to
\Cref{app:bg_rie}.



\section{Geodesic Stochastic Approximation Schemes} \label{sec:results}
Let $(\msx,\mcx,\PP,(\mcf_n)_{n \in \nset})$ be a filtered probability
space and consider $(X_n)_{n \in \nsets}$, an
$(\mcf_n)_{n \in \nset}$-adapted stochastic process.
The present paper studies stochastic approximation (SA) sequences $(\theta_n)_{n \in \nset}$, used as approximate solutions of \eqref{eq:opt}, starting from $\theta_0 \in \Theta$ and defined by the recursion
\vspace{-0.2cm}
\begin{equation} \label{eq:sa_exp_b}
  \theta_{n+1} = \Exp_{\theta_{n}}\defEns{\,\upeta_{n+1}\parenthese{H_{\theta_{n}}(X_{n+1}) +b_{\theta_n}(X_{n+1})}} \eqsp ,
\end{equation}
where  $H$ is given in \eqref{eq:opt},
$b : \Theta \times \msx \to \rmT \Theta$ is a bi-measurable function
satisfying for any $\theta \in \Theta$,
$b_{\theta} : \msx \to \planT_{\theta} \Theta$, and
$(\upeta_{n})_{n \in \nsets}$ is a sequence of positive stepsizes.
Compared to the recursion given in the introduction \eqref{eq:sa_exp},
$b_{\theta_n}(X_{n+1})$ is an additional bias term used to explicitly model cases when $h$ cannot be evaluated no matter how we sample. 

We discuss some basic assumptions to be used throughout this paper.
First, we ensure that the exponential map $\Exp$ is well defined over $\planT \Theta$  with the following condition (for a definition of the Riemannian exponential mapping, see \Cref{app:cutlocus}).
\vspace{-0.1cm}
\begin{assumptionA}
  \label{ass:completeness}
$\Theta$ is a geodesically complete Riemannian manifold of dimension $d\in \nsets$.
\end{assumptionA}
\vspace{-0.1cm}

\noindent
In addition, we assume the existence of a Lyapunov function $V$ for the vector field $h$.
\begin{assumptionA}
  \label{ass:lyap_mean_field}
There exists a continuously differentiable function  $V : \Theta \to  \rset^*_+$ satisfying:
\begin{enumerate}[wide, labelwidth=!, labelindent=0pt,label=(\alph*),noitemsep,nolistsep]
    \item   \label{ass:lyap_mean_field_b} There exist constants $\cl_1,\cl_2,\cu > 0$ such that for any $\theta \in \Theta$,
  \begin{equation}
    \label{eq:9}
    \cl_1 \normr{h(\theta)}[\theta]^2  \leq  \cl_2  -\psr{\grad V(\theta)}{h(\theta)}[\theta] \eqsp, \qquad \normr{\grad V(\theta)}[\theta] \leq  \cu \normr{h(\theta)}[\theta] \eqsp.
  \end{equation}
  \item \label{ass:lyap_mean_field_a} The Riemannian gradient $\grad V$ (see \Cref{app:grad_hess}) is geodesically $L$-Lipschitz, \ie~there exists $L \geq 0$ such that for any $\theta_0,\theta_1 \in \Theta$, and geodesic curve $\upgamma:[0,1]\to \Theta$ between $\theta_0$ and $\theta_1$,
\begin{equation}\label{eq:llipschitz}
\normr{\grad V(\theta_1)- \parallelTransport_{01}^{\upgamma} \grad V(\theta_0)}[\theta_1] \leq L \ell(\upgamma) \eqsp ,
\end{equation}
where $\ell(\upgamma)=\normrLigne{\dot{\upgamma}(0)}[\theta_0]$ is the length of the geodesic.
  \end{enumerate}
\end{assumptionA}
The first inequality in \ref{ass:lyap_mean_field_b} is a Lyapunov
condition (see
\cite{lasalle1960some}), and the constant $\cl_2$ underlines a
mismatch between the orientations of the mean-field and the gradient of
the Lyapunov function. In addition, \Cref{ass:lyap_mean_field}-\ref{ass:lyap_mean_field_b} implies
$\normr{h(\theta)}[\theta] \leq C_1 \normr{\grad V(\theta)}[\theta] +
C_2$ for any $\theta \in \Theta$, for some constants $C_1,C_2 \geq
0$. Meanwhile, \ref{ass:lyap_mean_field_a} is satisfied if
$V$ has a continuous Riemannian Hessian $\Hess V(\theta): \rmT_\theta
\Theta \to \rmT_\theta
\Theta$ with a bounded operator norm for all $\theta \in
\Theta$, see \Cref{lem:bounded_hessian} in \Cref{app:assumptionlemma}.
Overall, conditions \ref{ass:lyap_mean_field_b} and
\ref{ass:lyap_mean_field_a} ensure the stability of the recursion
\eqref{eq:sa_exp_b} as they imply that
$h$ is sublinear, \ie, for any $\theta_0, \theta_1 \in
\Theta$, there exists $C \geq
0$ such that for any geodesic curve $\upgamma:[0,1]\to
\Theta$ between $\theta_0$,
$\theta_1$, it holds since
$\parallelTransport_{01}^{\upgamma}$ is an isometry by
\cite[Proposition 5.5-(f)]{lee:2019}, $\normr{h(\theta_1)}[\theta_1]
\leq C_1 \normr{\grad V(\theta_1)}[\theta_1] + C_2 \leq C_3(
\ell(\upgamma) + 1)$, for $C_3 \geq
0$.  Importantly, \Cref{ass:completeness}-\Cref{ass:lyap_mean_field}
allow us to generalize the  descent lemma to the Riemannian
setting, as follows (see \cite{sra:2016}).
\begin{llemma}
  \label{lem:taylor_grad_lip}
  Assume \Cref{ass:completeness}, \Cref{ass:lyap_mean_field}-\ref{ass:lyap_mean_field_a} hold.
  For any $\theta_0,\theta_1 \in \Theta$ and geodesic curve $\upgamma : \ccint{0,1} \to \Theta$
  between $\theta_0,\theta_1$,
  \begin{equation}
    \label{eq:taylor_lip}
    \abs{V(\theta_1) - V(\theta_0) - \psr{\grad V(\theta_0)}{\dot{\upgamma}(0)}[\theta_0]} \leq L \ell(\upgamma)^2/2 \eqsp.
  \end{equation}
\end{llemma}
\begin{proof}
For completeness, the proof is given in \Cref{sec:proof-crefl:taylor_grad_lip}.
\end{proof}
This derivation is crucial for our proofs, and exists in different
settings in the literature, \eg~\cite[Lemma 7.4.7]{absil:2008}, which
differs with this result because \Cref{lem:taylor_grad_lip} holds
for any geodesic curves, and is not limited to the length-minimizing
ones.
Finally, we assume that the bias term in \eqref{eq:sa_exp_b} is uniformly bounded.
\vspace{-0.1cm}
\begin{assumptionA}
  \label{ass:bounded_bias}
There exists a constant $b_{\infty}$ such that $\sup_{(\theta,x) \in  \Theta \times \msx} \normr{b_{\theta}(x)}[\theta] \leq b_{\infty}$.
\end{assumptionA}
\vspace{-0.1cm}

We provide non-asymptotic guarantees for the geodesic SA scheme \eqref{eq:sa_exp_b} to find an \emph{(approximate) stationary point}.
Roughly, these results ensure the ability of
the scheme \eqref{eq:sa_exp_b} to produce an estimator
$\tilde{\theta} \in\Theta$ based on $\{\theta_1,\ldots,\theta_n\}$,
satisfying
$\expeLigne{\normrLigne{h(\tilde{\theta})}[\tilde{\theta}]^2} \leq
\epsilon + \widetilde{c} b_\infty$, for a given target precision
$\epsilon>0$ and if $n$ is large enough. More precisely, the estimator
that we consider in this paper is the $I_n$-th iterate $\theta_{I_n}$, where $I_n$ is a
random variable independent of $\theta_0$ and $\mcf_n$,
with distribution, for any $\ell \in \{0,\ldots,n\}$,
\begin{equation} \label{eq:nn} \textstyle
  \proba{I_n = \ell} = \big(\sum^n_{i=0}\,\upeta_{i+1})^{-1} \upeta_{\ell+1} \eqsp.
\end{equation}
To simplify notation in our subsequent discussions, define for any $n \in \nset$ and $p \geq 1$,
\begin{equation}
  \label{eq:def_Gamma} \textstyle
  \Gamma_n^{(p)} = \sum_{k=1}^n \upeta_{\,k}^p \eqsp, \qquad \Gamma_n  = \Gamma_n^{(1)} \eqsp.
\end{equation}
To gain intuition on the analysis for geodesic SA scheme \eqref{eq:sa_exp},
let us present a simple proof in a simplified setting. 
\begin{pproposition}\label{prop:deter}
    Consider the sequence $(\theta_n)_{n\in \nset}$ defined by the
    deterministic version of \eqref{eq:sa_exp}, \ie~$H_{\theta}(x)=h(\theta)$ and $b_{\theta}(x) = 0$ for any $x \in \msx$
and $\theta \in \Theta$. Assume
\Cref{ass:completeness}-\Cref{ass:lyap_mean_field} and $\sup_{k
\in \nsets} \upeta_k \leq \cl_1 /L$. Then, for any $n \in \nset$ we
have,
\begin{equation} \label{eq:thdeter}
  \expeLigne{\normrLigne{ h(\theta_{I_n})}[\theta_{I_n}]^2} \, \leq 2 \cl_1^{-1}\{V(\theta_0) -V(\theta_{n+1})\}/\Gamma_{n+1} +2 \cl_1^{-1}\cl_2 \eqsp,
\end{equation}
where $I_n$ has distribution defined by \eqref{eq:nn}.
\end{pproposition}
\begin{proof}
The analysis here is a consequence of \Cref{lem:taylor_grad_lip}.
For any $k \geq 0$ and $t \in \ccint{0,1}$, we observe that the geodesic segment $\upgamma^{(k)}(t) =
  \Exp_{\theta_k}(t \upeta_{k+1}h(\theta_k))$ satisfies $\dot{\upgamma}^{(k)}(0) = \upeta_{k+1}
  h(\theta_k)$ and  $\ell(\upgamma^{(k)}) = \upeta_{k+1}
  \normr{h(\theta_k)}[\theta_k]$; see \cite[Corollary 5.6-(b)]{lee:2019}. Applying \Cref{lem:taylor_grad_lip} yields,
  \begin{align}
    \label{eq:4_main}
    \abs{V(\theta_{k+1}) - V(\theta_k) - \upeta_{k+1}\psr{\grad
    V(\theta_k)}{h(\theta_k)}[\theta_k]} 
    &\leq (L/2)
    \ell(\upgamma^{(k)})^2    \\
    &= (L \upeta_{k+1}^2 /2)
    \normr{h(\theta_k)}[\theta_k]^2 \eqsp.
  \end{align}
  Using \Cref{ass:lyap_mean_field}-\ref{ass:lyap_mean_field_b}, we obtain for any $k \in \nsets$
  \begin{equation} \label{eq:5_main}
  \textstyle \sum_{k=0}^n \upeta_{k+1}(\cl_1-(L/2)\upeta_{k+1})\normr{h(\theta_k)}[\theta_k]^2 \leq V(\theta_0) - V(\theta_{n+1}) + \cl_2 \Gamma_{n+1} \eqsp.
  \end{equation}
Setting the step size as $\sup_{k \in \nsets}\upeta_{k} \leq \cl_1/L$
and using \eqref{eq:nn}, we obtain by integrating \wrt\ $I_n$ the result \eqref{eq:thdeter}.
\end{proof}
In the case were $\cl_2 = 0$ in
\Cref{ass:lyap_mean_field}-\ref{ass:lyap_mean_field_b} and taking a
constant step size $\upeta_k = \cl_1 / L$, \Cref{prop:deter}
implies that $\expeLigne{\normrLigne{ h(\theta_{I_n})}[\theta_{I_n}]^2} =
{\cal O}(1/n)$.

We now consider the the actual geodesic SA scheme \eqref{eq:sa_exp_b}. Define
\begin{equation} \label{eq:markovh}
\noise_{\theta}(x) = H_{\theta}(x) - h(\theta) \eqsp.
\end{equation}
Note that for any $x\in \msx$, $\theta \mapsto \noise_\theta (x)$ is a vector field on $\Theta$ 
corresponding to the noise in the estimation of the mean field $h$. By \Cref{lem:taylor_grad_lip} and arguments paralleling those used to derive \eqref{eq:5_main}, we show 
\begin{llemma}
  \label{lem:first_inequality_lemma_main}
  Assume \Cref{ass:completeness}, \Cref{ass:lyap_mean_field}. Setting $\Delta M_k=
  \psrLigne{\grad V(\theta_k)}{\noise_{\theta_k}(X_{k+1})}[\theta_k]$ for any $k\in \nset$, we have for any $n \in \nsets$ and $\varepsilon >0$,
  \begin{equation}
    \label{eq:first_inequality}
  \begin{aligned}
&\textstyle{\sum_{k=0}^{n} \upeta_{k+1} (\cl_1-(3L/2) \upeta_{k+1}-\cu^2\varepsilon)
 \normr{h(\theta_k)}[\theta_k]^2} \\ 
      & \qquad \qquad\textstyle{ \leq \sum_{k=0}^n \upeta_{k+1}\Delta M_k + (3L/2) \sum_{k=0}^n \upeta_{k+1}^2 \normr{\noise_{\theta_k}(X_{k+1})}[\theta_k]^2}\\
    & \qquad \qquad \quad \textstyle{ +V(\theta_0) 
 +\cl_2 \Gamma_{n+1} 
  + \sum_{k=0}^n \upeta_{k+1}\defEns{\parentheseLigne{4\varepsilon}^{-1} + (3L/2)\upeta_{k+1}}\normr{b_{\theta_k}(X_{k+1})}[\theta_k]^2 } \eqsp.
  \end{aligned}
  \end{equation}
\end{llemma}
\begin{proof}
The proof is postponed to \Cref{sec:tech_lemma}.
\end{proof}
The terms in the right-hand side of the inequality given by
\Cref{lem:first_inequality_lemma_main} correspond to different
sources of error that we can identify.
First, the two terms $\sum_{k=0}^n \upeta_{k+1}\Delta M_k$ and $\sum_{k=0}^n
\upeta_{k+1}^2 \normr{\noise_{\theta_k}(X_{k+1})}[\theta_k]^2$ come
from the noise vector field, in the first and second-order
approximation given by \Cref{lem:taylor_grad_lip} respectively.
The other term gathers the different sources of bias. More specifically,
$V(\theta_0)$ corresponds to the initial conditions, $\cl_2
\Gamma_{n+1}$ is introduced by the constant $\cl_2$ underlied in
\Cref{ass:lyap_mean_field}-\ref{ass:lyap_mean_field_b} and
the last term results from the bias on the measures of the mean-field
function.

With a sufficiently small step size, the left-hand side in the inequality above can be lower bounded by $\sum_{k=0}^n \upeta_{k+1} \normr{h(\theta_k)}[\theta_k]^2$. Meanwhile, the right-hand side can be controlled if $\sup_{n \in \nset} \sum_{k=0}^n \upeta_{k+1}\PE[\Delta M_k] < \plusinfty$. To control this term, we consider two different settings.

\begin{assumptionMD}[\emph{Martingale Setting}]
  \label{ass:0mean_noise}
The sequence $(\noise_{\theta_n}(X_{n+1}))_{n \in \nset}$ is a
martingale difference sequence with respect to the filtration
$(\mcf_{n})_{n \in \nset}$, \ie~
$\expe{\noise_{\theta_n}(X_{n+1})\middle| \mcf_n} = 0$ for any $n \in
\nsets$. Also, there exist $\sigmaZ,\sigmaU < \plusinfty$ such that for
any $n \in \nsets$,
$
\expeLigne{\normrLigne{\noise_{\theta_n}\parentheseLigne{X_{n+1}}}[\theta_n]^2 | \mcf_n} \leq \sigmaZ +
\sigmaU\, \normrLigne{h(\theta_n)}[\theta_n]^2$.
\end{assumptionMD}
Under \Cref{ass:0mean_noise}, we note that $\PE[ \sum_{k=0}^n \upeta_{k+1}\Delta M_k ] = 0 $ in \Cref{lem:first_inequality_lemma_main} and thus the following result can be easily obtained:
%
\begin{theorem} \label{th:martingale}
 Assume \Cref{ass:completeness}-\Cref{ass:lyap_mean_field}-\Cref{ass:bounded_bias}-\Cref{ass:0mean_noise}.
 Consider $(\theta_k)_{k \in \nset}$ defined by \eqref{eq:sa_exp_b}.
 If $\sup_{k\in \nsets}\upeta_{k} \leq \cl_1/(6L(1+\sigma^2_{1}))$, then for any $n \in \nsets$,
\begin{align}
 \label{eq:thmartingale}
 \expe{\normr{ h(\theta_{I_n})}[\theta_{I_n}]^2} &\leq \parenthese{\cl_1 \Gamma_{n+1}}^{-1}
  \defEns{2 \expe{V(\theta_{0})}+ 3L (\sigma_0^2 + b_{\infty}^2) \Gamma_{n+1}^{(2)} }\\
  &\quad + 2( b_{\infty} \cu / \cl_1)^2 + 2\cl_2/\cl_1 \eqsp,
\end{align}
where $I_n$ has distribution  defined by \eqref{eq:nn} and is independent of $(X_k)_{k \in\nsets}$.
\end{theorem}
\begin{proof} 
    The proof of \Cref{th:martingale} is postponed to \Cref{sec:proof-crefth:m}.
\end{proof}
We notice that the result is similar to \eqref{eq:thdeter} for a deterministic scheme, with
additional constants related to the variance $\sigma_0$ and bias $b_\infty$.

To control the term $\Gamma_{n+1}^{(2)}$, it is desirable to select a nonincreasing step size.
When the step size is $\upeta_k =
\square / (k + \triangle)^\alpha$,
for some $\square, \triangle >0$ and  $\alpha\in (0,1]$,
\Cref{th:martingale} shows that
$\expeLigne{\normrLigne{ h(\theta_{I_n})}[\theta_{I_n}]^2} \leq R(\square, \triangle,\alpha,n)  + 2( b_{\infty} \cu / \cl_1)^2 + 2\cl_2/\cl_1$
where $R(\square, \triangle,\alpha,n)= \bigO(1/n^{\alpha \wedge (1-\alpha)})$ for $\alpha \in (0,1/2) \cup (1/2,1)$, 
$R(\square, \triangle,1/2,n)= \bigO(\log(n)/\sqrt{n})$ and $R(\square, \triangle,1,n)= \bigO(\log(n)^{-1})$.


\vspace{.2cm}
\noindent
We now turn to the Markovian setting.

\begin{assumptionMA}[\emph{Markovian Setting}]
      \label{ass:markov}
      There exists a Markov kernel $P$ on $(\Theta \times \msx) \times
    \mcx$ such that for any $n \in \nset$ and bounded and measurable
    function $\varphi : \msx \to \rset_+$, $\expeLigne{\varphi(X_{n+1})
    |\mcf_n} = \int_{\msx}\varphi(y)P_{\theta_n}(X_n, \rmd y)$.  In
    addition, for any $\theta \in
            \Theta$, $P_\theta$ admits a unique invariant distribution
            $\pi_\theta$ satisfying $h(\theta) = \int_{\msx}
            H_\theta(y) \rmd \pi_\theta(y)$.
\end{assumptionMA}
Consider a measurable function $w : \msx \to \cointLigne{1,+\infty}$ and the following additional condition.
\begin{assumptionMA}[$w$]
    \label{ass:w_markov}
The Markov kernel $P$ and the measurable function $w$ satisfy the following conditions.
    \begin{enumerate}[wide, labelwidth=!, labelindent=0pt,label=(\alph*),noitemsep,nolistsep]
                    \item \label{ass:item:w_markov:e_bound} There exists $e_{\infty} >0$
            such that for any  $x \in \msx$,
            $\sup_{\theta\in\Theta} \normrLigne{\noise_{\theta}(x)}[\theta] \leq e_{\infty} w^{1/2}
            (x)$.
        \item \label{ass:item:w_markov:moment} There exists $C_w \geq 1$
            such that $                \sup_{k \in \nset} \expe{ w(X_{k+1}) } \leq C_w$.
        \item \label{ass:item:w_markov:poisson} For any $\theta \in
            \Theta$, there exists a measurable function $\hnoise :
            \Theta \times \msx \to \planT \Theta$ satisfying for any $x
            \in \msx$, $\theta \in \Theta$, $\hnoise_\theta(x) \in
            \planT_\theta \Theta$, and
            \begin{equation}
                \label{eq:markov:poisson}
                \hnoise_{\theta}(x) - \int_{\msx} P_{\theta}(x,\rmd y)
                \hnoise_{\theta}(y) = \noise_{\theta}(x) \eqsp.
            \end{equation}
            Moreover, there exists $\hnoise_{\infty} >0$ such that for any
            $\theta \in \Theta$, $x \in \msx$
            $\normrLigne{\hnoise_\theta(x)}[\theta] \leq \hnoise_{\infty}
            w^{1/2}(x)$.
        \item   \label{ass:item:w_markov:poisson_regularity} There
            exists $L_{\hnoise}\geq 0$ such that for any
            $\theta_0,\theta_1 \in \Theta$, $x \in \msx$ and geodesic
            curve $\upgamma : \ccint{0,1} \to \Theta$ between
            $\theta_0$ and $\theta_1$,
            \begin{equation}
               \normr{\int_{\msx}
                    P_{\theta_1}(x,\rmd y) \hnoise_{\theta_1}(y)
                    -\parallelTransport_{01}^{\upgamma}\parentheseDeux{\int_{\msx}
                P_{\theta_0}(x,\rmd y)\hnoise_{\theta_0}(y)}}[\theta_1]
                \leq L_{\hnoise} \ell(\upgamma) w^{1/2}(x) \eqsp.
            \end{equation}
    \end{enumerate}
\end{assumptionMA}

We shall discuss the setting which implies \Cref{ass:w_markov}.
Note that \ref{ass:item:w_markov:e_bound} is a mild assumption on a
$\Theta$-uniform control (uniformly on the manifold) of the noise;
\ref{ass:item:w_markov:moment} is automatically satisfied if $w$ is bounded;
\ref{ass:item:w_markov:poisson} assumes the existence of solutions to the Poisson equation \eqref{eq:markov:poisson},
 which can be established if the Markov kernel $P_{\theta}$ is $w$-geometrically
  for any $\theta \in \Theta$.
In addition, it is also  required that $\hnoise_{\theta}$ (for fixed $\theta$) to be uniformly bounded by $w^{1/2}$.
Finally, assumption  \ref{ass:item:w_markov:poisson_regularity}  is implied by smoothness conditions
on the Markov kernel and the noise with respect to the SA parameter $\theta$.
Here is a specific statement of the previous discussion.
\begin{pproposition}
  \label{propo:condition_P_implication_MA}
Assume \Cref{ass:completeness}, \Cref{ass:markov}.  The assumption \Cref{ass:w_markov}$(w)$ holds if one of the following is true:
  \begin{enumerate}[wide, labelwidth=!, labelindent=0pt,label=(\alph*),noitemsep,nolistsep]
  \item For any $\theta \in \Theta$, $P_{\theta}$ is uniformly
    ergodic with constant uniform in $\theta$, \ie~it has a unique stationary distribution $\pi_{\theta}$ and there exist
    $\varespilon_P \in \ooint{0,1}$ and $C_P \geq 0$ such that for any $\theta \in \Theta$, $x,x'\in\msx$,
    and $k \in \nset$, $\tvnorm{\updelta_{x} P^k_{\theta} - \updelta_{x'} P^k_{\theta}} \leq
   C_P (1-\varepsilon_P)^k$. In addition
    $\sup_{x \in \msx, \theta\in \Theta}
    \normr{\noise_{\theta}(x)}[\theta] < \plusinfty$ and there exists $C \geq 0$ such that for any $\theta_0,\theta_1 \in \Theta$, $x \in \msx$ and geodesic
            curve $\upgamma : \ccint{0,1} \to \Theta$ between
            $\theta_0$ and $\theta_1$,
          \end{enumerate}
          \begin{equation}
            \label{eq:prop_condition_P_implication_MA_a}
              \tvnorm{\updelta_x P_{\theta_1} - \updelta_x P_{\theta_0}} \leq C \ell(\upgamma) \eqsp, \qquad
               \normr{ \noise_{\theta_1}(x)
                    -\parallelTransport_{01}^{\upgamma}\noise_{\theta_0}(x)}[\theta_1]
                \leq C \ell(\upgamma)  \eqsp.
              \end{equation}
              \begin{enumerate}[resume, wide, labelwidth=!, labelindent=0pt,label=(\alph*),noitemsep,nolistsep]
              \item \label{item:prop_condition_P_implication_MA_b} For
                  any $\theta \in \Theta$, $P_{\theta}=P$, where $P$ is
                  a Markov kernel on $\msx \times \mcx$. Moreover,
                  there exists $w : \msx \to \coint{1,\plusinfty}$,
                  $\lambda \in \ooint{0,1}$, $\bw \geq 0$ such that
                  for any $x \in \msx$, $                P w(x) \leq
                  \lambda w(x) + \bw \1_{\msc}(x)$
              where $\msc \in \mcx$ is a small set for $P$. In
              addition, there exists a constant $C \geq 0$ such that
              for any $x \in \msx$,
              $\sup_{\theta\in \Theta}
              \normr{\noise_{\theta}(x)}[\theta] < Cw^{1/2}(x)$ and for
              any $\theta_0,\theta_1 \in \Theta$, $x \in \msx$ and
              geodesic curve $\upgamma : \ccint{0,1} \to \Theta$
              between $\theta_0$ and $\theta_1$,
     \end{enumerate}
              \begin{equation}
                            \label{eq:prop_condition_P_implication_MA_b}
               \normr{\noise_{\theta_1}(x)
                    -\parallelTransport_{01}^{\upgamma}\noise_{\theta_0}(x)}[\theta_1]
                \leq C \ell(\upgamma)w^{1/2}(x)  \eqsp.
              \end{equation}
\end{pproposition}
\begin{proof} 
    The proof is postponed to \Cref{proof:propo:condition_P_implication_MA}.
\end{proof}
In the case $w$ is not bounded, we need to consider the following
condition which is automatically satisfied if $\Theta$ is compact,
which is the case for the Grassmann manifold for example.
\vspace{-0.2cm}
\begin{assumptionA}
  \label{ass:bounded_vh}
  There exist $h_{\infty} \geq 0$ such that   $\sup_{ \theta \in \Theta } \normr{ h(\theta) }[\theta] \leq h_{\infty}$. \vspace{-0.2cm}
\end{assumptionA}
We obtain the following result for geodesic SA scheme with Markovian noise:
\begin{theorem}
  \label{th:markov}
  Assume  \Cref{ass:completeness}-\Cref{ass:lyap_mean_field}-\Cref{ass:bounded_bias}-\Cref{ass:markov}-\Cref{ass:w_markov}$(w)$
  hold for some measurable function $w: \msx \to \coint{1,\plusinfty}$. Assume in addition either $w$ is bounded or  \Cref{ass:bounded_vh}.
  Let $(\upeta_k)_{k \in \nsets}$ be a sequence of stepsizes, 
  $a_1,a_2 \geq 0$ satisfying $\sup_{k \in \nsets} \upeta_k \leq \cl_1 / (4(3L/2+D_{\hnoise}))$ and
  \begin{equation}
    \label{eq:hyp_gamma_k}
\begin{aligned}
    &\sup_{k \in \nsets} \{\upeta_{k+1}/ \upeta_k\} \leq 1 \eqsp, \quad     \sup_{k \in \nsets} \{\upeta_{k}/ \upeta_{k+1} \} \leq a_1 \eqsp, \quad \sup_{k \in \nsets} \{\abs{\upeta_{k}-\upeta_{k+1}}/ \upeta_{k}^2 \} \leq a_2 \eqsp.
    \end{aligned}
  \end{equation}
Consider $(\theta_k)_{k \in \nset}$ defined by \eqref{eq:sa_exp_b}. Then for any $n \in \nsets$,
    \begin{equation}
        \label{eq:markov_theo}
        \expe{\normr{h(\theta_{I_n})}[\theta_{I_n}]^2} \leq \left. 2 \defEns{
            \expe{V(\theta_0)}+C(\upeta_1)+C_{\hnoise}\Gamma^{(2)}_{n+1}}\middle/ (\cl_1
        \Gamma_{n+1}) \right. + 2
        (b_{\infty} \cu / \cl_1)^{2} + 2 \cl_2/\cl_1 \eqsp.
    \end{equation}
  where $I_n \in \iint{0}{n}$ is  independent
  of $\calF_n$ and with distribution defined by \eqref{eq:nn} and the constants $C(\upeta_1)$, $D_{\hnoise}$ and $C_{\hnoise}$ are given by
\eqref{eq:C_upeta_markov_w}, \eqref{eq:def_const_theo2} and
\eqref{eq:def_const_theo2_v_bounded} in the appendix.
\end{theorem}
\begin{proof} 
    The proof of \Cref{th:markov} is postponed to \Cref{app:proofth2}.
\end{proof}
When the step size is $\upeta_k = \square / (k + \triangle)^\alpha$,
for some $\square, \triangle >0$ and  $\alpha\in (0,1]$, the same conclusions of \Cref{th:martingale} can be drawn again.



\section{General Retraction SA Schemes} \label{subsec:resretract}
The previous section focused on the geodesic schemes \eqref{eq:sa_exp_b} that require performing the Riemannian exponential map $\Exp$ at each iteration.
Evaluating the $\Exp$ map is often computationally challenging.
An alternative is to use a retraction map $\Ret:\rmT\Theta \to \Theta$. The main focus of this section is to analyze a retraction SA scheme:
\begin{equation} \label{eq:sa}
	\theta_{n+1} = \Ret_{\theta_{n}}\defEns{\,\upeta_{n+1}\parenthese{H_{\theta_{n}}(X_{n+1}) +b_{\theta_n}(X_{n+1})}},~n \in \nset \eqsp,
\end{equation}
given the initialization $\theta_0 \in \Theta$, and $\Ret_\theta$ stands for
the restriction of $\Ret$ to $\planT_{\theta} \Theta$, $\theta \in \Theta$.
We first discuss the properties of retraction operators with some examples below.

\subsection{Retraction and Quantitative Estimates of the Retraction Error}
\label{sec:retr-quant-estim}
For any $\theta \in \Theta$, we denote $\Cut(\theta) \subset \Theta$ as the cut locus of $\theta$ (see \Cref{app:cutlocus}). The following assumptions are made on the retraction maps of interest:
\begin{assumptionR}
	\label{ass:retraction}
	\begin{enumerate*}[label=(\roman*)]
		\item \label{ass:retraction_zero}
		For any $\theta \in \Theta$,  $\Ret_\theta(0_\theta) = \theta$, where $0_\theta$ is the zero element of $\planT_\theta \Theta$ and  $\rmD  \Ret_\theta(0_\theta) = \Id$.
              \end{enumerate*}
              \\
              \begin{enumerate*}[label=(\roman*),resume]
              \item
	\label{ass:retraction_first_a}
	For any $(\theta,u) \in \planT \Theta$, $\Ret_\theta(u) \notin \Cut(\theta)$.
      \end{enumerate*}\\
\begin{enumerate*}[resume,label=(\roman*)]
    \item
	\label{ass:retraction_invariance}
	$\Theta$ is a homogeneous Riemannian manifold (see \Cref{app:isometry}) with isometry group $\msg$
	and for any $g \in \msg$, $(\theta,u) \in \planT \Theta$, $       g\cdot \Ret_\theta(u) \, = \,
       \Ret_{g\cdot \theta}(g\cdot u)$,
       where $g\cdot \theta = g(\theta)$ and $g\cdot u = \rmD g_{\theta} (u)$ is a vector in $\planT_{g\cdot u} \Theta$.
       \end{enumerate*}
\end{assumptionR}
\Cref{ass:retraction}-\ref{ass:retraction_first_a} ensures that the inverse exponential map is defined on $\Ret_\theta (\planT_\theta \Theta)$.
Thus, the following function is well defined:
\begin{equation}
  \label{eq:def_phi_theta_retract}
  \Phi_\theta = \Exp^{-1}_\theta\circ \Ret_\theta : \planT_\theta\Theta \rightarrow \planT_\theta\Theta \eqsp,
\end{equation}
which defines a bundle map $\Phi : \planT \Theta \to \planT \Theta$.
This mapping allows us to precisely measure the quality of the approximation of the exponential map by the retraction
through quantitative bounds on the \textit{retraction error map} $\Phi_\theta - \Id_\theta$ on $\planT_\theta \Theta$.
Note that in the special case when $\Ret \equiv \Exp$, we have $\Phi_{\theta} \equiv \Id_\theta$, and the retraction error map
is the null function.

\Cref{ass:retraction}-\ref{ass:retraction_zero} roughly implies by a
Taylor expansion that
$\normrLigne{\Phi_\theta (u) - u}[\theta] \leq
\normrLigne{u}[\theta]$, for $(\theta,u)\in \planT \Theta$.  However as we shall see in the analysis,
to ensure convergence of SA using retraction maps, we need to
establish a tighter bound,
\ie~$\normrLigne{\Phi_\theta (u) - u}[\theta] \leq
\normrLigne{u}[\theta]^\upbeta$, for $(\theta,u)\in \planT \Theta$,
$\upbeta \geq 2$. It boils down to showing that the first non-zero
term in the Taylor expansion is of  order higher than one. To obtain
such a result, we consider the following assumptions, defining regular \emph{first-order} \Cref{ass:retraction_first_b} and \emph{second-order}  \Cref{ass:retraction_second} retraction; see  \cite{malick:2012}.
\begin{assumptionR}
  \label{ass:retraction_first_b}
For any $\theta \in \Theta$,  there exists $\scrlU(\theta) \geq 0$, such that  $ \normr{\rmD^2\Phi_\theta(tu)[u,u]}[\theta] \,\leq\,\scrlU(\theta)\,\normr{u}[\theta]^2$, for any $u \in \planT_{\theta}  \Theta$ and $t \in \ccint{0,1}$,  where
the function $\Phi_\theta : \planT_\theta\Theta \rightarrow \planT_\theta\Theta$ is defined by \eqref{eq:def_phi_theta_retract}.\vspace{-.2cm}
\end{assumptionR}
\begin{assumptionR}
For any $\theta \in \Theta$, the following hold.
\label{ass:retraction_second}
\begin{enumerate}[leftmargin=6mm, labelwidth=!, labelindent=0pt,label=(\roman*),noitemsep,nolistsep]
   \item \label{ass:retraction_second_acc}
     Setting $\upgamma(t) = \Ret_{\theta}(tu)$ for $t \in \rset$ and $u \in \planT_{\theta} \Theta$, the initial acceleration of the curve $\upgamma$ satisfies  $\rmD_t\dot{\upgamma}(0) = 0_{\theta}$, where $\rmD_t$ is the covariant derivative along $\upgamma$ (see \cite[Theorem 4.24]{lee:2019}).
      \item   \label{ass:retraction_second_derivative}
    There exists $\scrlD(\theta) \geq 0$, such that $ \normr{\rmD^3\Phi_\theta(tu)[u,u,u]}[\theta] \,\leq\,\scrlD(\theta)\,\normr{u}[\theta]^3$, for any $u \in \planT_{\theta}  \Theta$ and $t \in \ccint{0,1}$,  where
    the function $\Phi_\theta : \planT_\theta\Theta \rightarrow \planT_\theta\Theta$ is defined by \eqref{eq:def_phi_theta_retract}.
\end{enumerate}
\end{assumptionR}
\vspace{-0.2cm}
Note \Cref{ass:retraction_second} does not imply
\Cref{ass:retraction_first_b} nor vice versa.
The following is a consequence of \Cref{ass:retraction_first_b} or \Cref{ass:retraction_second}:
\begin{llemma} \label{lem:retractions}
Assume \Cref{ass:retraction} and let $\ttheta \in \Theta$, it holds
  \begin{enumerate}[wide, labelwidth=!, labelindent=0pt,label=(\alph*),noitemsep,nolistsep]
  \item \label{lem_retraction_a} Under \Cref{ass:retraction_first_b}, $ \normr{ \Phi_\theta(u) - u }[\theta] \leq \scrlinf^{(1)}(\ttheta) \Vert u \Vert^2_\theta$ for any $(\theta,u) \in \planT \Theta$.
    \item\label{lem_retraction_b} Under \Cref{ass:retraction_second},  $ \normr{ \Phi_\theta(u) - u }[\theta] \leq \scrlinf^{(2)}(\ttheta) \Vert u \Vert^3_\theta$ for any $(\theta,u) \in \planT \Theta$.
  \end{enumerate}
\end{llemma}
\begin{proof} 
    The proof is postponed to \Cref{app:prooflemretractions}.
\end{proof}
  We conclude this subsection by illustrating two examples of matrix manifolds with retraction operators satisfying \Cref{ass:retraction}-\Cref{ass:retraction_first_b}-\Cref{ass:retraction_second}.\vspace{-.1cm}
  \begin{example}($d$-dimensional sphere $\sphere^{d-1}$)
   \textnormal{
For $\Theta = \sphere^{d-1} = \{ x \in \rset^d : \norm{ x } = 1 \}$, we may take $(\theta,u) \in \planT \sphere^{d-1} \mapsto (\theta+u)/\norm{\theta+u}$ as a retraction map.  We show in \Cref{app:retractionsphere} that \Cref{ass:retraction}-\Cref{ass:retraction_first_b}-\Cref{ass:retraction_second} are satisfied for this example.}\vspace{-.1cm}
\end{example}

\begin{example}{(Grassmannian manifold $\grassmann_r(\rset^d)$)} \label{ex:grassmann}
\textnormal{  For $r \in \{1,\ldots,d-1\}$, let $\Theta = \grassmann_r(\rset^d)$ be the
set of $r$-dimensional subspaces over $\rset^d$. 
Note that $\grassmann_r(\rset^d)$ is a compact $r \times (d-r)$-dimensional manifold \cite[Example 21.21]{lee:2003}.
dimensional manifold, see
\cite[Example 21.21]{lee:2003}.  Following \cite[Section
2.5]{edelman:arias:smith:1998} or \cite[Problem 2.7]{lee:2019}, we
consider here  $\grassmann_r(\rset^d)$  as a quotient manifold of the
Stiefel manifold
$\stiefel_r(\rset^d) = \lbrace B \in \rset^{d\times r}\,
:\,B^{\transpose}B = {\rm I}_r\rbrace$. In words, the set $\stiefel_r(\rset^d)$ is the subset of $\rset^{d \times r}$ which are matrices for which columns form an orthonormal family of $\rset^d$ . In addition,  $\stiefel_r(\rset^d)$ is  a $r\times(2d -r-1)/2$-dimensional compact manifold  (see \cite[Problem 2.7]{lee:2019}). The manifold $\grassmann_r(\rset^d)$ can be seen as the quotient manifold of $\stiefel_r(\rset^d)$ by the right action of the
group of $r$-dimensional orthogonal matrices
$\rmO_r(\rset) = \{ O \in \rset^{r \times r} \,: \, O O^{\transpose} =
\rmI_r\}$, \ie~the map
$(B,O) \in \stiefel_r(\rset^d) \times \rmO_r(\rset) \mapsto BO \in
\stiefel_r(\rset^d)$; see \cite[Section 2.5]{edelman:arias:smith:1998}.
Indeed, it can be shown that
$[B] \in \stiefel_r(\rset^d)/\rmO_r(\rset) \mapsto \spanD(B) \in
\grassmann_r(\rset^d)$, where $B \in \stiefel_r(\rset^d)$ is any
representative of $[B]$ and $\spanD(B)$ is the linear space spanned by
$B$, is a Riemannian isometry see \Cref{app:isometry}. We use this
representation in the sequel and for any $B \in \stiefel_r(\rset^d)$,
we denote by $[B]$ the equivalence class of $B$ under the action of
$\rmO_r(\rset)$, \ie~$$[B] = \{ B O \, : \, O \in\rmO_r(\rset)\} \eqsp.$$ The
tangent space at $[B] \in \grassmann_r(\rset^d)$ is given by:
$\planT_{[B]} \grassmann_r(\rset^d)= \{ D \in \rset^{d \times r} \,:
\, B^{\transpose} D = 0_{r \times r} \}$, where $B$ is any representative of $[B]$. It is easy to show that
$$\planT_{[B]} \grassmann_r(\rset^d) = \{ B_{\perp} C \, :\, C \in
\rset^{d-r \times r} \}\eqsp,$$ where $B_{\perp}$ is any matrix satisfying
$B_{\perp} \in \stiefel_{d-r}(\rset^d)$, $B^{\transpose} B_{\perp} = 0_{r
\times d-r}$. In words, $B_{\perp} \in \rset^{d \times d-r}$ is a $d \times
(d-r)$-matrix for which the columns form an orthonormal family of $\rset^d$ and
are orthogonal to the span corresponding to $B$.    We consider then, the
canonical metric 
(see
\cite[Section 2.5]{edelman:arias:smith:1998}) defined for any $[B] \in  \grassmann_r(\rset^d)$, and $D_1,D_2 \in \planT_{[B]} \grassmann_r(\rset^d)$, $D_1 = B_{\perp} C_1, D_2 = B_{\perp} C_2$,
$\metric^{\grassmann}_{[B]}(D_1,D_2) = \trace(D_1^{\transpose} D_2) = \trace(C_1^{\transpose} C_2)$. 
The exponential map on
$\grassmann_r(\rset^d)$ corresponding to this metric is given for any $[B] \in \grassmann_r(\rset^d)$, $D = B_{\perp} C  \in \planT_{[B]} \grassmann_r(\rset^d)$ by
\begin{equation} {\footnotesize \label{eq:grassmannexp_main}
  \mathrm{Exp}_{[B]}(D) = \left[ (B\,,B_\perp) \,
\exp 
  \left( \begin{array}{cc}
  0 & -C^\top \\ C & 0 \end{array} \right) 
\left( \begin{array}{c} {\rm I}_r \\ 0_{d-r \times r} \end{array} \right) \right]
= \left[
  (B \mathbf{V} \, \mathbf{U})
    \left( \begin{array}{cc}
  \cos(\bfSigma ) \\  \sin(\bfSigma )  \end{array} \right) \bfV^{\transpose}
  \right]
} \eqsp,
\end{equation}
where $\exp(\cdot)$ is the matrix exponential and $D = \bfU\bfSigma \bfV^{\transpose}$ is the compact singular value decomposition.  A retraction map is then defined for any $[B] \in \grassmann_r(\rset^d)$, $D \in \planT_{[B]} \grassmann_r(\rset^d)$ by\vspace{-.1cm}
\begin{equation} \label{eq:grassmannretraction}
    \Ret_{[B]}(D) \,=\, \mathrm{Span}(B+D) \eqsp.\vspace{-.1cm}
\end{equation}
This retraction operator has been considered in  \cite{bonnabel2013stochastic,malick:2012} which discuss in details the computational benefits of this approximation. In practice, for any $[B] \in \grassmann_r(\rset^d)$, $D \in \planT_{[B]} \grassmann_r(\rset^d)$, a representative of $\mathrm{Span}(B+D)$ is set as $\Qbf$ using a QR decomposition $B+D = \Qbf\Rbf$. \vspace{-.1cm}
}
\begin{pproposition} \label{prop:retractiongrass}
  The projective retraction $\Ret$ defined by \eqref{eq:grassmannretraction}
satisfies
\Cref{ass:retraction}-\Cref{ass:retraction_first_b}-\Cref{ass:retraction_second}.\vspace{-.1cm}
\end{pproposition}
\begin{proof} 
    {The proof is postponed to  \Cref{sec:proof-crefpr}.}
\end{proof}
\end{example}\vspace{-0.1cm}
\Cref{ex:grassmann} is important in our application of retraction SA to PCA problems to be discussed in \Cref{subsec:pca}. Indeed, the retraction SA scheme that we consider with the retraction map \eqref{eq:grassmannretraction}  corresponds  to the famous Oja's algorithm.




\subsection{Analysis of Retraction SA Schemes}
\label{sec:analys-gener-retr}
To analyze the retraction SA scheme in \eqref{eq:sa}, we observe that it can be re-written as:\vspace{-0.1cm}
\begin{align}
	\theta_{n+1}
		     &= \Exp_{\theta_n}\defEns{ \upeta_{n+1}\parenthese{ H_{\theta_{n}}(X_{n+1})
		     +b_{\theta_{n}}(X_{n+1}) + \Delta_{\theta_n,\upeta_{n+1}}(X_{n+1}) } } \eqsp,\label{eq:newschemexp}
\end{align}
where 
$\Delta_{\theta,\upeta}(x)$ is the `retraction bias' defined for any $(\theta,x) \in \Theta \times \msx$ and $\upeta >0$ by
\vspace{-0.1cm}
\begin{equation} \label{eq:Deltan}
\Delta_{\theta, \upeta}(x)  = \upeta^{-1}
\Phi_{\theta} (\upeta\{H_{\theta}(x)+b_{\theta}(x)\} )  -\{H_{\theta}(x)+b_{\theta}(x)\}\eqsp.\vspace{-0.1cm}
\end{equation}
Our strategy for analyzing \eqref{eq:sa} is
to incorporate the retraction bias in the analysis as it is done for the
geodesic SA scheme, using \Cref{lem:retractions} and noting that $\Delta_{\theta, \upeta}(x) = \upeta^{-1} \{\Phi_{\theta}(u) - u\}$ with $u = \upeta\{H_{\theta}(x)+b_{\theta}(x)\}$ for any $\theta \in \Theta$, $x \in \msx$ and $\upeta >0$.
Again, we consider separately the martingale and Markovian settings.

\paragraph{Martingale setting}  We shall strengthen this assumption first and consider for  some $a>0$:

\begin{assumptionMD}[$a$]
  \label{ass:retraction_martingale}
There exists $\moment_{(a)} \geq 0$ such that a.s.
$\mathbb{E}[\Vert \noise_{\theta_n}(X_{n+1})
\Vert^{a}_{\theta_n}|\mathcal{F}_n] \, \leq \moment_{(a)}$, for any $n \in \nset$.
\end{assumptionMD}

\vspace{-0.2cm}

\noindent
We acquire the following convergence result for $(\theta_n)_{n \in\nset}$.



\begin{theorem} \label{prop:retmartingale}
Assume
\Cref{ass:completeness}-\Cref{ass:lyap_mean_field}-\Cref{ass:bounded_bias}-\Cref{ass:bounded_vh},
\Cref{ass:retraction}, \Cref{ass:0mean_noise} hold. Suppose either
\Cref{ass:retraction_first_b},
      \Cref{ass:retraction_martingale}($4$) or
      \Cref{ass:retraction_second},
      \Cref{ass:retraction_martingale}($6$).
      Consider
$(\theta_k)_{k \in \nset}$ defined by \eqref{eq:sa}. Then, if $\sup_{k \in \nsets}
\upeta_k \leq \bar{\upeta}$, for $\bupeta$ defined in \eqref{eq:def_bupeta_marting_retract}, for any $n \in\nset$
    \begin{equation} \label{eq:retmartingale1}
 \expe{ \normr{ h( \theta_{I_n} ) }[\theta_{I_n}]^2 } \leq  2(\cl_1
\Gamma_{n+1})^{-1}  \defEnsLigne{ \expe{V(\theta_0)} + C_R^{(2)}\Gamma_{n+1}^{(2)} + A^{(R)}_{n+1}}  + 2(b_\infty \cu h_{\infty} + \cl_2)/\cl_1
\eqsp,
\end{equation}
where $C^{(2)}_R$ is given by \eqref{eq:def_C_R_2}, $A^{(R)}_{n+1} = \bigO(\Gamma_{n+1}^{(3)})$ is given by \eqref{eq:def_A_R} and $I_n \in \iint{0}{n}$ is a random variable independent of $\calF_n$ and
distributed according to \eqref{eq:nn}.
 \vspace{.1cm}
\end{theorem}
\begin{proof} 
    The proof is postponed to \Cref{app:proofretmartingale}.
\end{proof}
  \Cref{prop:retmartingale} shows that the bounds for the retraction scheme \eqref{eq:newschemexp} and the geodesic scheme  \eqref{eq:sa_exp_b} are (up to the precise definitions of constants) the same.


\noindent \paragraph{Markovian Setting}: We strengthen
\Cref{ass:w_markov}$(w)$-\ref{ass:item:w_markov:moment} as follows:
\vspace{-0.2cm}
\begin{assumptionMA}[$w$]
    \label{ass:w_moment}
    The Markov chain $(X_k)_{k\in \nsets}$ satisfies that
            $\sup_{k \in \nset} \expe{
            w^3(X_{k+1}) } \leq C^{(3)}_w$ for
            $C^{(3)}_w \geq 0$.
\end{assumptionMA}
\begin{theorem} \label{th:retmarkov}
Assume
\Cref{ass:completeness}-\Cref{ass:lyap_mean_field}-\Cref{ass:bounded_bias}-\Cref{ass:bounded_vh},
\Cref{ass:retraction},
\Cref{ass:markov}-\Cref{ass:w_markov}$(w)$-\Cref{ass:w_moment}$(w)$ hold for some measurable
function $w : \msx \to \cointLigne{1,\plusinfty}$. Suppose in addition that \Cref{ass:retraction_first_b} or \Cref{ass:retraction_second} holds. Assume that
$(\upeta_k)_{k \in \nsets}$ is a
sequence of stepsizes and $a_1,a_2 \geq 0$ satisfying \eqref{eq:hyp_gamma_k}.
  Consider $(\theta_k)_{k \in \nset}$ defined by \eqref{eq:sa}.  Then, if $\sup_{k \in \nsets}
  \upeta_k \leq \bar{\upeta}$, for $\bupeta$ defined in \eqref{eq:def_bupeta_marting_retract_markov}, for any $n \in\nset$
    \begin{equation} \label{eq:retmartingale1_markov}
        \begin{aligned}
 \expeLigne{ \normr{ h( \theta_{I_n} ) }[\theta_{I_n}]^2 } \leq  2(\cl_1
\Gamma_{n+1})^{-1}  \defEnsLigne{ \expe{V(\theta_0)} +C(\upeta_1)+
D_R^{(2)}\Gamma_{n+1}^{(2)} + B^{(R)}_{n+1}}& \\  
+ 2(b_\infty \cu h_{\infty} + \cl_2)/\cl_1 & \eqsp,
\end{aligned}
\end{equation}
$C(\upeta_1),D^{(2)}_R$ are given by \eqref{eq:C_upeta_markov_w}-\eqref{eq:def_D_R_2}, $B^{(R)}_{n+1} = \bigO(\Gamma_{n+1}^{(3)})$ by \eqref{eq:def_B_R}  and $I_n \in \iint{0}{n}$ is a random variable independent of $\calF_n$ and
distributed according to \eqref{eq:nn}.
\end{theorem}
\begin{proof} 
    The proof is postponed to \Cref{app:proofretmarkov}.
\end{proof}
Again, the bounds are matched with those obtained in the martingale setting.


\section{Applications}\label{sec:applications}
In this section, we illustrate our convergence analysis results on
two examples: subspace tracking method 
and robust barycenter problem.

\subsection{Principal Component Analysis}\label{subsec:pca}
We consider online principal component analysis (PCA) problem in which we
estimate the $r$ principal eigenvectors of a $d \times d$ covariance
matrix $\bfA$; see \cite{cardot2018online} and the references therein. We assume that we have access to noisy data
$(X_n)_{n \in \nsets}$ in $\rset^{d}$ which have covariance matrices close in some sense to $\bfA$.
We consider two cases: 
\vspace{-0.2cm}
\begin{assumptionPCA}
  \label{ass:pca:weak_statio}
  $(X_n)_{n \in \nsets}$ is \iid, $ \expeLigne{X_{1}X_{1}^{\transpose}} = \bfA$, and there exists $\moment_{\pca}^{(12)} < \infty $ such that  $ \expeLigne{\normrLigne{ X_{1} }^{12}} \leq \moment^{(12)}_{\pca}$.
\end{assumptionPCA}

\begin{assumptionPCA}
  \label{ass:pca:markov}
$(X_n)_{n \in \nsets}$ is a Markov chain on $\rset^d \times \mcbb(\rset^d)$ with Markov kernel $P$ and stationary distribution $\pi$ such that $\int_{\rset^d} x \rmd \pi(x) = 0$ and  $\bfA = \int_{\rset^d} x x^{\transpose} \rmd \pi(x)$. In addition, there exists $w : \msx \to \coint{1,\plusinfty}$, $\lambda \in \ooint{0,1}$, $\bw \geq 0$ such that    for any $x \in \msx$, $                P w(x) \leq \lambda w(x) + \bw \1_{\msc}(x)$
              where $\msc \in \mcbb(\rset^d)$ is a small set for $P$. Finally, $w$  satisfies  $\sup_{x \in \rset^d} \{\norm[12]{x}/w(x)\}< \plusinfty$.
 \end{assumptionPCA}
Online PCA in the \iid~case has been tackled by, among others,  \eg~\cite{oja:1992,bonnabel2013stochastic,sra:2016},  \cite[Section 5]{cardot2018online}. We consider two algorithms introduced in \cite[Section~4.1]{bonnabel2013stochastic} to estimate $\bfA$. In both cases, PCA is considered as a stochastic minimization problem on the Grassmann manifold (see \Cref{sec:retr-quant-estim}) of
$r$-dimensional linear subspaces of $\rset^d$, $\grassmann_r(\rset^d)$, $r \in\{1,\ldots,d-1\}$ with objective function \vspace{-.1cm}
\begin{equation}
\txts
f([B] ) = -\trace(B^{\transpose} \bfA B) / 2 \eqsp. \vspace{-.1cm}
\end{equation}
We use the notation of \Cref{ex:grassmann}.
\begin{theorem}
  \label{theo:pca}
Based on the exponential map \eqref{eq:grassmannexp_main} or the
retraction \eqref{eq:grassmannretraction}, we consider the two
recursions
\begin{align}
  \label{eq:def:pca_sa_exp}
  &  [B^{(E)}_{n+1}] = [\Exp_{[B^{(E)}_n]}(\upeta_{n+1} \{\rmI_d - B^{(E)}_{n}(B^{(E)}_{n})^{\transpose}\}\{X_{n+1}X_{n+1}^{\transpose}\} B^{(E)}_{n}\}) ] \\
    \label{eq:def:pca_sa_ret}
 &   [B^{(Re)}_{n+1}] = \spanD(B^{(Re)}_n+ \upeta_{n+1} \{\rmI_d - B^{(Re)}_{n}(B^{(Re)}_{n})^{\transpose}\}\{X_{n+1}X_{n+1}^{\transpose}\} B^{(Re)}_{n}\}) \eqsp.
\end{align}
The schemes defined by \eqref{eq:def:pca_sa_exp} and
\eqref{eq:def:pca_sa_ret} satisfy the conditions of
\Cref{th:martingale,prop:retmartingale} respectively if
\Cref{ass:pca:weak_statio} holds and the conditions  of
\Cref{th:markov,th:retmarkov} respectively if \Cref{ass:pca:markov}
holds with $\cu_2 = \binfty = 0$.
\end{theorem}
\begin{proof}
Note that \Cref{ass:completeness} is satisfied following for example \cite[Section 2.5]{edelman:arias:smith:1998}.
Note that $f$ defines a map on $\grassmann_r(\rset^d)$ since $f( B ) = f( B' )$ for any $B' \in [B]$.
Moreover, by \cite[Section 2.5.3]{edelman:arias:smith:1998}, the gradient of $f$ is given by $\grad f([B]) = -\{\rmI_d - BB^{\transpose}\}\bfA B$, for any $[B] \in \grassmann_r(\rset^d)$. Then, setting $h \leftarrow -\grad f$ and $V \leftarrow -f$, since $f$ is smooth and $\grassmann_r(\rset^d)$ is compact, it is easy to verify that \Cref{ass:lyap_mean_field} is satisfied with $\cu_2 =0$.

In light of the above recursion, the noise vector field can be written as $e_{[B]}(x) = \{\rmI_d -  B B^{\transpose}\}x x^{\transpose}  B - \{\rmI_d - BB^{\transpose}\}\bfA B$ and the bias function is the null function $b_{[B]}(x) = 0$, for any $[B] \in \grassmann_r(\rset^d)$, $x \in \rset^d$.
We recall from \Cref{ex:grassmann} that the retraction operator in \eqref{eq:def:pca_sa_ret} satisfies
\Cref{ass:retraction}-\Cref{ass:retraction_first_b}-\Cref{ass:retraction_second}. We now verify the other assumptions.

$\bullet$ Under \Cref{ass:pca:weak_statio},   \Cref{ass:0mean_noise},\Cref{ass:retraction_martingale}$(6)$ hold.

    $\bullet$ Under  \Cref{ass:pca:markov},  \Cref{ass:markov} is
    automatically satisfied. We show in \Cref{lem:grass_poisson_lip}
    that  for any $x \in \rset^d$, $[B] \mapsto e_{[B]}(x)$ satisfies
    \eqref{eq:prop_condition_P_implication_MA_b}. Finally,
    \Cref{ass:pca:markov} implies that for any $[B] \in
    \grassmann_r(\rset^d)$ and $ x \in \rset^d$,
    $\normrLigne{e_{[B]}(x)}[{[B]}] \leq C
    w^{1/2}(x)$\footnote{$\normrLigne{e_{[B]}(x)}[{[B]}]$ is the norm
    the tangent vector $e_{[B]}(x)$ associated with  the Riemannian
metric of $\grassmann_r(\rset^d)$ given in \Cref{sec:retr-quant-estim}
and $\norm{x}$ is the Euclidean norm of $x$.}, for some constant $C
\geq 0$. Therefore, \Cref{ass:pca:markov} and an application of
\Cref{propo:condition_P_implication_MA}-\ref{item:prop_condition_P_implication_MA_b}
shows that  \Cref{ass:w_markov}$(w)$ is satisfied. Also, the explicit
expression for the noise $e$ and \Cref{ass:pca:markov} implies that
\Cref{ass:w_moment} holds.
\end{proof}
\Cref{theo:pca} shows that the schemes \eqref{eq:def:pca_sa_exp} and \eqref{eq:def:pca_sa_ret} converges to a zero of $\grad f$ with convergence bounds provided by our main results since $\cu_2 = \binfty = 0$.

\subsection{Robust Barycenter in a Hadamard Manifold} \label{subsec:apphuber}
Let $\Theta$ be a Hadamard manifold -- a simply-connected, complete Riemannian manifold of non-positive sectional curvature~\cite{lee:2019}. We assume that the sectional curvature of $\Theta$ is bounded below, $-\kappa^2 \leq \mathrm{sec}\,\Theta \leq 0$. A common example of this situation is $\Theta \,=\,\mathrm{S}_d^{++}(\rset)$,  the space of real $d\times d$ symmetric positive-definite matrices, equipped with its   affine-invariant metric \cite{pennec:2006}.
Consider data points $(X_n)_{n \in \nsets}$ lying on the Riemannian manifold $\Theta$, i.e., $X_n \in \Theta$, that are drawn from a distribution $\pi$. A fundamental machine learning problem is to compute some kind of central value of $\pi$, defined as an optimal solution to    
\vspace{-0.1cm}
\begin{equation}  \label{eq:huberV} \textstyle
\min_{\Theta} f \eqsp, \quad \text{ where } f(\theta) \,=\, \int_\Theta\,\trho(\theta,x)\,\pi(\rmd x) \text{ for $\theta \in\Theta$} \eqsp,\vspace{-0.1cm}
\end{equation}
where $\trho(\theta,x)$ is some Riemannian dissimilarity measure. For instance, the Riemannian barycenter, also called the Fr\'echet mean, is obtained by taking $\trho(\theta,x) = \rho^2_\Theta(\theta,x)$ where $\rho_\Theta : \Theta \times \Theta \rightarrow \mathbb{R}_+$ is the Riemannian distance of $\Theta$ \cite{goh2008clustering, mathieu2019continuous}.

The Riemannian barycenter is known to be sensitive to outliers, which motivated the idea of considering the Riemannian median, obtained by taking $\trho(\theta,x) = \rho_\Theta(\theta,x)$~\cite{arnaudon:2013}.
We consider a \emph{robust barycenter} by using a Huber-like dissimilarity measure, $\rhoH:\Theta \times \Theta \rightarrow \mathbb{R}_+$, defined for any $\theta_0,\theta_1\in\Theta$ by $  \rhoH(\theta_0,\theta_1) \,=\, \delta^2\,[1+\lbrace \rho_\Theta(\theta_0,\theta_1)/\delta\rbrace^{2\;}]^{\scriptscriptstyle 1/2}\,-\,\delta^2$,
where $\delta > 0$ is a cut-off constant. Observe that $\rhoH(\theta_0,\theta_1)$ behaves like $(1/2)\,\rho^2_\Theta(\theta_0,\theta_1)$ when $\rho_\Theta(\theta_0,\theta_1)$ is small compared to $\delta$, and like $\delta\,\rho_\Theta(\theta_0,\theta_1)$ when $\rho_\Theta(\theta_0,\theta_1)$ is large compared to $\delta$. Let $\pi$ be a probability distribution on $\Theta$. Using
$\rhoH$ in the optimization problem \eqref{eq:huberV} yields a \emph{robust barycenter} problem, and the robust barycenter is a global minimizer of \eqref{eq:huberV}.


In the simplest setting where $(X_n)_{n\in\nsets}$ are i.i.d.~from $\pi$, tackling the problem \eqref{eq:huberV} can be done by considering the following geodesic SA scheme:
\begin{equation} \label{eq:huberscheme} \textstyle
\textstyle  \theta_{n+1}\,=\, \Exp_{\theta_n} ( \upeta_{n+1}\,\Exp^{-1}_{\theta_n}(X_{n+1})/[1+\lbrace \rho_\Theta(\theta_n\,,X_{n+1})/\delta\rbrace^{2\;}]^{\scriptscriptstyle  1/2} ) \eqsp.
\end{equation}
The above is the ``recursive barycenter'' scheme proposed by \cite{sturm:2003,arnaudon:2012}, except that a move in the direction of a new observation $X_{n+1}$ is attenuated when this new observation lies too far from the current estimate $\theta_n\,$. This means that less weight is assigned to extreme observations. Note that this strategy and the scheme \eqref{eq:huberscheme} have been also considered in \cite{chakraborty:2020} to solve PCA in a robust manner. We show in \Cref{app:proofshuber} that results from \Cref{sec:results} can be applied and furthermore \eqref{eq:huberscheme} finds a unique and global minimizer to \eqref{eq:huberV}.


%
\subsection{Numerical Examples on the PCA Problem}
\label{subsec:numerical}
We illustrate our results on a PCA problem on
$\grassmann_r(\rset^{d})$ with $r=8$ and $d=50$ using the two settings  \Cref{ass:pca:weak_statio} and
\Cref{ass:pca:markov}. We use both the retraction map
\eqref{eq:grassmannretraction} and the recursion  \eqref{eq:def:pca_sa_ret}.

We first consider the case where $(X_{n})_{n \in \nset^*}$ is 
\iid~with zero-mean Gaussian distribution and covariance matrix
$\bfA$.  For our experiments,
$\bfA$ is a randomly sampled symmetric matrix with eigenvalues in $\ccint{10^{-6},10}$. Note that
\Cref{ass:pca:weak_statio} holds.  In \Cref{fig:ret_iid} we illustrate the
convergence bounds provided by \Cref{prop:retmartingale}. This figure shows a Monte Carlo estimation of
$\{\PE[\normrLigne{\grad f(B_{I_n}^{(Re)})}[{[B_n]}]] \, :\, n \in
\{1,\ldots, 10^5\}\}$ for the step-size sequences
$\upeta_n =\square / (n + \triangle)^\alpha$, $\alpha \in \{0.3,0.5,0.8\}$. We
can observe that the convergence bounds provided by  \Cref{prop:retmartingale} which are of order
$\bigO(1/n^{\alpha \wedge (1-\alpha)})$ for
$\alpha \in (0,1/2) \cup (1/2,1)$,
and $\bigO(\log(n)/\sqrt{n})$ are met.
\begin{figure}
  \centering
  \includegraphics[width=\textwidth]{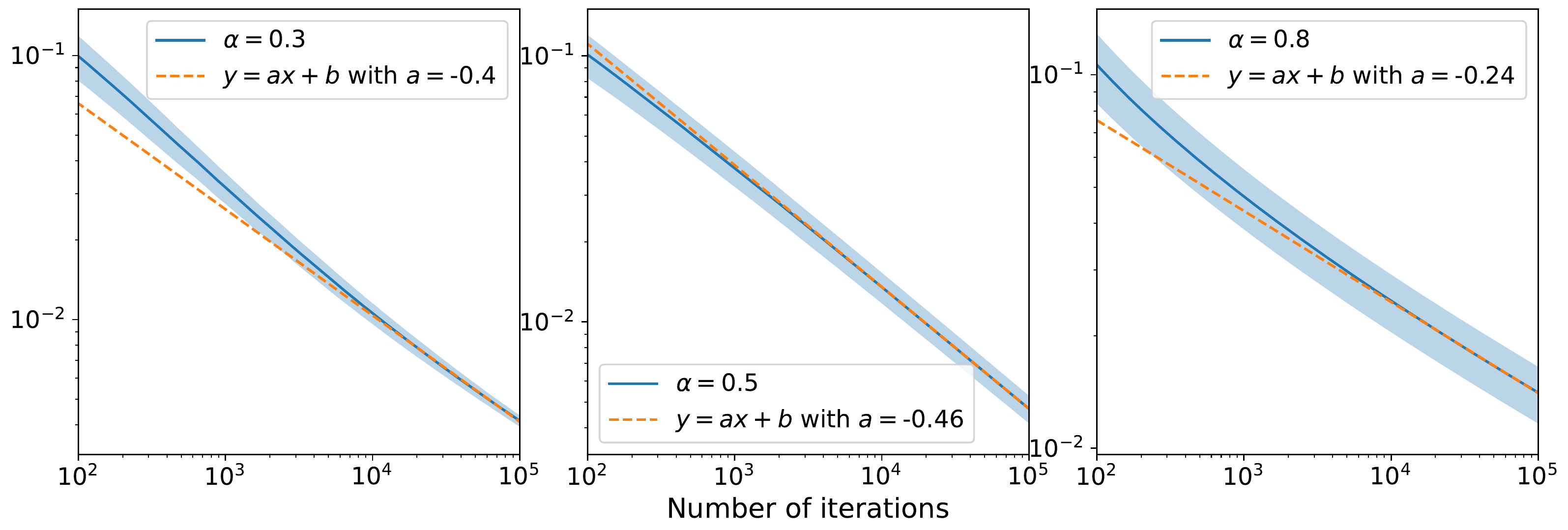}
  \caption{\label{fig:ret_iid} $\{\PE[\normrLigne{\grad f(B_{I_n}^{(Re)})}[{[B_n]}]] \}_{n}$ for 
$\upeta_n =\square / (n + \triangle)^\alpha$  in the \iid~setting}
\end{figure}

In a second experiment, we assume that $(X_{n})_{n \in \nsets}$ is a Markov chain,
with state space $\rset^d$, satisfying the recursion for any $k \in \nset$,
$X_{k+1} = \Mbf X_k + \varepsilon_{k+1}$, where 
$\Mbf \in \rset^{d \times d}$ and $(\varepsilon_k)_{k \in \nsets}$ is a
sequence of \iid~random variables with zero-mean Gaussian distribution
and covariance matrix $\bfA$. The matrix $\Mbf$ is chosen so that its
spectral radius is strictly less than 1, which implies that
\Cref{ass:pca:markov} is satisfied; see \Cref{sec:app_numerics}.  The
stationary distribution is zero-mean Gaussian with covariance matrix
satisfying the discrete Riccati equation:
$\Sigmabf_{\pi} = \Mbf \Sigmabf_{\pi} \Mbf^{\transpose} + \bfA$.
Similarly to the \iid~setting and with the same parameter choices, the convergence bounds provided by \Cref{th:retmarkov} are illustrated in \Cref{fig:ret_markov}.

\begin{figure}
  \centering
  \includegraphics[width=\textwidth]{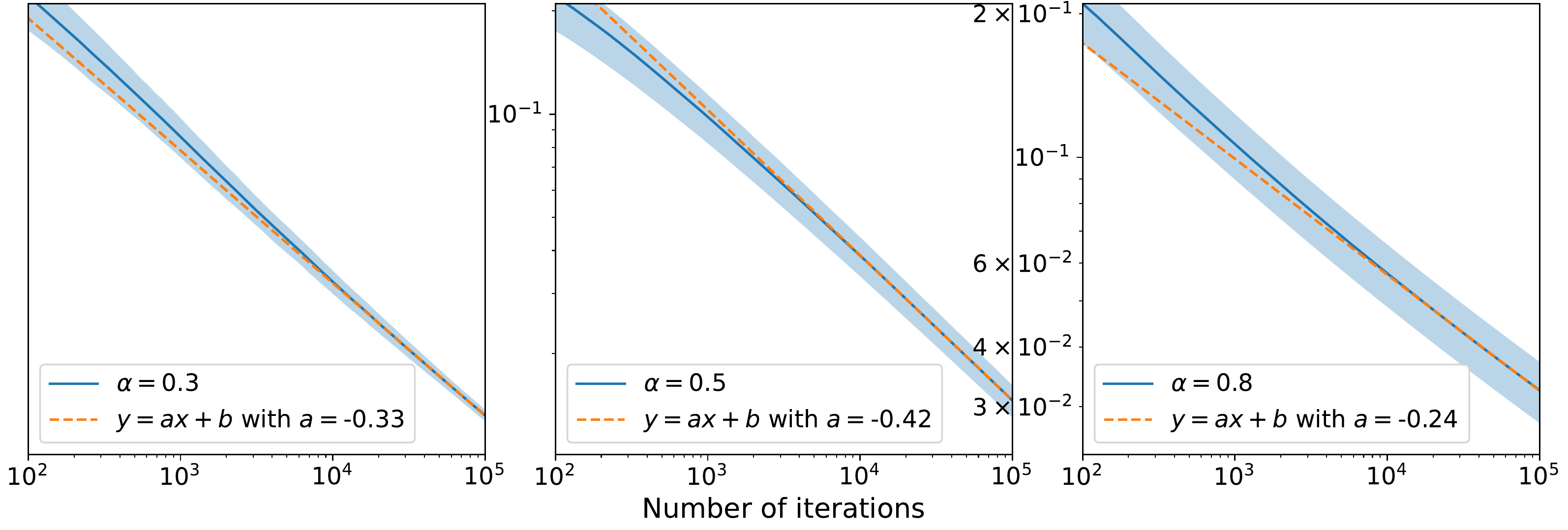}
  \caption{\label{fig:ret_markov} $\{\PE[\normrLigne{\grad f(B_{I_n}^{(Re)})}[{[B_n]}]] \}_{n}$ for 
$\upeta_n =\square / (n + \triangle)^\alpha$  in the Markovian setting}
\end{figure}


%

\bibliographystyle{abbrv}
\bibliography{bibliography}

\vfill
\pagebreak

\appendix 

\section{Proofs of \Cref{sec:results}}
\label{sec:proofs-theor-refth:m}

\subsection{Proof of \Cref{lem:taylor_grad_lip}}
\label{sec:proof-crefl:taylor_grad_lip}
  The proof uses results on parallel transport and geodesics on a Riemannian manifold. For reader convenience, these concepts are recalled in \Cref{app:metricdistance} and \ref{app:parallel}.

  Using $\upgamma(0) = \theta_0$, $\upgamma(1) = \theta_1$, by a Taylor expansion of $V\circ \upgamma$ and the definition of the Riemaniann gradient (see \Cref{app:grad_hess}), we have
 \begin{equation}
   V(\theta_1)-V(\theta_0) = \int_0^1 \psr{\grad V(\upgamma(t))}{\dot{\upgamma}(t)}[\upgamma(t)] \rmd t = \int_0^1 \psr{\grad V(\upgamma(t))}{\parallelTransport_{0t}^\upgamma \dot{\upgamma}(0)}[\upgamma(t)] \rmd t  \eqsp,
 \end{equation}
 where we have used for the last equality the uniqueness of the parallel transport \cite[Theorem 4.32]{lee:2019} and that $\upgamma$ is a geodesic. Therefore, we obtain, using that the parallel transport is a linear isometry  \cite[Proposition 5.5]{lee:2019}, that
 \begin{align}
   &\abs{V(\theta_1) - V(\theta_0) - \psr{\grad V(\theta_0)}{\dot{\upgamma}(0)}[\upgamma(0)]}   \\
   & \qquad \leq \int_{0}^1\abs{ \psr{\grad V(\upgamma(t))}{\parallelTransport_{0t}^{\upgamma} \dot{\upgamma}(0)}[\upgamma(t)] - \psr{\grad V(\upgamma(0))}{\dot{\upgamma}(0)}[\upgamma(0)]} \rmd t \eqsp ,\\
                                                                         &\qquad =\int_{0}^1\abs{ \psr{\grad V(\upgamma(t)) - \parallelTransport_{0t}^{\upgamma} \grad V(\upgamma(0))}{\parallelTransport_{0t}^{\upgamma} \dot{\upgamma}(0)}[\upgamma(t)] } \rmd t \eqsp ,\\
                                                                         & \qquad \leq  L \int_{0}^1 \ell\parenthese{\restriction{\upgamma}{\ccint{0,t}}} \normr{\parallelTransport_{0t}^{\upgamma} \dot{\upgamma}(0)}[\upgamma(t)]  \rmd t \leq L \ell(\upgamma)  \int_{0}^1 \ell(\upgamma) t \,  \rmd t \eqsp,
 \end{align}
 where we have used that since $\upgamma$ is a geodesic $\parallelTransport_{0t}^{\upgamma} \dot{\upgamma}(0) = \dot{\upgamma}(t)$ and by \cite[Corollary 5.6]{lee:2019}, $\normr{\dot{\upgamma}(t)}[\upgamma(t)] = \ell(\upgamma)$, and $\ell(\restriction{\upgamma}{\ccint{0,t}}) = t \ell(\upgamma)$ by  \cite[Lemma 5.18]{lee:2019}.

\subsection{A technical lemma} \label{sec:tech_lemma}
We begin this section with a first estimate which will be used in the proofs of our main results \Cref{th:martingale} and \Cref{th:markov}.
Recall the statement of \Cref{lem:first_inequality_lemma_main}.
\begin{llemma}
  \label{lem:first_inequality_lemma}
  Assume \Cref{ass:completeness}, \Cref{ass:lyap_mean_field} and let $\Delta M_k=
  \psrLigne{\grad V(\theta_k) }{\noise_{\theta_k}(X_{k+1})}[\theta_k]$ for any
  $k\in \nset$. We have for any $n \in \nsets$ and $\varepsilon >0$,
  \begin{equation}
  \begin{aligned}
&\textstyle{\sum_{k=0}^{n} \upeta_{k+1} (\cl_1-(3L/2) \upeta_{k+1}-\cu^2\varepsilon)
 \normr{h(\theta_k)}[\theta_k]^2} \\ 
      & \qquad \qquad\textstyle{ \leq \sum_{k=0}^n \upeta_{k+1}\Delta M_k + (3L/2) \sum_{k=0}^n \upeta_{k+1}^2 \normr{\noise_{\theta_k}(X_{k+1})}[\theta_k]^2}\\
    & \qquad \qquad \quad \textstyle{ +V(\theta_0) 
 +\cl_2 \Gamma_{n+1} 
  + \sum_{k=0}^n \upeta_{k+1}\defEns{\parentheseLigne{4\varepsilon}^{-1} + (3L/2)\upeta_{k+1}}\normr{b_{\theta_k}(X_{k+1})}[\theta_k]^2 } \eqsp.
  \end{aligned}
  \end{equation}
  where $e$ is defined by \eqref{eq:markovh}.
\end{llemma}

\begin{proof}
	For any $k \geq 0$ and $t \in \ccint{0,1}$, consider $\upgamma^{(k)}(t) = \Exp_{\theta_k}\{t \upeta_{k+1}(H_{\theta_k}(X_{k+1})+b_{\theta_k}(X_{k+1}))\}$. 
	Note that $\dot{\upgamma}^{(k)}(0) = \upeta_{k+1} \{H_{\theta_k}(X_{k+1})+b_{\theta_k}(X_{k+1})\}$ and $\ell(\upgamma^{(k)}) = \upeta_{k+1}\normr{H_{\theta_k}(X_{k+1})+b_{\theta_k}(X_{k+1})}[\theta_k]$. Then, by  \Cref{lem:taylor_grad_lip}, \eqref{eq:markovh} and using that for any $\theta \in \Theta$, $a,b,c \in \rmT_{\theta} \Theta$, $\normr{a+b+c}[\theta]^2 \leq 3(\normr{a}[\theta]^2+ \normr{b}[\theta]^2 +\normr{c}[\theta]^2 )$ , we get that for any $k \geq 0$,
  \begin{align}
    \label{eq:4}
&    \abs{V(\theta_{k+1}) - V(\theta_k) - \upeta_{k+1}\psr{\grad V(\theta_k)}{H_{\theta_k}(X_{k+1})+b_{\theta_k}(X_{k+1})}[\theta_k]} \leq (L/2) \ell(\upgamma^{(k)})^2 \eqsp ,\\
& \phantom{aaaaaaaaaaaaaaaaa}= (L \upeta_{k+1}^2 /2) \normr{H_{\theta_k}(X_{k+1})+b_{\theta_k}(X_{k+1})}[\theta_k]^2 \eqsp , \\
& \phantom{aaaaaaaaaaaaaaaaa}\leq (3L/2) \upeta_{k+1}^2 \defEns{\normr{h(\theta_k)}[\theta_k]^2 +\normr{b_{\theta_k}(X_{k+1})}[\theta_k]^2 + \normr{\noise_{\theta_k}(X_{k+1})}[\theta_k]^2 } \eqsp.
  \end{align}
Therefore, we get that for any $k \in \nset$,
\begin{equation}
\label{eq:proof_lemme_descente}
\begin{aligned}
    &-\upeta_{k+1}\psr{\grad V(\theta_k)}{h(\theta_k)}[\theta_k] \\
    &\qquad \qquad \leq V(\theta_k) - V(\theta_{k+1}) + \upeta_{k+1}\psr{\grad
    V(\theta_k)}{\noise_{\theta_k}(X_{k+1})+b_{\theta_k}(X_{k+1})}[\theta_k] \\
    &\qquad \qquad \quad +  (3L/2) \upeta_{k+1}^2
    \defEns{\normr{h(\theta_k)}[\theta_k]^2  +
        \normr{b_{\theta_k}(X_{k+1})}[\theta_k]^2
        +\normr{\noise_{\theta_k}(X_{k+1})}[\theta_k]^2
    } \eqsp. 
  \end{aligned}
  \end{equation}
By \Cref{ass:lyap_mean_field}-\ref{ass:lyap_mean_field_b} and the Cauchy-Schwarz inequality, we obtain for any $k \in \nsets$ and $\varepsilon>0$,
\begin{equation}
\psrLigne{\grad V(\theta_k)}{b_{\theta_k}(X_{k+1})}[\theta_k] 
\leq \cu^2 \varepsilon \normrLigne{h(\theta_k)}[\theta_k]^2 + (1/4\varepsilon) \normrLigne{b_{\theta_k}(X_{k+1})}[\theta_k]^2 \eqsp .
\end{equation}
Thus, plugging this in \eqref{eq:proof_lemme_descente} and using the
first inequality in \Cref{ass:lyap_mean_field}-\ref{ass:lyap_mean_field_b} gives,
\begin{equation}
    \begin{aligned}
        &\upeta_{k+1}(\cl_1-(3L/2)\upeta_{k+1} - \cu^2\varepsilon)\normr{h(\theta_k)}[\theta_k]^2 \\
        &\qquad \leq V(\theta_k) - V(\theta_{k+1}) 
        + \upeta_{k+1} \cl_2
        + \upeta_{k+1}\psr{\grad V(\theta_k)}{\noise_{\theta_k}(X_{k+1})}[\theta_k] \\
        &\qquad \quad+ [\upeta_{k+1}/4\varepsilon  + (3L/2) \upeta_{k+1}^2]\normr{b_{\theta_k}(X_{k+1})}[\theta_k]^2 
        + (3L/2) \upeta_{k+1}^2 \normr{\noise_{\theta_k}(X_{k+1})}[\theta_k]^2   \eqsp. 
    \end{aligned}
\end{equation}
  Adding these inequalities from $0$ to $n$ and rearranging terms concludes the proof.
\end{proof}

\subsection{Proof of \Cref{th:martingale}} \label{app:proofth1}
\label{sec:proof-crefth:m}

  Let $n \in \nsets$. First note that for any $k \in \nset$, 
  $\expeLigne{\Delta M_k}= 0$, 
  using that $\theta_k$ is $\mcf_k$-measurable and \Cref{ass:0mean_noise}. 
  Therefore, taking the expectation in the inequality given by \Cref{lem:first_inequality_lemma} 
  and $\normr{b_{\theta_k}(X_{k+1})}[\theta_k] \leq b_{\infty}$ using \Cref{ass:bounded_bias}, we obtain
 \begin{multline}
    \sum_{k=0}^{n} \upeta_{k+1} (\cl_1-(3L/2) \upeta_{k+1}-\cu^2\varepsilon) \expe{\normr{h(\theta_k)}[\theta_k]^2}  
    \leq \expe{V(\theta_0) } + \cl_2 \Gamma_{n+1} \\
    + (3L/2) \sum_{k=0}^n \upeta_{k+1}^2 \expe{\normr{\noise_{\theta_k}(X_{k+1})}[\theta_k]^2}  
    + b_{\infty}^2 \sum_{k=0}^n \upeta_{k+1} \defEns{(4\varepsilon)^{-1} + (3L/2)\upeta_{k+1}} \eqsp .
  \end{multline}
Since for any $k \in \nset$, $\expeLigne{\normr{\noise_{\theta_k}(X_{k+1})}[\theta_k]^2} \leq \sigma_0^2 + \sigma_1^2 \normr{h(\theta_k)}[\theta_k]^2$ using that $\theta_k$ is $\mcf_k$-measurable and \Cref{ass:0mean_noise}, we get
 \begin{align}
    &\sum_{k=0}^{n} \upeta_{k+1} (\cl_1-\cu^2\varepsilon-(3L/2)\{1+\sigma_1^2 \} \upeta_{k+1})\expe{ \normr{h(\theta_k)}[\theta_k]^2} \\
    & \qquad \leq \expe{V(\theta_0)} +\cl_2 \Gamma_{n+1} + 3L \sigma_0^2 \Gamma_{n+1}^{(2)}/2 + b_{\infty}^2 \defEnsLigne{ \Gamma_{n+1}/(4\varepsilon) + 3L \Gamma_{n+1}^{(2)}/2} \eqsp.
  \end{align}
  Taking $\varepsilon = \cl_1/(4\cu^2)$ and dividing by $\Gamma_{n+1}$, we get
  \begin{align}
  &\Gamma_{n+1}^{-1}\sum_{k=0}^{n} \upeta_{k+1} (3\cl_1/4-(3L/2)\{1+\sigma_1^2 \} \upeta_{k+1})\expe{ \normr{h(\theta_k)}[\theta_k]^2} \\
    & \qquad \leq \expe{V(\theta_0) }/\Gamma_{n+1} + 3L  \Gamma_{n+1}^{(2)} \defEnsLigne{\sigma_0^2 
    + b_{\infty}^2}/(2\Gamma_{n+1}) + \cl_2 + b_{\infty}^2\cu^2/\cl_1 \eqsp.
  \end{align}
  The proof is then completed using that for any $k \in \nsets$, $(\cl_1/4-(3L/2)(1+\sigma_1^2) \upeta_{k+1}) \geq 0$ and \eqref{eq:nn}.

\subsection{Proof of \Cref{propo:condition_P_implication_MA}}
\label{proof:propo:condition_P_implication_MA}

  \begin{enumerate*}[label=(\alph*)]
      \item Setting $w \equiv 1$, we prove that \Cref{ass:w_markov} holds. Since
    $e_{\infty} = \sup_{x\in\msx,\theta\in
    \Theta}\normrLigne{\noise_{\theta}(x)}[\theta]$ is finite, then
    \Cref{ass:w_markov}-\ref{ass:item:w_markov:e_bound}-\ref{ass:item:w_markov:moment}
    hold.
  \end{enumerate*}\\
Note that the condition  for any $\theta \in \Theta$, $x,x'\in\msx$,
    $\tvnorm{\updelta_{x} P^k_{\theta} - \updelta_{x'} P^k_{\theta}} \leq C_P
    (1-\varepsilon_P)^k$ implies, since $\pi_{\theta}$ is the unique stationary distribution for $P_{\theta}$, using Jensen inequality and the Markov property,
    \begin{equation}
      \label{eq:unif_convergence_proof_prop_condition_P_implication_MA}
      \tvnorm{\updelta_{x} P^k_{\theta} - \pi_{\theta}} \leq
C_P(1-\varepsilon_P)^k \eqsp.
    \end{equation}
    Therefore, the function $\hnoise$ given for any $\theta \in \Theta$ and $x \in \msx$,
  \begin{equation}
    \label{eq:hnoise_sufficient_condition}
    \hnoise_{\theta}(x) = \sum_{k=0}^{\plusinfty}\{ P^k_{\theta}\noise_{\theta}(x) - \pi_{\theta}(\noise_{\theta}) \} \eqsp,
  \end{equation}
 is well defined using the Minkowski's integral inequality for any $x \in \msx$, $\theta \in \Theta$, 
\begin{equation}
    \label{eq:hnoise_sufficient_condition_bounded_w}
\normr{      \hnoise_{\theta} (x)}[\theta] \leq \sum_{k=0}^{\plusinfty}\normrLigne{ P^k_{\theta}\noise_{\theta}(x) - \pi_{\theta}(\noise_{\theta}) }[\theta]  = e_{\infty} \sum_{k=0}^{\plusinfty} \tvnorm{\updelta_x P^k_{\theta} - \pi_{\theta}} \leq C_P e_{\infty}/\varepsilon_P \eqsp. 
\end{equation}
Then, \Cref{ass:w_markov}-\ref{ass:item:w_markov:poisson} holds.

To prove \Cref{ass:w_markov}-\ref{ass:item:w_markov:poisson_regularity}, since $\hnoise$ satisfies \eqref{eq:markov:poisson}, it is sufficient to show that there exists $D_1\geq 0$ such that
\begin{equation}
    \normr{\hnoise_{\theta_1}(x) - \parallelTransport_{01}^{\upgamma}\hnoise_{\theta_0}(x)}[\theta_1] \leq D_1 \ell(\upgamma) \eqsp, 
  \end{equation}
  for any $x \in \msx$, $\theta_0,\theta_1 \in \Theta$ and $\upgamma$ a geodesic between $\theta_0$ and $\theta_1$. 
Consider $\theta_0,\theta_1 \in \Theta$ and $\upgamma$ a geodesic between
$\theta_0$ and $\theta_1$. Note that by
\eqref{eq:hnoise_sufficient_condition}-\eqref{eq:prop_condition_P_implication_MA_a}
and using  that the parallel transport associated with the Levi-Civita
connection is a linear isometry \cite[Proposition 5.5]{lee:2019} and  the same
argument as proving  \eqref{eq:hnoise_sufficient_condition_bounded_w} for any
$x \in \msx$, setting $\Delta \noise_{\theta}(x) = \noise_{\theta_1}(x) -
\parallelTransport_{01}^{\upgamma}\noise_{\theta_0}(x)$ and using
$\normrLigne{\Delta \noise_\theta(x)}[\theta] \leq C \ell(\upgamma)$,
\begin{align}
  \label{eq:hnoise_sufficient_condition_bounded_w_2}
 & \normr{\hnoise_{\theta_1}(x) -
 \parallelTransport_{01}^{\upgamma}\hnoise_{\theta_0}(x)}[\theta_1] \\
 &\quad \leq \sum_{k=0}^{\plusinfty} \normr{   P_{\theta_1}^k \noise_{\theta_1}(x) -
 \parallelTransport_{01}^{\upgamma} P_{\theta_0}^k\noise_{\theta_0}(x) -
 (\pi_{\theta_1}(\noise_{\theta_1}) - \parallelTransport_{01}^{\upgamma}
 \pi_{\theta_0}( \noise_{\theta_0}))}[\theta_1] \\
 & \quad  \leq \sum_{k=0}^{\plusinfty} \normr{   P_{\theta_1}^k
     \Delta\noise_{\theta}(x)  - \pi_{\theta_1}(\Delta
     \noise_\theta)}[\theta_1] + \sum_{k=0}^{\plusinfty} \normr{
     [P_{\theta_1}^k - P_{\theta_0}^k](  \noise_{\theta_0})(x) -
 [\pi_{\theta_1}-\pi_{\theta_0}](\noise_{\theta_0})}[\theta_0] \\ 
 & \quad \leq C C_P \ell(\upgamma)/ \varepsilon_P + \sum_{k=0}^{\plusinfty} \normr{
 [P_{\theta_1}^k - P_{\theta_0}^k](  \noise_{\theta_0})(x) -
 [\pi_{\theta_1}-\pi_{\theta_0}](\noise_{\theta_0})}[\theta_0]  \eqsp. 
\end{align}
It remains to show that there exists $D_2 \geq 0$ such that 
\begin{equation}
  \label{eq:10}
  \sum_{k=0}^{\plusinfty} \normr{   [P_{\theta_1}^k - P_{\theta_0}^k](
  \noise_{\theta_0})(x) -
  [\pi_{\theta_1}-\pi_{\theta_0}](\noise_{\theta_0})}[\theta_0] \leq
  D_2 \ell(\upgamma) \eqsp,
\end{equation}
which follows from 
the following decomposition and \eqref{eq:prop_condition_P_implication_MA_a} and
\eqref{eq:unif_convergence_proof_prop_condition_P_implication_MA}.
%
%
\begin{multline}
  \label{eq:11}
  [P_{\theta_1}^k - P_{\theta_0}^k](  \noise_{\theta_0})(x) -
  [\pi_{\theta_1}-\pi_{\theta_0}](\noise_{\theta_0}) = -\pi_{\theta_1}
  P_{\theta_0}^k \noise_{\theta_0} + \pi_{\theta_0}(\noise_{\theta_0})
  \\
  + \sum_{i=0}^{k-1} \{\updelta_x P_{\theta_1}^{k-i-1} - \pi_{\theta_1}\}[P_{\theta_1}-P_{\theta_0}]\{P_{\theta_0}^i - \pi_{\theta_0}\}  \noise_{\theta_0} \eqsp. 
\end{multline}

    \begin{enumerate*}[label=(\alph*),resume]
    \item By assumption, \Cref{ass:w_markov}-\ref{ass:item:w_markov:e_bound} holds.
Besides, by a straightforward induction for any $x \in \msx$, 
      $k \in \nsets$,
      $P^kw(x) \leq \lambda^k w(x) + \bw \sum_{i=0}^{k-1} \lambda^i \leq
      \lambda^k w(x) + \bw/(1-\lambda)$ and therefore
      \Cref{ass:w_markov}-\ref{ass:item:w_markov:moment} holds. Since
      for any $\theta \in \Theta$, $P_{\theta} = P$, $\pi_{\theta} =
      \pi$ for some probability distribution $\pi$ on $(\msx,\mcx)$. 
      In addition, using Jensen inequality, for any $x \in \msx$, $Pw^{1/2}(x) \leq \lambda^{1/2} w^{1/2}(x) + \bw^{1/2}\1_{\msc}$. By \cite[Theorem 16.0.1]{meyn:tweedie:2009}, there exist $D_2 \geq 0$ and $\rho \in \coint{0,1}$ such that for any $x \in \msx$,
          \end{enumerate*}
      \begin{equation}
        \label{eq:proposi_implq_MA_geo_ergo}
        \Vnorm[w^{1/2}]{\updelta_x P^k - \pi} \leq C_p'\rho_P^k w^{1/2}(x) \eqsp. 
      \end{equation}
      Therefore, $\hnoise$ given by \eqref{eq:hnoise_sufficient_condition} is well defined since for any $\theta \in \Theta$ and $x \in \msx$, using the Minkowski's inequality and \eqref{eq:prop_condition_P_implication_MA_b}, 
      \begin{multline}
    \label{eq:hnoise_sufficient_condition_bounded_w_b}
\normr{      \hnoise_{\theta} (x)}[\theta] \leq
\sum_{k=0}^{\plusinfty}\normrLigne{ P^k\noise_{\theta}(x) -
\pi_{\theta}(\noise_{\theta}) }[\theta]  \\
= e_{\infty} \sum_{k=0}^{\plusinfty}        \Vnorm[w^{1/2}]{\updelta_x
P^k - \pi} \leq C_P' e_{\infty} w^{1/2}(x)/(1-\rho_P)  \eqsp. 
\end{multline}
Then, \Cref{ass:w_markov}-\ref{ass:item:w_markov:poisson} holds.

To prove
\Cref{ass:w_markov}-\ref{ass:item:w_markov:poisson_regularity}, since
$\hnoise$ satisfies \eqref{eq:markov:poisson}, it is sufficient to show
that there exists $D_3\geq 0$ such that $ \normr{\hnoise_{\theta_1}(x)
    -
\parallelTransport_{01}^{\upgamma}\hnoise_{\theta_0}(x)}[\theta_1] \leq
D_3 w^{1/2}(x) \ell(\upgamma)$. Note that by
\eqref{eq:hnoise_sufficient_condition}-\eqref{eq:prop_condition_P_implication_MA_b}-\eqref{eq:proposi_implq_MA_geo_ergo}
and using  that the parallel transport associated with the Levi-Civita
connection is a linear isometry \cite[Proposition 5.5]{lee:2019} and
the same argument as proving
\eqref{eq:hnoise_sufficient_condition_bounded_w_b} for any $x \in
\msx$, setting $\Delta \noise_{\theta}(x) = \noise_{\theta_1}(x) -
\parallelTransport_{01}^{\upgamma}\noise_{\theta_0}(x)$,
\begin{align}
  \label{eq:hnoise_sufficient_condition_bounded_w_2}
 \normr{\hnoise_{\theta_1}(x) - \parallelTransport_{01}^{\upgamma}\hnoise_{\theta_0}(x)}[\theta_1]
  & \leq \sum_{k=0}^{\plusinfty} \normr{   P^k \noise_{\theta_1}(x) - \parallelTransport_{01}^{\upgamma} P^k\noise_{\theta_0}(x) - (\pi(\noise_{\theta_1}) - \parallelTransport_{01}^{\upgamma} \pi( \noise_{\theta_0}))}[\theta_1]
  \\ & = \sum_{k=0}^{\plusinfty} \normr{   P^k \Delta\noise_{\theta}  - \pi(\Delta \noise_\theta)}[\theta_1] \leq C'_P C \ell(\upgamma) w^{1/2}(x)\sum_{k=0}^{\plusinfty} \rho_P^k \eqsp,
\end{align}
which completes the proof. 

\subsection{Proof of \Cref{th:markov}} \label{app:proofth2}
Under \Cref{ass:completeness}-\Cref{ass:lyap_mean_field}-\Cref{ass:bounded_bias}-\Cref{ass:w_markov}$(w)$ 
and for a sequence of step sizes $(\upeta_k)_{k \in \nsets}$, consider the constants $D_{\hnoise}$, $C(\upeta_1)$ and $C_{\hnoise}^0$ defined as follows:
\begin{equation}
  \label{eq:C_upeta_markov_w}
  C(\upeta_1) = \cu \hnoise_{\infty} ( \upeta_1 +2)\eqsp,
\end{equation}
if $w$ is bounded by $w_{\infty} = \sup_{\msx} w$,  
\begin{equation}
  \label{eq:def_const_theo2}
\begin{aligned}
  &  D_{\hnoise} = \hnoise_{\infty} C_w L + 2\cu a_2 +\cu \hnoise_{\infty}C_w 
  + L_{\hnoise} \cu w_{\infty}[1+a_1(\noise_{\infty} +b_{\infty})] \eqsp,  \\
  & C_{\hnoise} = C_{\hnoise}^0 + (3L/2)(b_{\infty}^2 + \noise_{\infty}^2 C_w)\eqsp,\\
  & C_{\hnoise}^0 =  L \hnoise_{\infty}C_w  (e_{\infty}+b_{\infty}) + L_{\hnoise}\cu C_w(\noise_{\infty}+ b_{\infty})  \eqsp;
\end{aligned}
\end{equation}
if \Cref{ass:bounded_vh} holds,
\begin{equation}
  \label{eq:def_const_theo2_v_bounded}
\begin{aligned}
  &  D_{\hnoise} = \hnoise_{\infty} C_w  L + 2\cu a_2 +\cu \hnoise_{\infty}C_w  \eqsp, \\
    & C_{\hnoise} = C_{\hnoise}^0 + (3L/2)(b_{\infty}^2 + \noise_{\infty}^2 C_w)\eqsp,\\
  & C_{\hnoise}^0 =L \hnoise_{\infty}C_w (e_{\infty}+b_{\infty})+L_{\hnoise}\cu C_w[(\noise_{\infty}+ b_{\infty})\{a_1h_{\infty}^2+1\} + h_{\infty}^2] \eqsp.
\end{aligned}
\end{equation}

\begin{llemma}
    \label{lem:poisson_w_markov}
    Assume \Cref{ass:completeness}-\Cref{ass:lyap_mean_field}-\Cref{ass:bounded_bias}-\Cref{ass:markov}-\Cref{ass:w_markov}$(w)$
    hold for some measurable function $w:\cointLigne{1,+\infty}$.  Assume in addition either $\sup_{x \in \msx}
            w(x) = w_\infty < \plusinfty$ or  \Cref{ass:bounded_vh}. 
    Let $(\upeta_k)_{k \in \nsets}$ be a sequence satisfying
    \eqref{eq:hyp_gamma_k} and set for any $k\in \nset$, $\Delta
    M_k=\psrLigne{\grad V(\theta_k)}{\noise_{\theta_k}
    (X_{k+1})}[\theta_k]$. Consider $(\theta_k)_{k \in \nset}$ defined by \eqref{eq:sa_exp_b}.
            Then, 
            \begin{equation}
                \label{eq:2}
\txts               \abs{ \expeLigne{\sum_{k=0}^{n} \upeta_{k+1}\Delta M_k }}\leq
                D_{\hnoise} \sum_{k=0}^{n} \upeta_{k+1}^2
                \expe{\normr{h(\theta_k)}[\theta_k]^2} +
                C_{\hnoise}^0 \Gamma_{n+1}^{(2)} + C(\upeta_1) 
                \eqsp.
            \end{equation}
\end{llemma}

\begin{proof}
  Consider the measurable function $\hnoise : \Theta \times \msx \to
  \planT \Theta$ which satisfies  \Cref{ass:w_markov}$(w)$ 
  and for any $k \in \nset$, consider $\upgamma^{(k+1)} : \ccint{0,1}
  \to \Theta$ the geodesic between $\theta_{k}$ and $\theta_{k+1}$
  defined by $\upgamma^{(k+1)}(t)  = \Exp\{t \upeta_{k+1}
      (H_{\theta_{k}}(X_{k+1}) +b_{\theta_{k}}(X_{k+1}))\}$ for any $t
      \in \ccint{0,1}$. Note that for any $k \in\nset$,
  \begin{equation}
    \label{eq:w_length_upgamma_k}
    \ell(\upgamma^{(k+1)}) = \upeta_{k+1}\normr{H_{\theta_{k}}(X_{k+1})+b_{\theta_k}(X_{k+1})}[\theta_{k}] \eqsp.
  \end{equation}
  Using that the parallel transport associated with the Levi-Civita
  connection is a linear isometry \cite[Proposition 5.5]{lee:2019} and
  $(\parallelTransport_{01}^{\upgamma})^{-1} =
  \parallelTransport_{10}^{\upgamma}$ by uniqueness of parallel
  transport \cite[Theorem 4.32]{lee:2019}, we obtain the following
  decomposition
  \begin{equation}
    \label{eq:w_decompo_geod_proof_markov}
        \expe{-\sum_{k=0}^{n} \upeta_{k+1}\Delta M_k } = -\expe{\sum_{i=1}^5 A_i} \eqsp,
      \end{equation}
      where 
      \begin{equation}
        \label{eq:w_fiveterms_decompose}
        \begin{aligned}
	  A_1  &=\sum_{k=1}^{n} \upeta_{k+1}\psr{\grad V(\theta_{k})}{\hnoise_{\theta_k}(X_{k+1}) - P_{\theta_k} \hnoise_{\theta_k}(X_k)}[\theta_k] \eqsp ,\\
	  A_2  &=\sum_{k=1}^{n} \upeta_{k+1}\psr{\grad V(\theta_{k})}{P_{\theta_k} \hnoise_{\theta_k}(X_k)-\parallelTransport_{01}^{\upgamma^{(k)}}  P_{\theta_{k-1}} \hnoise_{\theta_{k-1}}(X_k)}[\theta_k] \eqsp ,\\
          A_3 
          & = \sum_{k=1}^{n} \upeta_{k+1}\psr{\parallelTransport_{10}^{\upgamma^{(k)}}\grad V(\theta_{k}) - \grad V(\theta_{k-1})}{ P_{\theta_{k-1}} \hnoise_{\theta_{k-1}}(X_k)}[\theta_{k-1}] \eqsp ,\\
	  A_4  &=\sum_{k=1}^{n}( \upeta_{k+1}-\upeta_k)\psr{\grad V(\theta_{k-1})}{P_{\theta_{k-1}} \hnoise_{\theta_{k-1}}(X_k)}[\theta_{k-1}] \eqsp , \\
      A_5  &= \upeta_1\psr{\grad
      V(\theta_0)}{\hnoise_{\theta_0}(X_1)}[\theta_0] 
      - \upeta_{n+1}\psr{\grad
      V(\theta_n)}{P_{\theta_{n}}\hnoise_{\theta_n}(X_{n+1})}[\theta_n]  \eqsp. 
        \end{aligned}
      \end{equation}
      We now bound each term of this decomposition. Note that $A_2$ is
      the only term which will be bounded differently depending on
      either the assumption $\sup_{x \in \msx} w(x) = w_\infty <
      \plusinfty$ or \Cref{ass:bounded_vh}. That is why we deal first
      with $A_i$, $i \in \{1,\ldots,5\} \setminus \{2\}$.
      Note that using~\Cref{ass:markov} and
      \Cref{ass:w_markov}$(w)$-\ref{ass:item:w_markov:poisson},
      we obtain 
      $\expeLigne{\psrLigne{\grad
      V(\theta_{k})}{\hnoise_{\theta_k}(X_{k+1}) - P_{\theta_k}
      \hnoise_{\theta_k}(X_k)}[\theta_k]|\mcf_{k}} 
      = 0$ for any $k \in \iint{1}{n}$, and therefore
      \begin{equation}
                \label{eq:w_bound_a_1_geod}
        \expe{A_1} = 0 \eqsp. 
      \end{equation}
      Using the Cauchy-Schwarz inequality,
      \Cref{ass:lyap_mean_field}-\ref{ass:lyap_mean_field_a}, the
      definition of $\noise$ \eqref{eq:markovh},
      \Cref{ass:bounded_bias},
      \Cref{ass:w_markov}$(w)$-\ref{ass:item:w_markov:e_bound}-\ref{ass:item:w_markov:poisson}
      and Jensen's inequality, we obtain
      \begin{align}
        \abs{A_3}&\leq L \sum_{k=1}^n  \upeta_{k+1}\ell(\upgamma^{(k)})\normrLigne{P_{\theta_{k-1}}\hnoise_{\theta_{k-1}}
          (X_k)}[\theta_{k-1}]
\\
                   &\leq L \sum_{k=1}^n \upeta_k\upeta_{k+1}\defEns{
              \normrLigne{H_{\theta_{k-1}}(X_k)}[\theta_{k-1}]+
          \normrLigne{b_{\theta_{k-1}}(X_k)}[\theta_{k-1}]}
          \normrLigne{P_{\theta_{k-1}}\hnoise_{\theta_{k-1}}
          (X_k)}[\theta_{k-1}] \eqsp , \\
                 & \leq L \hnoise_{\infty} \sum_{k=1}^n \upeta_k\upeta_{k+1}
                 \parenthese{\noise_{\infty} w^{1/2}(X_k) + b_{\infty}+
                 \normr{h(\theta_{k-1})}[\theta_{k-1}]}
                 P_{\theta_{k-1}} w^{1/2}(X_k)\eqsp .
      \end{align}
      Taking the expectation, using 
      \Cref{ass:w_markov}$(w)$-\ref{ass:item:w_markov:moment},
      that $(\upeta_k)_{k \in \nsets}$ satisfies
      \eqref{eq:hyp_gamma_k} and the Cauchy-Schwarz inequality brings,
      \begin{equation}
                 \label{eq:w_bound_a_3_geod}
\txts                 \expe{\abs{A_3}}
                  \leq L  \hnoise_{\infty} C_w\defEnsLigne{ \parentheseDeuxLigne{ \noise_{\infty}
                 + b_{\infty} } \Gamma^{(2)}_{n+1} 
                 + \sum_{k=1}^n \upeta_{k}^2
                 \expeLigne{\normrLigne{h(\theta_{k-1})}[\theta_{k-1}]^2}}
                  \eqsp. 
      \end{equation}
      Using the Cauchy-Schwarz and Jensen inequality,
      \Cref{ass:lyap_mean_field}-\ref{ass:lyap_mean_field_b},
      \Cref{ass:w_markov}$(w)$-\ref{ass:item:w_markov:poisson}, for any
      $a,b \in \rset$, $\abs{ab} \leq (a^2+b^2)/2$ and that $(\upeta_k)_{k \in
      \nsets}$ satisfies \eqref{eq:hyp_gamma_k}, we have
      \begin{align}
        \abs{A_4} & \leq \cu\sum_{k=1}^n \abs{\upeta_{k+1}- \upeta_{k}}
        \normrLigne{h(\theta_{k-1})}[\theta_{k-1}]
        \normrLigne{P_{\theta_{k-1}}
        \hnoise_{\theta_{k-1}}(X_k)}[\theta_{k-1}] \eqsp, \\
                  & \leq 2  \cu \hnoise_{\infty}^2 \sum_{k=1}^n \abs{\upeta_{k+1}- \upeta_{k}}\{P_{\theta_{k-1}}w^{1/2}(X_k)\}^2 +  2 \cu  a_2 \sum_{k=1}^n \upeta_k^2
                  \normrLigne{h(\theta_{k-1})}[\theta_{k-1}]^2 
                   \eqsp.
        \end{align}
        Taking the expectation, using 
        \Cref{ass:w_markov}$(w)$-\ref{ass:item:w_markov:moment} and $(\upeta_k)_{k \in\nsets}$ is non-increasing bring,
        \begin{equation}
                  \label{eq:w_bound_a_4_geod} 
    \txts    \expe{|A_4|}
 \leq 2  \cu \hnoise_{\infty}^2 C_w\upeta_1+ 2 \cu a_2 \sum_{k=1}^n
    \upeta_{k}^2 \expeLigne{\normrLigne{h(\theta_{k-1})}[\theta_{k-1}]^2}    \eqsp.
    \end{equation}
      Finally, using
      \Cref{ass:lyap_mean_field}-\ref{ass:lyap_mean_field_b},
      \Cref{ass:w_markov}$(w)$-\ref{ass:item:w_markov:poisson} and
      Jensen's inequality, we obtain,
      \begin{align}
        \abs{A_5} & \leq \upeta_1 \cu \normr{h(\theta_0)}[\theta_0]
        \normr{\hnoise_{\theta_0}(X_1)}[\theta_0] + \upeta_{n+1} \cu
        \normr{h(\theta_n)}[\theta_n]
        P_{\theta_n}\normr{\hnoise_{\theta_n}}[\theta_n](X_{n+1}) \eqsp, \\
        &   \leq \cu \hnoise_{\infty} \defEns{ \upeta_1
        w^{1/2}(X_1)\normr{h(\theta_0)}[\theta_0] + \upeta_{n+1}
    P_{\theta_n}w^{1/2}(X_{n+1})\normr{h(\theta_n)}[\theta_n]} \eqsp. 
      \end{align}
      Taking the expectation, using the Cauchy-Schwarz inequality,
      \Cref{ass:w_markov}$(w)$-\ref{ass:item:w_markov:moment} and 
      that for any $a \in \rset$, $a \leq a^2 + 1$ brings,
      \begin{align} \expeLigne{|A_5|} & \leq \cu \hnoise_{\infty} \parentheseLigne{ \upeta_1
          C_w^{1/2} \expeLigne{\normr{h(\theta_0)}[\theta_0]^2}^{1/2} +
          \upeta_{n+1} C_w^{1/2} \expeLigne{\normr{
          h(\theta_n)}[\theta_n]^2}^{1/2} } \eqsp, \\ 
          & \leq \cu \hnoise_{\infty} \parentheseLigne{ 2 + \upeta^2_1 C_w \expeLigne{ 
                  \normr{h(\theta_0)}[\theta_0]^2} +
          \upeta^2_{n+1} C_w \expeLigne{\normr{
            h(\theta_n)}[\theta_n]^2}} \eqsp, \\
        \label{eq:w_bound_a_5_geod}
      & \txts \leq \cu \hnoise_{\infty} \parentheseLigne{ 2 + C_w \sum_{k=0}^n \upeta^2_{k+1}
      \expeLigne{\normr{h(\theta_k)}[\theta_k]^2}} \eqsp.
      \end{align}
      It remains to treat $A_2$ depending on the additional two
      conditions we consider. We start by proving a general bound which
      hold.
Using the Cauchy-Schwarz inequality,  \eqref{eq:w_length_upgamma_k},
\Cref{ass:w_markov}$(w)$-\ref{ass:item:w_markov:e_bound}-\ref{ass:item:w_markov:poisson}-\ref{ass:item:w_markov:poisson_regularity}, 
      \eqref{eq:w_length_upgamma_k} and \Cref{ass:lyap_mean_field}-\ref{ass:lyap_mean_field_b}, we get
      \begin{align}
        &\abs{A_2} \\
                  &\quad \leq L_{\hnoise}\sum_{k=1}^{n} \upeta_{k+1}
                  \upeta_{k}w^{1/2}(X_k) \normr{\grad
                  V(\theta_{k})}[\theta_k]
                  \defEns{\normrLigne{H_{\theta_{k-1}}(X_k)}[\theta_{k-1}]
                  +\normrLigne{b_{\theta_{k-1}}(X_k)}[\theta_{k-1}]}
                  \eqsp ,\\
                  &\quad\leq L_{\hnoise}\cu\sum_{k=1}^{n} \upeta_{k+1}
                  \upeta_{k} w^{1/2}(X_k)
                  \normr{h(\theta_k)}[\theta_k]\{ \noise_{\infty} w^{1/2}(X_k)\\
                  & \qquad \qquad\qquad \qquad\qquad\qquad\qquad\qquad\qquad\qquad+ \normr{h(\theta_{k-1})}[\theta_{k-1}]
          +\normrLigne{b_{\theta_{k-1}}(X_k)}[\theta_{k-1}]\}
              \eqsp.
      \end{align}
Using that $(\upeta_k)_{k\in
      \nsets}$ satisfies \eqref{eq:hyp_gamma_k},
      \Cref{ass:bounded_bias} and for any $a,b \in
      \rset, \absLigne{ab} \leq (a^2 + b^2) /2$, we obtain
      \begin{align}
          \abs{A_2} &\leq L_{\hnoise} \cu \sum_{k=1}^n \upeta_k^2
           \normr{h(\theta_k)}[\theta_k] \parenthese{\noise_{\infty} w(X_k) +
           b_\infty w^{1/2}(X_k)} \\
                    &\quad + (L_{\hnoise} \cu/2) \sum_{k=1}^n
                    w^{1/2}(X_k) \parenthese{
                    \normr{h(\theta_k)}[\theta_k]^2 \upeta^2_{k+1} +
                \normr{h(\theta_{k-1})}[\theta_{k-1}]^2 \upeta^2_k}
                \eqsp.
      \end{align}
      Changing the index in the second sum, using that for any $x \in
      \msx$, $w^{1/2}(x) \leq w(x)$ since $w(x) \geq 1$, 
      for any $a \in \rset$, $a \leq a^2+1$ on
      $\normrLigne{h(\theta_k)}[\theta_k]$ and the second 
      inequality in \eqref{eq:hyp_gamma_k} we get
      \begin{equation}
                \label{eq:w_bound_a_2_geod}
                \begin{aligned}
                    \abs{A_2}  &\leq L_{\hnoise}\cu \sum_{k=1}^{n}
              \upeta_{k}^2 w(X_k) (\noise_{\infty} + b_{\infty}) \\ 
                               & \quad +L_{\hnoise}\cu\sum_{k=0}^n\upeta_{k+1}^2 
              \normr{h(\theta_k)}[\theta_k]^2 \parenthese{[w(X_k)+w(X_{k+1})]/2+ a_1 w(X_k)[ \noise_{\infty} +
      b_{\infty} ]}
          \eqsp.
                \end{aligned}
      \end{equation}
      Now, taking the full expectation and using
      \Cref{ass:w_markov}$(w)$-\ref{ass:item:w_markov:moment} 
      brings
      \begin{equation}
          \label{eq:wbound_a_2}
          \begin{aligned}
              \expe{\abs{A_2}} 
              &\leq L_{\hnoise} \cu (\noise_{\infty} + b_{\infty})
              C_w \Gamma^{(2)}_{n+1} \\
              &\quad + L_{\hnoise} \cu \sum_{k=1}^n \upeta_{k+1}^2
                  \expe{\normr{h(\theta_k)}[\theta_k]^2
              \parenthese{ \{1 /2 + a_1 (\noise_{\infty} + b_{\infty} )\} w(X_k) +
          w(X_{k+1})/2} } \eqsp.
          \end{aligned}
      \end{equation}
We now distinguish the two cases. If  $\sup_{x \in \msx}
            w(x) = w_\infty < \plusinfty$, 
      we obtain 
      \begin{equation}
          \label{eq:wbound_a_2:w}
              \expe{\abs{A_2}} \leq L_{\hnoise} \cu (\noise_{\infty} + b_{\infty})
              C_w \Gamma^{(2)}_{n+1} + L_{\hnoise} \cu w_\infty \{1 +
              a_1(\noise_{\infty} + b_{\infty}) \} \sum_{k=1}^n \upeta_{k+1}^2
              \expeLigne{\normrLigne{h(\theta_k)}[\theta_k]^2} \eqsp.
      \end{equation}
If \Cref{ass:bounded_vh} holds,  we get 
      \begin{equation}
          \label{eq:wbound_a_2:h}
              \expe{\abs{A_2}} \leq L_{\hnoise} \cu \parentheseDeux{
                  (\noise_{\infty} + b_{\infty}) \{ h_\infty^2 a_1 +1\} +
              h_\infty^2 } C_w \Gamma^{(2)}_{n+1} \eqsp.
      \end{equation}
      Combining      \eqref{eq:w_bound_a_1_geod}-\eqref{eq:w_bound_a_3_geod}-\eqref{eq:w_bound_a_4_geod}-\eqref{eq:w_bound_a_5_geod}-\eqref{eq:wbound_a_2:w}-\eqref{eq:wbound_a_2:h},
      completes the proof.
\end{proof}

\begin{proof}[Proof of \Cref{th:markov}]
  The proof only consists in applying \Cref{lem:first_inequality_lemma}
  taking $\varepsilon = \cl_1/(4\cu^2)$, \Cref{lem:poisson_w_markov}
  using $\sup_{k \in \nsets} \upeta_k \leq \cl_1/(4(3L/2+D_{\hnoise}))$ and \Cref{ass:bounded_bias}-\Cref{ass:w_markov}$(w)$-\ref{ass:item:w_markov:e_bound}-\ref{ass:item:w_markov:moment}.
\end{proof}
  


\section{ Proofs for \Cref{sec:retr-quant-estim}} 

\subsection{Proof of  \Cref{lem:retractions}} \label{app:prooflemretractions}
We preface the proof of the Lemma by a preliminary result which does
not assume \Cref{ass:retraction}-\ref{ass:retraction_invariance}.
\begin{llemma} \label{lem:retractions_gene}
Assume \Cref{ass:retraction}-\ref{ass:retraction_zero}-\ref{ass:retraction_first_a} hold.
  \begin{enumerate}[wide, labelwidth=!, labelindent=0pt,label=(\alph*)]
  \item \label{lem_retraction_gene_a} Under \Cref{ass:retraction_first_b}, $ \normr{ \Phi_\theta(u) - u }[\theta] \leq \scrl^{(1)}(\theta)\Vert u \Vert^2_\theta/2$ for any $(\theta,u) \in \planT \Theta$.
    \item\label{lem_retraction_gene_b} Under \Cref{ass:retraction_second}, $ \normr{ \Phi_\theta(u) - u }[\theta] \leq \scrl^{(2)}(\theta) \Vert u \Vert^3_\theta/6$ for any $(\theta,u) \in \planT \Theta$.
  \end{enumerate}
\end{llemma}
\begin{proof}
\ref{lem_retraction_gene_a}
Let $\theta \in \Theta$ and for any $u \in \planT_{\theta} \Theta$, consider the first-order Taylor expansion of $\Phi_\theta:\planT_{\theta} \Theta \rightarrow \planT_{\theta} \Theta$, taken at $0_\theta$,
\begin{equation} \label{eq:taylor_first_1}
\Phi_\theta(u) = \Phi_\theta(0_\theta) + \rmD\Phi_\theta(0_\theta)[u] + \int^1_0\,(1-t)\,\rmD^2\Phi_\theta(tu)[u,u]\,\rmd t \eqsp,
\end{equation}
where $\rmD\Phi_\theta$ and $\rmD^2\Phi_\theta$ denote the first and second derivative of $\Phi_\theta$. For the first term, 
\begin{equation} \label{eq:taylor_first_1a}
\Phi_\theta(0_\theta) \,=\, \Exp^{-1}_\theta \circ \Ret_\theta(0_\theta) \,=\, \Exp^{-1}_\theta(\theta) \,=\, 0_\theta \eqsp,
\end{equation}
where the second equality follows because $\Ret_\theta(0_\theta) =
\theta$, by \Cref{ass:retraction}-\ref{ass:retraction_zero} and by definition that
$\Exp_{\theta}(0_{\theta}) = \theta$. For the second term, using that
$\Exp^{-1}_{\theta}$ and $\Ret_{\theta}$ are continuously
differentiable as function between smooth manifolds, $\rmD \Exp_{\theta}(0_{\theta}) = \Id$ by definition and 
\Cref{ass:retraction}-\ref{ass:retraction_zero}, we obtain 
  that 
\begin{equation} \label{eq:taylor_first_1b}
 \rmD\Phi_\theta(0_\theta)  = \rmD \Exp^{-1}_{\theta}(\Ret_{\theta}(0_{\theta})) \rmD \Ret_{\theta}(0_{\theta}) = \rmD \Exp^{-1}_{\theta}(\theta) \rmD \Ret_{\theta}(0_{\theta}) = \Id \eqsp. 
\end{equation}
The proof is then completed using \eqref{eq:taylor_first_1a}, \eqref{eq:taylor_first_1b} and \Cref{ass:retraction_first_b} in \eqref{eq:taylor_first_1}.

\ref{lem_retraction_gene_b}
Let $\theta \in \Theta$
 and consider the second-order Taylor expansion of $\Phi_\theta:\rmT_\theta\Theta \to \rmT_\theta\Theta$, taken at $0_\theta$:
\begin{equation} \label{eq:taylor_first_2_taylor}
\Phi_\theta(u) = u + \rmD^2\Phi_\theta(0_\theta)[u,u] /2 + 2^{-1}\int^1_0(1-t)^2\rmD^3\Phi_\theta(tu)[u,u,u]\, \rmd t \eqsp. 
\end{equation}
where $\rmD^3\Phi_\theta$ is the third derivative of $\Phi_\theta$.
 The proof relies on the use of normal coordinates with origin at $\theta$~\cite[Chapter 5]{lee:2019}. These coordinates are smooth and simply defined identifying $\planT_{\theta} \Theta$ with $\rset^d$ through $\Exp^{-1}_{\theta}$. More precisely, setting an orthonormal basis $\{\bfb_i \, : \, i \in\{1,\ldots,d\}\}$ of $\planT_{\theta} \Theta$, define for any $\ttheta \not \in \Cut(\theta)$,
\begin{equation}
  \label{eq:5}
  \varphi^i(\ttheta) = \psr{\Exp^{-1}_{\theta}(\ttheta)}{\bfb_i}[\theta]\eqsp. 
\end{equation}
Then, $\varphi = \{\varphi^i\,: \, i \in \{1,\ldots,d\}\}$ are smooth coordinates around $\theta$. Therefore, by definition $\Phi_{\theta}$ is simply $\Ret_{\theta}$ read in these coordinates. Then by \cite[Equation (4.15)]{lee:2019}, setting for any $t \in \rset_+$, $\upgamma(t) = \Ret_{\theta}(tu)$ for $u \in \planT_{\theta}\Theta$, we get that, in these coordinates,
\begin{equation} \label{eq:secondderivativetoacceleration0}
  \rmD_t \dot{\upgamma}(t) = \rmD^2 \Phi_{\theta}(tu)[u,u] + \sum^d_{k=1}\,\sum^d_{i,j=1} \rmD \Phi_{\theta}^i \dot{\upgamma}^{j}(t)\Gammabf^{k}_{i,j}(\upgamma(t)) \partial_k \eqsp, 
\end{equation}
where $\rmD_t$ is the covariant derivative along $\upgamma$, $\{\Gammabf_{i,j}^k \, : \, i,j,k \in \{1,\ldots,d\}\}$ are the Christoffel symbols and $\{\partial_k \, : \, k \in\{1,\ldots,d\}\}$ are the coordinate vector fields on $\planT \Theta$ corresponding to $\varphi$. But using \cite[Proposition 5.24]{lee:2019}, we get that $\Gammabf^{k}_{i,j}(\upgamma(0)) = \Gammabf^{k}_{i,j}(\theta)=0$ for any $i,j,k \in \{1,\ldots,d\}$. Therefore,
\begin{equation} \label{eq:secondderivativetoacceleration}
 \rmD_t \dot{\upgamma}(0) = \rmD^2 \Phi_{\theta}(0)[u,u] 
\end{equation}
and by \Cref{ass:retraction_second}, $  \rmD_t \dot{\upgamma}(0) =0$, which implies that $\rmD^2 \Phi_{\theta}(0)[u,u] = 0_{\theta}$. Plugging this result into \eqref{eq:taylor_first_2_taylor} and using the bound on the third derivative of $\Phi_{\theta}$ given by \Cref{ass:retraction_second} completes the proof.
\end{proof}

\begin{proof}[Proof of \Cref{lem:retractions}]

Recall that for any $(\theta,u) \in \planT \Theta$ and $g \in \msg$, we denote by $g \cdot u = \rmD g_{\theta}(u)$  and $g$ is an isometry,  \ie~$\normr{g \cdot u}[g \cdot \theta] = \normr{u}[\theta]$, see \Cref{app:isometry}.   The proof consists in showing that for any $(\theta,u) \in \planT \Theta$, and $g \in \msg$, 
\begin{equation}
  \label{eq:proof:lem:retractions}
\normr{\Phi_\theta(u) - u}[\theta]  = \normr{\Phi_{g \cdot \theta}(g\cdot u) - g\cdot u}[g \cdot \theta]  \eqsp. 
\end{equation}
Indeed, suppose that this result holds and  consider a fixed $\theta_0 \in\Theta$. Since $\Theta$ is assumed to be homogeneous \Cref{ass:retraction}-\ref{ass:retraction_invariance}, for any $(\theta,u)\in\planT \Theta$, there exists $g\in \msg$ such that $\theta_0 = g \cdot \theta$, which implies by \eqref{eq:proof:lem:retractions},
\begin{equation}
  \label{eq:proof:lem:retractions_D}
\normr{\Phi_\theta(u) - u}[\theta]  = \normr{\Phi_{\theta_0}(g\cdot u) - g\cdot u}[\theta_0]  \eqsp. 
\end{equation}
Using that $g \in \msg$ is an isometry and \Cref{lem:retractions_gene} completes the proof.

We now show \eqref{eq:proof:lem:retractions}. Let $g \in \msg$ and
$(\theta,u)\planT \Theta$. Since $g^{-1}$ is an isometry, we have
$\normr{\Phi_\theta(u) - u}[\theta] = \normr{g^{-1}\cdot g \cdot
  \Phi_\theta(u) - g^{-1}\cdot g \cdot u}[\theta] = \normr{g
  \Phi_{\theta}(u) - g\cdot u}[g \cdot \theta]$, so we only need to
show that $g \Phi_{\theta}(u) = \Phi_{g\cdot \theta}(g \cdot u)$. Using
that $g$ maps geodesics to geodesic (see \eqref{eq:isomexp} in \Cref{app:isometry}) and \Cref{ass:retraction}-\ref{ass:retraction_invariance} successively, we get that $g \Phi_{\theta}(u)  = \Exp^{-1}_{g \cdot \theta}(g \Ret_{\theta}(u))  = \Phi_{g \cdot \theta}( g\cdot u)$ and the proof follows.

\end{proof}

\subsection{Projective Retraction on the Sphere}
\label{app:retractionsphere}
Consider the Euclidean unit sphere manifold $\Theta = \sphere^{d} = \{ x \in \rset^{d+1} : \| x \| = 1 \}$, where $\norm{\cdot}$ stands for the standard Euclidean norm on $\rset^{d+1}$. By \cite[Example 3.5.1]{absil:2008}
for any $\theta \in\sphere^{d}$, $\planT_{\theta} \sphere^{d} = \{u \in \rset^{d+1} \, : \, u^{\transpose}\theta =0\}$. The Riemannian metric $\metricM$ is the canonical metric on the sphere, defined as the restriction of Euclidean scalar product from $\rset^{d+1}$ to the tangent space $\planT_\theta \sphere^d$. The corresponding Riemannian exponential is given by:
\begin{equation} \label{eq:expsphere_main} 
\Exp_\theta(u) = \cos(\| u \|)\,\theta + \sin(\| u \|) (u/\norm{u}) \eqsp.
\end{equation} 
The following result holds.
\begin{pproposition} \label{prop:retractionsphere}
The projective retraction $\Ret_\theta$ defined for any $(\theta,u) \in \planT \sphere^d$ by
\begin{equation} \label{eq:retractionsphere} 
  \Ret_\theta(u) = (\theta+u)/\norm{ \theta+u} 
\end{equation}
satisfies \Cref{ass:retraction}-\Cref{ass:retraction_first_b}-\Cref{ass:retraction_second}. 
\end{pproposition}

The retraction \eqref{eq:retractionsphere} is \emph{both} a first-order and second-order retraction, and $\Phi_\theta$ also has a bounded first-order derivative. Consequently,  \Cref{prop:retmartingale} and  \Cref{th:retmarkov}  can be applied according to conditions on the noise properties.
We remark that by comparing \eqref{eq:expsphere_main} with \eqref{eq:retractionsphere}, the retraction map $\Ret_\theta$ has a better numerical stability as it does not involve evaluating the trigonometric functions.\vspace{.1cm}

\begin{proof}[Proof of \Cref{prop:retractionsphere}]
First, for any $\theta \in \sphere^d$, the Riemannian exponential map is given for any $ u \in \planT_{\theta} \sphere^d$ by (see~\cite[Proposition 5.27 and its proof]{lee:2019})
\begin{equation} \label{eq:expsphere}
  \Exp_\theta(u) = \cos(\| u \|)\,\theta + \sin(\| u \|) (u/\norm{u}) \eqsp.
\end{equation}

In addition, the retraction $\Ret$  given by  \eqref{eq:retractionsphere} can be written as 
\begin{equation} \label{eq:phisphere1}
 \Ret_\theta(u) \,=\, \Exp_\theta\left(\arctan(\|u\|)\frac{u}{\|u\|}\right)
\end{equation}
This can be proven by replacing the identities,
$$
\cos(\arctan(\|u\|)) = \frac{1}{\sqrt{1 + \|u\|^2}} \hspace{0.5cm} \sin(\arctan(\|u\|)) = \frac{\|u\|}{\sqrt{1 + \|u\|^2}}
$$
into (\ref{eq:expsphere}). Indeed, this yields,
$$
\Exp_\theta\left(\arctan(\|u\|)\frac{u}{\|u\|}\right) \,=\, \frac{1}{\sqrt{1 + \|u\|^2}}\,\theta +  \frac{1}{\sqrt{1 + \|u\|^2}}\,u
$$
To see that this is equal to $\Ret_\theta(u)$, note that $1 + \|u\|^2 = \|\theta + u \|^2$, because $\|\theta\| = 1$ and $u$ is orthogonal to $\theta$ (since $u \in \planT_{\theta}\Theta$). Then, (\ref{eq:phisphere1}) follows from (\ref{eq:retractionsphere}). \\[0.1cm]
The following are now proven. \\[0.1cm]
$\bullet$ Condition \Cref{ass:retraction}-\ref{ass:retraction_zero} is satisfied\,: this condition is just the definition of a retraction, as given in~\cite{absil:2008}. \\[0.1cm]
$\bullet$ Condition \Cref{ass:retraction}-\ref{ass:retraction_first_a} is satisfied\,: the cut locus of a point $\theta$ on the sphere $S^d$ is $\Cut(\theta) = \lbrace-\theta\rbrace$ \cite{lee:2019} (Page 308). The Riemannian (that is, spherical) distance between $\theta$ and $-\theta$ is $\distT(\theta,-\theta) = \pi$. On the other hand, from(\ref{eq:phisphere1}), $\distT(\theta,\Ret_\theta(u)) < \frac{\pi}{2}$ because $\arctan(\|u\|) < \pi/2$ for all $u \in \planT_{\theta}\Theta$. It is then clear that $\Ret_\theta(u) \neq \lbrace-\theta\rbrace$ for any $u \in \planT_{\theta}\Theta$. \\[0.1cm]
$\bullet$ Condition \Cref{ass:retraction}-\ref{ass:retraction_invariance} is satisfied\,: the isometry group of  $\Theta = S^d$ is $G = O(d)$, the group of $d \times d$ orthogonal matrices. The action of $\msg$ on $\Theta$ is given by matrix-vector multiplication, $g\cdot \theta = g\theta$ and $g\cdot u = gu$. From (\ref{eq:retractionsphere}),
\begin{equation} \label{eq:h2sphere1}
g\cdot \Ret_\theta(u) = \frac{g\cdot (\theta + u)}{\|\theta + u\|}
\end{equation}
However, since $g$ is an orthogonal matrix, $g$ preserves Euclidean norms, so $\|\theta + u\| = \|g\cdot (\theta + u)\|$. Replacing into (\ref{eq:h2sphere1}),
\begin{equation} \label{eq:h2sphere2}
g\cdot \Ret_\theta(u) = \frac{g\cdot (\theta + u)}{\|g\cdot (\theta + u)\|} \,=\, 
\frac{g\cdot \theta + g\cdot u}{\|g\cdot \theta + g\cdot u\|}
\end{equation}
where the second equality follows since the action of $g$ is linear. Finally, the right-hand side of  (\ref{eq:h2sphere2}) is $\Ret_{g\cdot \theta}(g\cdot u)$. \\[0.1cm]
 $\bullet$ Condition \Cref{ass:retraction_first_b} is satisfied\,: from (\ref{eq:phisphere1}) and \Cref{ass:retraction}-\ref{ass:retraction_first_a}, 
\begin{equation} \label{eq:phisphere2}
  \Phi_\theta(u) \,=\, \arctan(\|u\|)\frac{u}{\|u\|}
\end{equation}
The required second derivative can now be computed, thanks to the identity,
\begin{equation} \label{eq:secondderivativetime}
  \rmD^2\Phi_\theta(tu)[u,u] \,=\, \frac{\rmd^2}{\rmd t^2}\Phi_\theta(tu)
\end{equation}
Indeed, using (\ref{eq:phisphere2}) and (\ref{eq:secondderivativetime}),
$$
\rmD^2\Phi_\theta(tu)[u,u] \,=\, \frac{\rmd^2}{\rmd t^2}\,\arctan(t\|u\|)\frac{u}{\|u\|} \,=\, \|u\|^2\left(f_2(t\|u\|)\frac{u}{\|u\|}\right) 
$$
where $f_2$ is the second derivative of the $\arctan$ function, so $|f_2(x)| \leq 1$ for real $x$. Now, since $\Theta = S^d$, here $\normr{u}[\theta] = \|u\|$. Thus,  Condition \Cref{ass:retraction_first_b} is satisfied with $\scrl^{(1)}(\theta) = 1$. \\[0.1cm]
$\bullet$ Condition \Cref{ass:retraction_second}-\ref{ass:retraction_second_acc} is satisfied\,: recall (\ref{eq:secondderivativetoacceleration}) from the proof of Lemma \ref{lem:retractions_gene}. This states,
$$
 \rmD_t \dot{\upgamma}(0) = \rmD^2 \Phi_{\theta}(0)[u,u] 
$$
From (\ref{eq:secondderivativetime}), it then follows,
\begin{equation} \label{eq:accproof}
 \rmD_t \dot{\upgamma}(0)\,=\, \left.\frac{\rmd^2}{\rmd t^2}\right|_{t=0}\Phi_\theta(tu)
\end{equation}
Since $\Phi_\theta$ is given by (\ref{eq:phisphere2}), an elementary calculation shows the right-hand side is here equal to zero.\\[0.1cm]
$\bullet$ Condition \Cref{ass:retraction_second}-\ref{ass:retraction_second_derivative} is satisfied\,: the proof is similar to the above one for \Cref{ass:retraction_first_b}. Here, instead of  (\ref{eq:secondderivativetime}), it is enough to use
\begin{equation} \label{eq:thirdderivativetime}
  \rmD^3\Phi_\theta(tu)[u,u,u] \,=\, \frac{\rmd^3}{\rmd t^3}\Phi_\theta(tu)
\end{equation} 
Using (\ref{eq:phisphere2}), this shows that \Cref{ass:retraction_second}-\ref{ass:retraction_second_acc} is satisfied with $\scrl^{(2)}(\theta) = 2$. \\[0.1cm]
$\bullet$ $\Phi_\theta$ has bounded first derivative\,: from (\ref{eq:phisphere2}), by differentiating,
$$
\rmD\Phi_\theta(u)[v] \,=\, \frac{1}{1+\| u \|^2}\,\frac{\langle u,v\rangle}{\| u \|}\frac{u}{\| u \|}\,+\, \frac{\arctan(\|u\|)}{\|u\|}\left(v - \frac{\langle u,v\rangle}{\| u \|}\frac{u}{\| u \|}\right)
$$
for any $u$ and $v$ in $\planT_{\theta}\Theta$. Then, by an elementary calculation, and recalling that, since $\Theta = S^d$, Riemannian scalar products and norms are equal to Euclidean ones, $\normr{\rmD\Phi_\theta(u)[v]}[\theta] \,\leq 2\normr{v}[\theta]$\,. Thus, the operator norm of $\rmD\Phi_\Theta(u)$ is bounded by $\overline{\rm D} = 2$.
\end{proof}
\subsection{Proof of \Cref{prop:retractiongrass}}
\label{sec:proof-crefpr}

First, by \cite[Equation 2.32]{edelman:arias:smith:1998}, the Riemannian exponential map at $\theta$ is given for $D \in \planT_\theta\grassmann_r(\rset^d)$, with $D = B_\perp C$, $C \in \rset^{(d-r)\times r}$:
\begin{equation} \label{eq:grassmannexp}
  \mathrm{Exp}_\theta(D) = \left[ (B\,,B_\perp) \,
\exp 
  \left( \begin{array}{cc}
  0 & -C^\top \\ C & 0 \end{array} \right) 
  \left( \begin{array}{c} {\rm I}_r \\ 0_{d-r \times r} \end{array} \right) \right]
\end{equation}
where $\exp$ is the matrix exponential.
In addition, we show below that the retraction $\Ret$ defined by  \eqref{eq:grassmannretraction} can be written on the form
\begin{equation} \label{eq:phigrass1}
  \Ret_\theta(D) \,=\, \Exp_\theta(\Phi_\theta(D)) \hspace{0.5cm}  \Phi_\theta(D) \,=\, B_\perp\,V\arctan(a)\,U^{\transpose}
\end{equation}
for $D \in \planT_{\theta}\Theta$ with $D = B_\perp C$, where $C$ has singular value decomposition $C = V\,a\,U^{\transpose}$. Here, $V$ is $(d-r) \times (d-r)$ orthogonal and $U$ is $r \times r$ orthogonal. Moreover, $\arctan(a)$ is obtained by taking the arctangent of each element of the matrix $a$. Accepting (\ref{eq:phigrass1}), it is possible to show that. \\[0.1cm]
$\bullet$ Condition and \Cref{ass:retraction}-\ref{ass:retraction_zero} is satisfied\,: this condition is just the definition of a retraction, as given in~\cite{absil:2008}. \\[0.1cm]
$\bullet$ Condition \Cref{ass:retraction}-\ref{ass:retraction_first_a} is satisfied\,: when $\Theta = \grassmann_r(\rset^d)$, the cut locus of each $\theta \in \Theta$ is given by~\cite{sakai:1976},
\begin{equation} \label{eq:cutsun}
  \Cut(\theta) \,=\, \left\lbrace \Exp_\theta(B_\perp C)\middle| C = V\,a\,U^{\transpose}\,;\,\Vert a \Vert_{\infty} \,=\, \frac{\pi}{2} \right\rbrace \eqsp,
\end{equation}
where $\Vert a \Vert_{\infty} = \max_{ij} |a_{ij}|$. From (\ref{eq:phigrass1}), for any $D \in \planT_{\theta}\Theta$, one has
$$
\Ret_\theta(D) \,=\, \Exp_\theta(B_\perp C(D)) \text{ where } C(D) = V\arctan(a)\,U^{\transpose} \eqsp.
$$
Since $\Vert \arctan(a) \Vert_\infty < \pi/2$, it follows that $\Ret_\theta(D) \notin \Cut(\theta)$. \\[0.1cm]
$\bullet$ Condition \Cref{ass:retraction}-\ref{ass:retraction_invariance} is satisfied\,: the isometry group of  $\Theta = \grassmann_r(\rset^d)$ is $G = O(d)$, the group of $d \times d$ orthogonal matrices. The action of $\msg$ on $\Theta$ is given by $g\cdot \theta = g(\theta)$ (the image of the subspace $\theta$ of $\rset^d$ by the orthogonal tranformation $g$). If $D \in \planT_{\theta} \Theta$, then $g\cdot D = gD$ is a matrix product. 

Note that, if $\theta = [B]$ for some $B \in \stiefel_r(\rset^d)$, then $g\cdot \theta = [gB]$. Applying this property in (\ref{eq:grassmannretraction}),
\begin{equation} \label{eq:grassinvariance}
g\cdot \Ret_\theta(D) \,=\, g\left(\left[ B + D\right]\right) \,=\, \left[ gB + g D\right]
\end{equation}
But, since $g\cdot \theta = [gB]$ and $gD = g\cdot D$, (\ref{eq:grassinvariance}) implies
$$
g\cdot \Ret_\theta(D) \,=\, \Ret_{g\cdot \theta}(g\cdot D)
$$

 $\bullet$ Condition \Cref{ass:retraction_first_b} is satisfied\,: the required second derivative is computed using the identity (this is a repetition of (\ref{eq:secondderivativetime})),
\begin{equation} \label{eq:secondderivativetimebis}
  \rmD^2\Phi_\theta(tD)[D,D] \,=\, \frac{\rmd^2}{\rmd t^2}\Phi_\theta(tD)
\end{equation}
Using (\ref{eq:phigrass1}) and (\ref{eq:secondderivativetimebis}), 
\begin{equation} \label{eq:pr_grass_fb1}
\rmD^2\Phi_\theta(tD)[D,D] \,=\, B_\perp\,V\left( a \odot a \odot f_2(ta)\right)U^{\transpose} 
\end{equation}
where $\odot$ denotes the Kronecker product, and $f_2$ is the second derivative of the $\arctan$ function (again, this is applied to each element of the matrix $(ta)$). From~\cite{edelman:arias:smith:1998} (Page 314)
\begin{equation} 
\normr{\rmD^2\Phi_\theta(tD)[D,D]}[\theta]^2 \,=\, \mathrm{tr}\left( ( a \odot a \odot f_2(ta))( a \odot a \odot f_2(ta))^{\transpose}\right)
\end{equation}
where $\mathrm{tr}$ denotes the trace. Then, using the fact that $|f_2(x)| \leq 1$ for real $x$, the right-hand side is less than $\mathrm{tr}(aa^{\transpose})$, which is equal to $\normr{D}[\theta]^2\,$. Thus,  Condition \Cref{ass:retraction_first_b} is satisfied with $\scrl^{(1)}(\theta) = 1$. \\[0.1cm]
$\bullet$ Condition \Cref{ass:retraction_second}-\ref{ass:retraction_second_acc} is satisfied\,: recall (\ref{eq:secondderivativetoacceleration}) from the proof of Lemma \ref{lem:retractions_gene}. This states,
$$
 \rmD_t \dot{\upgamma}(0) = \rmD^2 \Phi_{\theta}(0)[D,D] 
$$
Setting $t = 0$ in (\ref{eq:pr_grass_fb1}), it then follows
$$
 \rmD_t \dot{\upgamma}(0)  \,=\, B_\perp\,V\left( a \odot a \odot f_2(0)\right)U^{\transpose}
$$
which is equal to zero since $f_2(0) = 0$.\\[0.1cm]
$\bullet$ Condition \Cref{ass:retraction_second}-\ref{ass:retraction_second_derivative} is satisfied\,: the proof is similar to the above one for \Cref{ass:retraction_first_b}. Here, instead of  (\ref{eq:secondderivativetimebis}), it is enough to use
\begin{equation} \label{eq:thirdderivativetimebis}
  \rmD^3\Phi_\theta(tD)[D,D,D] \,=\, \frac{\rmd^3}{\rmd t^3}\Phi_\theta(tD)
\end{equation} 
by computing the derivative, as in (\ref{eq:pr_grass_fb1}), it can be shown that \Cref{ass:retraction_second} is satisfied with $\scrl^{(2)}(\theta) = 2$. \\[0.1cm]
$\bullet$ $\Phi_\theta$ does not have bounded first derivative\,: assume $r > 1$. Recall that $\Phi_\theta(D)$ is given by (\ref{eq:phigrass1}), which can be written 
\begin{equation} \label{eq:unboundedD1}
  \Phi_\theta(D) = B_\perp \varphi(\psi_\theta(D))
\end{equation}
where, $\psi_\theta : \rset^{d\times r} \rightarrow \rset^{(d-r)\times r}$ and $\varphi:\rset^{(d-r)\times r}\rightarrow \rset^{(d-r)\times r}$ are given by
\begin{equation} \label{eq:unboundedD2}
 \psi_\theta(D) = B^\transpose_\perp D\hspace{0.5cm}   \varphi(C) = V\arctan(a)\, U^\transpose
\end{equation}
whenever $C$ has singular value decomposition $C = Va\,U^{\transpose}$. Indeed, if $D = B_\perp C$, then $\psi_\theta(D) = C$, so (\ref{eq:unboundedD1}) is equivalent to (\ref{eq:phigrass1}). From (\ref{eq:unboundedD1}) and (\ref{eq:unboundedD2}), by an application of the chain rule
$$
\rmD \Phi_\theta(D)[\bar{D}] = B_\perp \rmD \varphi(C)[B^\transpose_\perp \bar{D}]
$$
for $\bar{D} \in \planT_{\theta}\Theta$, where $C = \psi_\theta(D)$. Now, to show that $\rmD \Phi_\theta(D)$ is not bounded, it is enough to show that $\rmD \varphi(C)$ is not bounded. However,
$$
\rmD\varphi(C)[w] = \rmD V[w]\left(\arctan(a)\right)U^\transpose +V\rmD\left(\arctan(a)\right)\![w]\,U^\transpose + V\left(\arctan(a)\right)\rmD U[w] 
$$
for $w \in \rset^{(d-r)\times r}$, where $\rmD V$, $\rmD\left(\arctan(a)\right)$ and $\rmD U$ denote the derivatives of $V$, $\arctan(a)$ and $U$, as functions of $C$, by an abuse of notation. To simplify the proof, assume, without loss of generality, that $C$ is a square matrix (for example, if $d-r \geq r$, it is enough to add zero columns to $C$). With this assumption, the following formulae hold~\cite{townsend},
\begin{equation} \label{eq:DV}
\rmD V[w] = V \left[ F\odot \left( V^\transpose w\,U\,a + a\,U^\transpose w^\transpose\,V\right) \right]
\end{equation}
\begin{equation} \label{eq:DatanA}
\rmD \left(\arctan(a)\right)\![w] = \rmI_{(d-r)} \left[  V^\transpose w\,U \right]
\end{equation}
\begin{equation} \label{eq:DU}
\rmD U[w] = U \left[ F\odot \left( a\,V^\transpose w\,U+ U^\transpose w^\transpose\,V\,a\right) \right]
\end{equation}
where $F$ is the matrix with entries $F_{ij} = (a^2_j - a^2_i)^{-1}$ for $i \neq j$ and $F_{ii} = 0$. However, taking $w = V\omega a U^\transpose$ where $\omega$ is a $(d-r) \times (d-r)$ antisymmetric matrix,  yields $\rmD \left(\arctan(a)\right)[w] = 0$ and $\rmD U[w] = 0$, while
$$
\rmD V[w] = V \left[ G\odot \omega\right]
$$
where $G$ has matrix elements $G_{ij} = a^2_ia^2_j/(a^2_j - a^2_i)$ for $i \neq j$ and $G_{ii} = 0$. Clearly, these do not remain bounded as $a_i - a_j \rightarrow 0$.
\paragraph{Proof of (\ref{eq:phigrass1})} here, let $s = d-r$ and assume, without any loss of generality, that $s \geq r$. 

Recall that, in (\ref{eq:phigrass1}), $D = B_\perp C$ where $C$ has singular value decomposition $C = Va\,U^{\transpose}$. Here, $V$ and $U$ are orthogonal, and $a = (\alpha\,,0_{r\times s-r})^{\transpose}$, with $r \times r$ diagonal matrix  $\alpha$. Write $\Phi_\theta(D)$ under the form
\begin{equation} \label{eq:cu}
\Phi_\theta(D) = B_\perp C(D) \text{ where } C(D) = V\arctan(a)\,U^{\transpose}
\end{equation}
Using (\ref{eq:grassmannexp}), it follows
\begin{equation} \label{eq:grassmannexp1}
 \Exp_\theta(\Phi_\theta(D)) \,=\, 
\left[ Q \,
 \exp 
  \left( \begin{array}{cc}
  0 & -C(D)^\top \\ C(D) & 0 \end{array} \right) 
  \left( \begin{array}{c} {\rm I}_r \\ 0_{s \times r} \end{array} \right) \right]
\end{equation}
where $Q = (B\,,B_\perp)$. The aim is to show this is equal to $\Ret_\theta(D)$, given by (\ref{eq:grassmannretraction}). Using the expression of $C(D)$ in (\ref{eq:cu}), and performing the matrix multiplication, it is possible to check that
\begin{equation} \label{eq:svdtok}
  \left( \begin{array}{cc}
  0 & -C(D)^\top \\ C(D) & 0 \end{array} \right) \,=\, 
\left(\begin{array}{cc} U & \\  & V\end{array}\right)
  \left( \begin{array}{cc}
  0 & -\arctan(a)^\top \\ \arctan(a) & 0 \end{array} \right) 
\left(\begin{array}{cc} U^{\transpose} & \\  & V^{\transpose}\end{array}\right)
\end{equation}
Recall $\exp(AXA^{-1}) = A\exp(X)A^{-1}$ for any square matrices $A$ and $X$, where $A$ is invertible. It follows from (\ref{eq:svdtok}),
$$
  \exp\left( \begin{array}{cc}
  0 & -C(D)^\top \\ C(D) & 0 \end{array} \right) \,=\, 
\left(\begin{array}{cc} U & \\  & V\end{array}\right)
  \exp\left( \begin{array}{cc}
  0 & -\arctan(a)^\top \\ \arctan(a) & 0 \end{array} \right) 
\left(\begin{array}{cc} U^{\transpose} & \\  & V^{\transpose}\end{array}\right)
$$
By plugging this into (\ref{eq:grassmannexp1}), and noticing that
$$
\left[  \left(\begin{array}{cc} U^{\transpose} & \\  & V^{\transpose}\end{array}\right)
  \left( \begin{array}{c} {\rm I}_r \\ 0_{s \times r} \end{array} \right) \right]
=
\left[ 
  \left( \begin{array}{c} {\rm I}_r \\ 0_{s \times r} \end{array} \right) \right]
$$
it follows
\begin{equation} \label{eq:grassmannexp2}
 \Exp_\theta(\Phi_\theta(D)) \,=\, 
\left[ Q \,
\left(\begin{array}{cc} U & \\  & V\end{array}\right)
  \exp\left( \begin{array}{cc}
  0 & -\arctan(a)^\top \\ \arctan(a) & 0 \end{array} \right)  \left( \begin{array}{c} {\rm I}_r \\ 0_{s \times r} \end{array} \right) \right]
\end{equation} 
If $f = (\phi\,,0_{r\times s-r})^{\transpose}$ where $\phi$ is $r\times r$ diagonal, then, under the assumption that $s \geq r$, (this is proven in detail, at the end of the present proof),
\begin{equation} \label{eq:expidentity}
\exp\left(\begin{array}{cc} 0 & -f^{\transpose} \\ f & 0 \end{array}\right) \left( \begin{array}{c} {\rm I}_r \\ 0_{s \times r} \end{array} \right) \,=\,
\left( \begin{array}{c} C(\phi) \\[0.1cm] S(f) \end{array} \right)
\end{equation}
where $C(\phi) = \cos(\phi)$ and $S(f) = \sin(f)$, with the functions $\cos$ and $\sin$ applied to each matrix element of $\phi$ and $f$, respectively. The identity (\ref{eq:expidentity}) can be used to evaluate the matrix exponential in (\ref{eq:grassmannexp2}), since $a = (\alpha\,,0_{r\times s-r})^{\transpose}$. This yields,
$$
 \Exp_\theta(\Phi_\theta(D)) \,=\, 
\left[ Q \,
\left(\begin{array}{cc} U & \\  & V\end{array}\right)
 \left( \begin{array}{c} \cos(\arctan(\alpha)) \\ \sin(\arctan(a)) \end{array} \right) \right]
$$
However, since $\cos(\arctan(x)) = 1/(1+x^2)^{1/2}$ and $\sin(\arctan(x)) = x/(1+x^2)^{1/2}$, this becomes
\begin{equation} \label{eq:expgrassmann3}
 \Exp_\theta(\Phi_\theta(D)) \,=\, 
\left[ Q \,
\left(\begin{array}{cc} U & \\  & V\end{array}\right)
 \left( \begin{array}{c} {\rm I}_r \\ a \end{array} \right)\,({\rm I}_r + \alpha)^{-1/2} \right] \,=\,
\left[ Q \,
\left(\begin{array}{cc} U & \\  & V\end{array}\right)
 \left( \begin{array}{c} {\rm I}_r \\ a \end{array} \right) \right]
\end{equation}
where the second equality holds because $({\rm I}_r + \alpha)$ is invertible (the diagonal elements of $\alpha$ are the singular values of $C$, and are therefore positive). It follows from (\ref{eq:expgrassmann3}) that
\begin{equation} \label{eq:expgrassmann4}
 \Exp_\theta(\Phi_\theta(D)) \,=\, 
\left[ Q \,
 \left( \begin{array}{c} U \\ Va \end{array} \right) \right] \,=\, \left[ Q \,
 \left( \begin{array}{c} U \\ Va \end{array} \right)U^{\transpose} \right]
\end{equation}
where the second equality holds because $U^{\transpose}$ is an invertible $r \times r$ matrix (which therefore does not change the span of the columns of the overall matrix product). Performing the matrix product in (\ref{eq:expgrassmann4}), and noting $UU^{\transpose} = {\rm I_r}$ and $C = V a\,U^{\transpose}$, it finally follows that
$$
 \Exp_\theta(\Phi_\theta(D)) \,=\,  \left[ Q \,
 \left( \begin{array}{c} {\rm I}_r \\ C \end{array} \right) \right] \,=\,
 \Exp_\theta(\Phi_\theta(D)) \,=\,  \left[ (B\,,B_\perp) \,
 \left( \begin{array}{c} {\rm I}_r \\ C \end{array} \right) \right]
$$
From $D = B_\perp C$, this immediately implies 
$$
\Exp_\theta(\Phi_\theta(D)) = \left[ B + D\right]
$$ 
which means $\Exp_\theta(\Phi_\theta(D)) = \Ret_\theta(D)$, as required in (\ref{eq:phigrass1}). \\[0.1cm]
\paragraph{Proof of (\ref{eq:expidentity})} this follows from
\begin{equation} \label{eq:expidentity1}
\exp\left(\begin{array}{cc} 0 & -f^{\transpose} \\ f & 0 \end{array}\right) \,=\,
\left( \begin{array}{ccc} C(\phi)  & -S(f)^{\transpose}\\[0.1cm] S(f) & C(\phi)\end{array} \right)
\end{equation}
where $C(\phi)$ and $S(f)$ are as in (\ref{eq:expidentity}) and where (recall it is assumed that $s \geq r$),
$$
C(\phi) = \left(\begin{array}{cc} \cos(\phi) & 0_{r\times s-r} \\ 0_{s-r\times r} & 0_{s-r\times s-r}\end{array}\right)
$$ 
To prove (\ref{eq:expidentity1}), write
\begin{equation} \label{eq:expidentity2}
\left(\begin{array}{cc} 0 & -f^{\transpose} \\ f & 0 \end{array}\right) \,=\, \sum^r_{i=1}\,\phi_i\,\mathrm{b}_{r+i}
\end{equation}
where $f = (\phi\,,0_{r\times s-r})^{\transpose}$ with diagonal $\phi$, and where 
$$
\mathrm{b}_{r+i} = e_{r+i,i} - e_{i,r+i}
$$
with $e_{j,k}$ a matrix all of whose elements are zero, except the one at row $j$ and column $k$, which is equal to $1$. One easily checks the matrices $\mathrm{b}_{r+i}$ commute with each other. Therefore, (\ref{eq:expidentity2}) implies
\begin{equation} \label{eq:expidentity3}
\exp\left(\begin{array}{cc} 0 & -f^{\transpose} \\ f & 0 \end{array}\right) \,=\, \prod^r_{i=1}\,\exp(\phi_i\,\mathrm{b}_{r+i})
\end{equation}
However, it is elementary that
\begin{equation} \label{eq:expidentity4}
\exp(\phi_i\,\mathrm{b}_{r+i}) \,=\, {\rm I}_d + (\cos(\phi_i)-1)\,\mathrm{a}_{r+i} \,+\,\sin(\phi_i)\,\mathrm{b}_{r+i}
\end{equation}
where ${\rm I}_d$ is the $d \times d$ identity matrix and 
$$
\mathrm{a}_{r+i} = e_{i,i}+e_{r+i,r+i}
$$
Finally, (\ref{eq:expidentity1}) follows from (\ref{eq:expidentity3}) and (\ref{eq:expidentity4}), after noting the matrix products, for $i \neq k$,
$$
\begin{array}{ccc}
\mathrm{a}_{r+i}\mathrm{a}_{r+k} = 0 &&  \mathrm{a}_{r+i}\mathrm{b}_{r+k} = 0 \\[0.1cm]
\mathrm{b}_{r+i}\mathrm{a}_{r+k} = 0 && \mathrm{b}_{r+i}\mathrm{b}_{r+k} = 0
\end{array}
$$
which can be checked immediately. 



\section{Proofs for \Cref{sec:analys-gener-retr}}

\subsection{A technical lemma}
We preface our proofs by a version of \Cref{lem:first_inequality_lemma} 
adapted to the new scheme \eqref{eq:newschemexp}.

\begin{llemma}
  \label{lem:ret_inequality_lemma}
  Assume \Cref{ass:completeness}-\Cref{ass:lyap_mean_field}-\Cref{ass:retraction}. Consider the sequence $(\theta_n)_{n \in nsets}$ satisfying the scheme \eqref{eq:newschemexp}. Setting, for any $k\in \nset$, $\Delta M_k=
  \psrLigne{\grad V(\theta_k)}{\noise_{\theta_k}(X_{k+1})}[\theta_k]$, we have for any $n \in \nsets$ and $\varepsilon >0$,
  \begin{align}
      &\textstyle{\sum_{k=0}^{n} \upeta_{k+1} (\cl_1-2L
          \upeta_{k+1}-\cu^2\varepsilon) 
 \normrLigne{h(\theta_k)}[\theta_k]^2  \leq V(\theta_0) - V(\theta_{n+1})
 +  \sum_{k=0}^n \upeta_{k+1}\Delta M_k+\cl_2 \Gamma_{n+1}} \\
      &\qquad \qquad\qquad\qquad\textstyle{
   \qquad   + \sum_{k=0}^n \upeta_{k+1} \normrLigne{\grad V
     (\theta_k)}[\theta_k]  \normrLigne{\Delta_{\theta_k,
 \upeta_{k+1}}(X_{k+1})}[\theta_k] }\\
  &\qquad \qquad\qquad\qquad\textstyle{ \qquad+ 2L \sum_{k=0}^n \upeta_{k+1}^2 (
     \normrLigne{\noise_{\theta_k}(X_{k+1})}[\theta_k]^2  
     + \normrLigne{\Delta_{\theta_k,
 \upeta_{k+1}}(X_{k+1})}[\theta_k]^2) }\\
             &\qquad\qquad\qquad\qquad\textstyle{
  \qquad+ \sum_{k=0}^n \upeta_{k+1}\defEns{\parentheseLigne{4\varepsilon}^{-1}
 + 2L\upeta_{k+1}}\normrLigne{b_{\theta_k}(X_{k+1})}[\theta_k]^2 }
 \eqsp,
  \end{align}
  where $e,\Delta$  and are defined by \eqref{eq:markovh} and \eqref{eq:Deltan} respectively.
\end{llemma}

\begin{proof}
    For any $k \geq 0$ and $t \in \ccint{0,1}$, consider
    \begin{equation}
        \upgamma^{(k)}(t) =
        \Exp_{\theta_k}\parenthese{ t \upeta_{k+1}
            \defEns{H_{\theta_k}(X_{k+1})+b_{\theta_k}(X_{k+1})
        + \Delta_{\theta_k, \upeta_{k+1}}(X_{k+1}) } } \eqsp .
    \end{equation}
    Note that $\dot{\upgamma}^{(k)}(0) = \upeta_{k+1}
    \{H_{\theta_k}(X_{k+1})+b_{\theta_k}(X_{k+1}) +\Delta_{\theta_k,
    \upeta_{k+1}}(X_{k+1})\}$ and that \\ $\ell(\upgamma^{(k)}) =
    \upeta_{k+1}\normrLigne{H_{\theta_k}(X_{k+1})+b_{\theta_k}(X_{k+1})  +
        \Delta_{\theta_k,
    \upeta_{k+1}}(X_{k+1}) }[\theta_k]$.
    Then, by  \Cref{lem:taylor_grad_lip}, \eqref{eq:markovh} and using
    that for any $\theta \in \Theta$, $a,b,c,d \in \rmT_{\theta}
    \Theta$,
    $\normrLigne{a+b+c+d}[\theta]^2 \leq 4(\normrLigne{a}[\theta]^2+
    \normrLigne{b}[\theta]^2 +\normrLigne{c}[\theta]^2 + \normrLigne{d}[\theta]^2 )$ , we get that for any
    $k \geq 0$,
  \begin{align}
    \label{eq:4}
&    \abs{V(\theta_{k+1}) - V(\theta_k) - \upeta_{k+1}\psr{\grad
        V(\theta_k)}{H_{\theta_k}(X_{k+1})+b_{\theta_k}(X_{k+1})
        +\Delta_{\theta_k, \upeta_{k+1}}(X_{k+1}) }[\theta_k]} \\
& \qquad \leq (L/2) \ell(\upgamma^{(k)})^2 
  = (L \upeta_{k+1}^2 /2)
  \normrLigne{H_{\theta_k}(X_{k+1})+b_{\theta_k}(X_{k+1}) +\Delta_{\theta_k,
  \upeta_{k+1}}(X_{k+1})}[\theta_k]^2 \eqsp , \\
& \qquad \leq 2L \upeta_{k+1}^2
\defEns{\normrLigne{h(\theta_k)}[\theta_k]^2
+\normrLigne{b_{\theta_k}(X_{k+1})}[\theta_k]^2 +
\normrLigne{\noise_{\theta_k}(X_{k+1})}[\theta_k]^2  
+\normrLigne{\Delta_{\theta_k, \upeta_{k+1}}(X_{k+1})}[\theta_k]^2 } \eqsp.
  \end{align}
Therefore, we get that for any $k \in \nset$,
\begin{equation}
\label{eq:proof_lemme_descente_ret}
\begin{aligned}
	& -\upeta_{k+1}\psr{\grad V(\theta_k)}{h(\theta_k)}[\theta_k] \\
    & \quad \leq V(\theta_k) - V(\theta_{k+1}) + \upeta_{k+1}\psr{\grad
    V(\theta_k)}{\noise_{\theta_k}(X_{k+1})+b_{\theta_k}(X_{k+1})+\Delta_{\theta_k,
    \upeta_{k+1}}(X_{k+1})}[\theta_k] \\
    & \qquad +  2L \upeta_{k+1}^2
    \defEns{\normrLigne{h(\theta_k)}[\theta_k]^2
        + \normrLigne{b_{\theta_k}(X_{k+1})}[\theta_k]^2
        +\normrLigne{\noise_{\theta_k}(X_{k+1})}[\theta_k]^2 +
    \normrLigne{\Delta_{\theta_k, \upeta_{k+1}}(X_{k+1})}[\theta_k]^2} \eqsp. 
  \end{aligned}
  \end{equation}
By \Cref{ass:lyap_mean_field}-\ref{ass:lyap_mean_field_a} and the Cauchy-Schwarz inequality, we obtain for any $k \in \nsets$ and $\varepsilon>0$,
\begin{equation}
\psrLigne{\grad V(\theta_k)}{b_{\theta_k}(X_{k+1})}[\theta_k] 
\leq \cu^2 \varepsilon \normrLigne{h(\theta_k)}[\theta_k]^2 + (1/4\varepsilon) \normrLigne{b_{\theta_k}(X_{k+1})}[\theta_k] \eqsp .
\end{equation}
Thus, using \Cref{ass:lyap_mean_field}-\ref{ass:lyap_mean_field_b}, the Cauchy-Schwarz inequality and plugging this in
\eqref{eq:proof_lemme_descente_ret} gives,
  \begin{multline}
 \upeta_{k+1}(\cl_1-2L\upeta_{k+1} - \cu^2\varepsilon)\normrLigne{h(\theta_k)}[\theta_k]^2 \leq V(\theta_k) - V(\theta_{k+1}) 
 + \upeta_{k+1}\psr{\grad
 V(\theta_k)}{\noise_{\theta_k}(X_{k+1})}[\theta_k] \\
 + \upeta_{k+1} \normrLigne{\grad V(\theta_k)}[\theta_k] 
 \normrLigne{\Delta_{\theta_k, \upeta_{k+1}}(X_{k+1})}[\theta_k]
 + 2L \upeta^2_{k+1} \normrLigne{\Delta_{\theta_k,
 \upeta_{k+1}}(X_{k+1})}[\theta_k]^2
 \\
 + [\upeta_{k+1}/4\varepsilon  + 2L
 \upeta_{k+1}^2]\normrLigne{b_{\theta_k}(X_{k+1})}[\theta_k]^2 
 + 2L \upeta_{k+1}^2 \normrLigne{\noise_{\theta_k}(X_{k+1})}[\theta_k]^2  +
 \upeta_{k+1} \cl_2 \eqsp. 
  \end{multline}
  Adding these inequalities from $0$ to $n$ and rearranging terms concludes the proof.
\end{proof}

\subsection{Proof of  \Cref{prop:retmartingale}} \label{app:proofretmartingale}

If \Cref{ass:retraction_martingale}($a$) holds for $a \in \nsets$, then \Cref{ass:retraction_martingale}($\tilde{a}$) holds for any $\tilde{a} \in \{0,1,...,a\}$ and $\moment_{(\tilde{a})}$ will then stand for a constant such that almost surely, $\mathbb{E}[\Vert \noise_{\theta_n}(X_{n+1}) \Vert^{\tilde{a}}_{\theta_n}|\mathcal{F}_n]\,\leq \moment_{(\tilde{a})}$. 
Before giving the proof of   \Cref{prop:retmartingale}, we specify the statement of this result in \Cref{prop:retmartingale_supp} below. In particular, we give an explicit expression of $A^{(R)}_{n+1}$.
We define
\begin{equation}
  \label{eq:def_bupeta_marting_retract}
  \bupeta =
  \begin{cases}
\cl_1[4(2L+3\scrl^{(1)}\cu \hinfty)]^{-1}   & \text{ if \Cref{ass:retraction_first_b}-\Cref{ass:retraction_martingale}($4$) hold } \eqsp,\\
  \cl_1/(8L)  & \text{ if \Cref{ass:retraction_second}-
      \Cref{ass:retraction_martingale}($6$) hold } \eqsp,
  \end{cases}
\end{equation}

\begin{align}
  \label{eq:def_C_R_2}
  C^{(2)}_R &=
              \begin{cases}
              2\cl_1^{-1} \parentheseLigne{2 L+  3 \cu h_{\infty} \scrl^{(1)}(\ttheta) } \parentheseLigne{   b^2_{\infty} +
                \moment_{(2)} } & \text{ if \Cref{ass:retraction_first_b}-\Cref{ass:retraction_martingale}($4$) hold } \eqsp,\\
              4\cl_1^{-1} L \parentheseLigne{   b^2_{\infty} +
                \moment_{(2)} } & \text{ if \Cref{ass:retraction_second}-
      \Cref{ass:retraction_martingale}($6$) hold } \eqsp,
              \end{cases}
              \\
    \label{eq:def_C_R_4}
  C^{(4)}_R & = 54 L \cl_1^{-1} (\scrl^{(1)}(\ttheta))^2 \defEnsLigne{ h_{\infty}^4 +
              b^4_{\infty} + \moment_{(4)}}\eqsp, \\
  \label{eq:def_C_R_3}
  C^{(3)}_R & = 3^2 \scrl^{(2)}(\ttheta)  \cu \hinfty   \defEns{ \hinfty^3 + b_\infty^3 + \moment_{(3)}  } \eqsp,\\
  \label{eq:def_C_R_6}
  C^{(6)}_R & = 3^5 \scrl^{(2)}(\ttheta)\{\binfty^6 + \hinfty^6 + \moment_{(6)}\} \eqsp, 
\end{align}
where $\ttheta \in \Theta$ is fixed and can be chosen arbitrary. 

\begin{theorem} \label{prop:retmartingale_supp}
Assume
\Cref{ass:completeness}-\Cref{ass:lyap_mean_field}-\Cref{ass:bounded_bias}-\Cref{ass:bounded_vh},
\Cref{ass:retraction}, \Cref{ass:0mean_noise} hold. Suppose either
\Cref{ass:retraction_first_b},
      \Cref{ass:retraction_martingale}($4$) or
      \Cref{ass:retraction_second},
      \Cref{ass:retraction_martingale}($6$).
      Consider
$(\theta_k)_{k \in \nset}$ defined by \eqref{eq:sa}. Then, if $\sup_{k \in \nsets}
\upeta_k \leq \bar{\upeta}$, for $\bupeta$ defined in \eqref{eq:def_bupeta_marting_retract}, for any $n \in\nset$
    \begin{align} \label{eq:retmartingale1}
& \expe{ \normr{ h( \theta_{I_n} ) }[\theta_{I_n}]^2 } \leq  2(\cl_1
\Gamma_{n+1})^{-1}  \defEnsLigne{ \expe{V(\theta_0)} + C_R^{(2)}\Gamma_{n+1}^{(2)} + A^{(R)}_{n+1}}  + 2(b_\infty \cu h_{\infty} + \cl_2)/\cl_1
  \eqsp,\\
      \label{eq:def_A_R}
& \text{ with } A^{(R)}_{n+1} =
\begin{cases}
  C_R^{(4)} \Gamma_{n+1}^{(4)}  & \text{ if \Cref{ass:retraction_first_b},
      \Cref{ass:retraction_martingale}($4$) hold, $C_R^{(4)}$ is defined in \eqref{eq:def_C_R_4}} \\
  C_R^{(3)} \Gamma_{n+1}^{(3)}  +   C_R^{(6)} \Gamma_{n+1}^{(6)}& \text{ if \Cref{ass:retraction_second},
      \Cref{ass:retraction_martingale}($6$) hold, $C_R^{(3)},  C_R^{(6)}$ are defined in \eqref{eq:def_C_R_3}-\eqref{eq:def_C_R_6}} \eqsp,
\end{cases}
\end{align}
$C^{(2)}_R$ is given by \eqref{eq:def_C_R_2} and $I_n \in \iint{0}{n}$ is a random variable independent of $\calF_n$ and
distributed according to \eqref{eq:nn}.
 \vspace{.1cm}
\end{theorem}

\begin{proof}[Proof of \Cref{prop:retmartingale_supp}]
  In all the proof $\ttheta \in \Theta$ is a fixed element of $\Theta$, which will be used as we will apply \Cref{lem:retractions}. Also for ease of notation, we simply denote $\scrl_\infty^{(1)} = \scrl^{(1)}(\ttheta)$ and similarly $\scrl^{(2)} = \scrl^{(2)}(\ttheta)$.

  Using the assumptions in \Cref{prop:retmartingale}, the proof consists in 
    bounding each term of \Cref{lem:ret_inequality_lemma} after taking the
    expectation. We first make a first estimate which holds if either  \Cref{ass:retraction_first_b}
or \Cref{ass:retraction_second}
hold.

Using \Cref{ass:0mean_noise} and that
    $\theta_k$ is $\calF_k$-measurable, we have
    for any $k \in \nset$, $\expe{\Delta M_k} = 0$.
In addition,
by \Cref{ass:lyap_mean_field}-\ref{ass:lyap_mean_field_b}-\Cref{ass:bounded_vh}
we have 
\begin{equation}
\label{eq:firstret1}
\normrLigne{ \grad V(\theta_k) }[\theta_k] \normrLigne{ \Delta_{\theta_k,
\upeta_{k+1}}(X_{k+1}) }[\theta_k]
\leq \cu h_\infty \normrLigne{ \Delta_{\theta_k, \upeta_{k+1}}(X_{k+1}) }[\theta_k] \eqsp.
\end{equation}
Therefore, taking expectation in \Cref{lem:ret_inequality_lemma}, we get  using $V$ is non-negative, \Cref{ass:bounded_bias} and \Cref{ass:retraction_martingale}$(4)$ or \Cref{ass:retraction_martingale}$(6)$,
for any $n \in \nsets$ and $\varepsilon >0$,
\begin{equation}
  \label{proof:ret_theo_martingale_1}
  \begin{aligned}
      &\textstyle{\sum_{k=0}^{n} \upeta_{k+1} (\cl_1-2L
          \upeta_{k+1}-\cu^2\varepsilon) 
 \PE[\normrLigne{h(\theta_k)}[\theta_k]^2]  \leq \PE[V(\theta_0)]
 + \cl_2 \Gamma_{n+1}} \\
      &\qquad \qquad\qquad\qquad\textstyle{
   \qquad   + \cu \hinfty \sum_{k=0}^n   \upeta_{k+1}\normrLigne{\Delta_{\theta_k,
 \upeta_{k+1}}(X_{k+1})}[\theta_k] }+ 2L(\binfty^2 + \moment_{(2)}) \Gamma_{n+1}^{(2)}\\
  &\qquad \qquad\qquad\qquad\textstyle{ \qquad+ \sum_{k=0}^n \upeta_{k+1}^2 
      \normrLigne{\Delta_{\theta_k,
 \upeta_{k+1}}(X_{k+1})}[\theta_k]^2}\textstyle{
+ [\binfty^{2}/(4\varepsilon)] \Gamma_{n+1}}
 \eqsp.
  \end{aligned}
\end{equation}
Therefore, it remains to bound the terms involving the retraction bias $\Delta$, for which we distinguish the two different sets of conditions  \Cref{ass:retraction_first_b}-\Cref{ass:retraction_martingale}($4$) or  \Cref{ass:retraction_second},
\Cref{ass:retraction_martingale}($6$).

\textbf{If \Cref{ass:retraction_first_b}-\Cref{ass:retraction_martingale}($4$) hold.}
Using
\Cref{lem:retractions}-\ref{lem_retraction_a},
\Cref{ass:bounded_bias}-\Cref{ass:bounded_vh} and Hölder inequality shows that 
\begin{equation} \label{eq:firstret2}
\begin{aligned}
\normrLigne{ \Delta_{\theta_k, \upeta_{k+1}}(X_{k+1}) }[\theta_k] & \leq \scrl^{(1)} \upeta_{k+1} \normrLigne{ H_{\theta_{k}}(X_{k+1})+b_{\theta_{k}}(X_{k+1}) }[\theta_k]^2 \eqsp, \\
& \leq 3 \scrl^{(1)} \upeta_{k+1} \defEns{ \normrLigne{ h(\theta_k)
}[\theta_k]^2 + b_\infty^2 + \normrLigne{ e_{\theta_k} ( X_{k+1} )
}[\theta_k]^2 }\\
\normrLigne{ \Delta_{\theta_k, \upeta_{k+1}}(X_{k+1}) }[\theta_k]^2 &\leq 9 \times 3[ \scrl^{(1)}]^2 \upeta_{k+1}^2 \defEns{\hinfty^4 + b_\infty^4 + \normrLigne{ e_{\theta_k} ( X_{k+1} )
}[\theta_k]^4 } \eqsp. 
\end{aligned}
\end{equation}
Therefore, plugging these estimates in  \eqref{proof:ret_theo_martingale_1}, using  \Cref{ass:retraction_martingale}$(4)$,  we  get that 
\begin{equation}
  \label{proof:ret_theo_martingale_2}
  \begin{aligned}
      &\textstyle{\sum_{k=0}^{n} \upeta_{k+1} (\cl_1-2L
          \upeta_{k+1}-\cu^2\varepsilon) 
 \PE[\normrLigne{h(\theta_k)}[\theta_k]^2]  \leq \PE[V(\theta_0)]
 + \cl_2 \Gamma_{n+1}} \\
      &\qquad \qquad\qquad\qquad\textstyle{
   + 3 \scrl^{(1)}  \cu \hinfty   \sum_{k=0}^n   \upeta_{k+1}^2 \defEns{ \PE[\normrLigne{ h(\theta_k)
}[\theta_k]^2] + b_\infty^2 + \moment_{(2)}  } }+ 2L(\binfty^2 + \moment_{(2)}) \Gamma_{n+1}^{(2)}\\
  & \qquad\qquad\qquad\textstyle{ \qquad+ 27 \scrl^{(1)}\{\binfty^4 + \hinfty^4 + \moment_{(4)}\} \Gamma_{n+1}^{(4)}} \textstyle{
+ [\binfty^{2}/(4\varepsilon)] \Gamma_{n+1}}
 \eqsp.
  \end{aligned}
\end{equation}
Rearranging terms and taking $\varepsilon = \cl_1/(4\cu^2)$, we get
\begin{equation}
  \label{proof:ret_theo_martingale_3}
  \begin{aligned}
      &\textstyle{\sum_{k=0}^{n} \upeta_{k+1} (3\cl_1/4-(2L+3\scrl^{(1)}\cu \hinfty)
          \upeta_{k+1}) 
 \PE[\normrLigne{h(\theta_k)}[\theta_k]^2]  \leq \PE[V(\theta_0)]
 + (\cl_2+\binfty^{2}\cu^2/\cl_1) \Gamma_{n+1}} \\
      &\qquad \qquad \qquad+ (2L+3 \scrl^{(1)}  \cu \hinfty )(\binfty^2 + \moment_{(2)}) \Gamma_{n+1}^{(2)}\textstyle{+ 27 \scrl^{(1)}\{\binfty^4 + \hinfty^4 + \moment_{(4)}\} \Gamma_{n+1}^{(4)}}  \eqsp.
  \end{aligned}
\end{equation}

Using $\sup_{k \in\nsets} \upeta_k \leq \bupeta \leq \cl_1[4(2L+3\scrl^{(1)}\cu \hinfty)]^{-1}$ ensures that $\cl_1/4 \geq (2L+3\scrl^{(1)}\cu \hinfty)\upeta_{k+1} $ for any $k \in \nset$ and the proof follows.

\textbf{If  \Cref{ass:retraction_second},
  \Cref{ass:retraction_martingale}($6$) hold.}
Using
\Cref{lem:retractions}-\ref{lem_retraction_b},
\Cref{ass:bounded_bias}-\Cref{ass:bounded_vh} and Hölder inequality shows that 
\begin{equation} \label{eq:firstret2}
\begin{aligned}
\normrLigne{ \Delta_{\theta_k, \upeta_{k+1}}(X_{k+1}) }[\theta_k] & \leq \scrl^{(2)} \upeta_{k+1}^2 \normrLigne{ H_{\theta_{k}}(X_{k+1})+b_{\theta_{k}}(X_{k+1}) }[\theta_k]^3 \eqsp, \\
& \leq 3^2 \scrl^{(2)} \upeta_{k+1}^2 \defEns{ h_{\infty}^3+ b_\infty^3 + \normrLigne{ e_{\theta_k} ( X_{k+1} )
}[\theta_k]^3 }\\
\normrLigne{ \Delta_{\theta_k, \upeta_{k+1}}(X_{k+1}) }[\theta_k]^2 &\leq 3^5[ \scrl^{(2)}]^2 \upeta_{k+1}^2 \defEns{\hinfty^6 + b_\infty^6 + \normrLigne{ e_{\theta_k} ( X_{k+1} )
}[\theta_k]^6 } \eqsp. 
\end{aligned}
\end{equation}
Therefore, plugging these estimates in  \eqref{proof:ret_theo_martingale_1}, using   \Cref{ass:retraction_martingale}$(6)$, we  get that 
\begin{equation}
  \label{proof:ret_theo_martingale_2_R3}
  \begin{aligned}
      &\textstyle{\sum_{k=0}^{n} \upeta_{k+1} (\cl_1-2L
          \upeta_{k+1}-\cu^2\varepsilon) 
 \PE[\normrLigne{h(\theta_k)}[\theta_k]^2]  \leq \PE[V(\theta_0)]
 + \cl_2 \Gamma_{n+1}} \\
      &\qquad \qquad\qquad\qquad\textstyle{
   + 3^2 \scrl^{(2)}  \cu \hinfty   \defEns{ \hinfty^3 + b_\infty^3 + \moment_{(3)}  } \Gamma_{n+1}^{(3)} }+ 2L(\binfty^2 + \moment_{(2)}) \Gamma_{n+1}^{(2)}\\
  & \qquad\qquad\qquad\textstyle{ \qquad+ 3^5 \scrl^{(2)}\{\binfty^6 + \hinfty^6 + \moment_{(6)}\} \Gamma_{n+1}^{(6)}} \textstyle{
+ [\binfty^{2}/(4\varepsilon)] \Gamma_{n+1}}
 \eqsp.
  \end{aligned}
\end{equation}
Rearranging terms and taking $\varepsilon = \cl_1/(4\cu^2)$, we get
\begin{equation}
  \label{proof:ret_theo_martingale_2_R3}
  \begin{aligned}
      &\textstyle{\sum_{k=0}^{n} \upeta_{k+1} (3\cl_1/4-2L
          \upeta_{k+1}) 
 \PE[\normrLigne{h(\theta_k)}[\theta_k]^2]  \leq \PE[V(\theta_0)]
 + (\cl_2+\binfty^{2}\cu^2/\cl_1) \Gamma_{n+1}} \\
 &\qquad \qquad \qquad \qquad\qquad \qquad + 2L(\binfty^2 + \moment_{(2)}) \Gamma_{n+1}^{(2)}\textstyle{
   + 3^2 \scrl^{(2)}  \cu \hinfty   \defEns{ \hinfty^3 + b_\infty^3 + \moment_{(3)}  } \Gamma_{n+1}^{(3)} }\\
 & \qquad \qquad\qquad \qquad\qquad \qquad
 \textstyle{+ 3^5 \scrl^{(2)}\{\binfty^6 + \hinfty^6 + \moment_{(6)}\} \Gamma_{n+1}^{(6)}} \eqsp.
\end{aligned}
\end{equation}
Using that $\sup_{k \in\nsets} \upeta_k \leq \bupeta \leq \cl_1/(8L)$ ensures that $\cl_1/4 \geq 2L \upeta_{k+1}$ for any $k \in \nset$ which completes the proof. 
\end{proof}
\subsection{Proof of \Cref{th:retmarkov}} \label{app:proofretmarkov}

Under \Cref{ass:w_moment}$(w)$,  for any $p \in \ccint{1,3}$,  $C^{(p)}_w$ stands for a constant such that 
\begin{equation}
  \label{eq:def_C_p_w}
\sup_{k \in \nset} \expe{
              w^p(X_{k+1}) } \leq C^{(p)}_w \eqsp. 
          \end{equation}        
          In addition, in the sequel, $\ttheta \in \Theta$ is a fixed element of $\Theta$, which will be used as we will apply \Cref{lem:retractions}.
          Also for ease of notation, we simply denote $\scrl_\infty^{(1)} = \scrl^{(1)}(\ttheta)$ and similarly $\scrl^{(2)} = \scrl^{(2)}(\ttheta)$.          

Before giving the proof of   \Cref{th:retmarkov},          we specify the statement of this result in \Cref{th:retmarkov_supp} below. In particular, we give an explicit expression of $B^{(R)}_{n+1}$.
We define
\begin{equation}
  \label{eq:def_bupeta_marting_retract_markov}
  \bupeta =
  \begin{cases}
   \cl_1/[4(2L+D_{\hnoise}+3 \scrl^{(1)}\cu \hinfty)]^{-1} & \text{ if \Cref{ass:retraction_first_b} holds } \eqsp,\\
   \cl_1/[4(2L+D_{\hnoise})]^{-1}    & \text{ if \Cref{ass:retraction_second} holds } \eqsp,
  \end{cases}
\end{equation}

\begin{align}
  \label{eq:def_D_R_2}
  D^{(2)}_R & =
              \begin{cases}
                C^0_{\hnoise}+ (\binfty^2 + C_w e_{\infty}^2) (3 \cu \hinfty  \scrl^{(1)}+2L) & \text{ if \Cref{ass:retraction_first_b} holds } \eqsp,\\
                C^0_{\hnoise}+ 2L(\binfty^2 + C_w e_{\infty}^2)   & \text{if \Cref{ass:retraction_second} holds}  \eqsp,                
              \end{cases}
  \\
    \label{eq:def_D_R_3_2}
  D^{(3)}_{R2} & = 3 \scrl^{(1)} [L\hnoise_{\infty} C_w\{\hinfty^2+\binfty^2+e_{\infty}^2 C_w^{(2)}\} + L_{\hnoise}\cu \hinfty\{\hinfty^2+\binfty^2+e_{\infty}^2 C_w \}] \eqsp,\\
    \label{eq:def_D_R_4}
  D^{(4)}_R & = 3^2 \scrl^{(2)} [L\hnoise_{\infty} C_w\{\hinfty^3+\binfty^3+e_{\infty}^3 C_w^{(3)}\} + L_{\hnoise}\cu \hinfty\{\hinfty^3+\binfty^3+e_{\infty}^3 C_w^{(3/2)} \}] \\
  \label{eq:def_D_R_3_3}
  D^{(3)}_{R3} & = 3^2 \scrl^{(2)} \cu \hinfty \defEns{ h_{\infty}^3+ b_\infty^3 + e_{\infty}^3 C_w^{(3/2)} }\\
  \label{eq:def_D_R_6}
  D^{(6)}_R & =  2 3^5 L [\scrl^{(2)}]^2\{\hinfty^6 + \binfty^6 + e_{\infty}^6 C_w^{(3)}\} \eqsp. 
\end{align}

\begin{theorem} \label{th:retmarkov_supp}
Assume
\Cref{ass:completeness}-\Cref{ass:lyap_mean_field}-\Cref{ass:bounded_bias}-\Cref{ass:bounded_vh},
\Cref{ass:retraction},
\Cref{ass:markov}-\Cref{ass:w_markov}$(w)$-\Cref{ass:w_moment}$(w)$ hold for some measurable
function $w : \msx \to \cointLigne{1,\plusinfty}$. Suppose in addition that \Cref{ass:retraction_first_b} or \Cref{ass:retraction_second} holds. Assume that
$(\upeta_k)_{k \in \nsets}$ is a
sequence of stepsizes and $a_1,a_2 \geq 0$ satisfying \eqref{eq:hyp_gamma_k}.

  Consider $(\theta_k)_{k \in \nset}$ defined by \eqref{eq:sa}.  Then, if $\sup_{k \in \nsets}
  \upeta_k \leq \bar{\upeta}$, for $\bupeta$ defined in \eqref{eq:def_bupeta_marting_retract_markov}, for any $n \in\nset$
    \begin{align} \label{eq:retmartingale1_markov}
& \expeLigne{ \normr{ h( \theta_{I_n} ) }[\theta_{I_n}]^2 } \leq  2(\cl_1
\Gamma_{n+1})^{-1}  \defEnsLigne{ \expe{V(\theta_0)} +C(\upeta_1)+ D_R^{(2)}\Gamma_{n+1}^{(2)} + B^{(R)}_{n+1}}  + 2(b_\infty \cu h_{\infty} + \cl_2)/\cl_1
  \eqsp,\\
      \label{eq:def_B_R}
& \text{ with } B^{(R)}_{n+1} =
\begin{cases}
 D_{R2}^{(3)} \Gamma_{n+1}^{(3)} + D_R^{(4)} \Gamma_{n+1}^{(4)}  & \text{ if \Cref{ass:retraction_first_b} holds, $D_{R2}^{(3)}, D_R^{(4)}$ are defined in \eqref{eq:def_D_R_3_2}-\eqref{eq:def_D_R_4}} \\
  D_{R3}^{(3)} \Gamma_{n+1}^{(3)}  +   D_R^{(6)} \Gamma_{n+1}^{(6)}& \text{ if \Cref{ass:retraction_second} holds, $D_{R3}^{(3)},  D_R^{(6)}$ are defined in \eqref{eq:def_D_R_3_3}-\eqref{eq:def_D_R_6}} \eqsp,
\end{cases}
\end{align}
$C(\upeta_1),D^{(2)}_R$ are given by \eqref{eq:C_upeta_markov_w}-\eqref{eq:def_D_R_2} and $I_n \in \iint{0}{n}$ is a random variable independent of $\calF_n$ and
distributed according to \eqref{eq:nn}.
\end{theorem}

We begin the proof by showing a similar lemma to
\Cref{lem:poisson_w_markov}. 
\begin{llemma}
  \label{lem:poisson_ret_makov}
  Assume \Cref{ass:completeness}-\Cref{ass:lyap_mean_field}-\Cref{ass:bounded_bias}-\Cref{ass:bounded_vh}-\Cref{ass:markov}-\Cref{ass:w_markov}$(w)$-\Cref{ass:w_moment}$(w)$
   hold for some measurable function $w:\msx \to
   \cointLigne{1,+\infty}$.  Suppose in addition that \Cref{ass:retraction_first_b} or \Cref{ass:retraction_second} holds. 
  Let $(\upeta_k)_{k \in \nsets}$ be a sequence satisfying
  \eqref{eq:hyp_gamma_k} and  consider $(\theta_k)_{k\in \nset}$
  defined by \eqref{eq:newschemexp}. Set for any $k \in \nset$, 
  $\Delta M_k = \psrLigne{\grad
    V(\theta_k)}{\noise_{\theta_k}(X_{k+1})}[\theta_k]$.
  It holds that for any $n \in\nset$,
  \begin{equation}
    \label{eq:retraction_markov_lemma_first}
\txts               \abs{ \expeLigne{\sum_{k=0}^{n} \upeta_{k+1}\Delta M_k }}\leq D_{\hnoise} \sum_{k=0}^{n} \upeta_{k+1}^2
\expe{\normrLigne{h(\theta_k)}[\theta_k]^2} + C^0_{\hnoise}
\Gamma_{n+1}^{(2)} + C(\upeta_1) + B_{R,0} \eqsp,
  \end{equation}
  where $D_{\hnoise},C^0_{\hnoise},C(\upeta_1)$ are given in   \eqref{eq:def_const_theo2_v_bounded}-\eqref{eq:C_upeta_markov_w} and 
  \begin{align}
    \label{eq:12}
    B_{R,0}
    &=\begin{cases}
      E^{(3)}_{\Ret} \Gamma_{n+1}^{(3)} & \text{ if  \Cref{ass:retraction_first_b} holds}\eqsp, \\
      E^{(4)}_{\Ret} \Gamma_{n+1}^{(4)}  & \text{ if \Cref{ass:retraction_second} holds} \eqsp,     \end{cases}
    \\
    E^{(3)}_{\Ret} = D^{(3)}_{R2} &= 3 \scrl^{(1)} [L\hnoise_{\infty} C_w\{\hinfty^2+\binfty^2+e_{\infty}^2 C_w^{(2)}\} + L_{\hnoise}\cu \hinfty\{\hinfty^2+\binfty^2+e_{\infty}^2 C_w \}] \eqsp, \\
    E^{(4)}_{\Ret} = D^{(4)}_{R3} &=
3^2 \scrl^{(2)} [L\hnoise_{\infty} C_w\{\hinfty^3+\binfty^3+e_{\infty}^3 C_w^{(3)}\} + L_{\hnoise}\cu \hinfty\{\hinfty^3+\binfty^3+e_{\infty}^3 C_w^{(3/2)} \}]
      \eqsp.
  \end{align}

\end{llemma}

\begin{proof} 
  The proof is an adaptation of \Cref{lem:poisson_w_markov} in which we need to deal with the retraction bias.
  Consider the measurable function $\hnoise : \Theta \times \msx \to
  \planT \Theta$ which satisfies  \Cref{ass:w_markov}$(w)$ 
  and for any $k \in \nset$. By \eqref{eq:sa}, we can consider $\upgamma^{(k+1)} : \ccint{0,1}
  \to \Theta$, the geodesic between $\theta_{k}$ and $\theta_{k+1}$
  defined by $\upgamma^{(k+1)}(t)  = \Exp\{t \upeta_{k+1}
      (H_{\theta_{k}}(X_{k+1}) +b_{\theta_{k}}(X_{k+1}) + \Delta_{\theta_k, \upeta_{k+1}}(X_{k+1}))\}$ for any $t
      \in \ccint{0,1}$. Note that for any $k \in\nset$,
        \begin{equation}
    \label{eq:length_upgamma_k_markov_ret}
    \begin{aligned}
        \ell(\upgamma^{(k+1)}) & =
        \upeta_{k+1}\normrLigne{H_{\theta_{k}}(X_{k+1})+b_{\theta_k}(X_{k+1})
        + \Delta_{\theta_k, \upeta_{k+1}}(X_{k+1}) }[\theta_{k}] \eqsp,
        \\
    & \leq \upeta_{k+1} \defEnsLigne{ \normrLigne{H_{\theta_{k}}(X_{k+1})+b_{\theta_k}(X_{k+1})}[\theta_k] + \normrLigne{\Delta_{\theta_{k},\upeta_{k+1}}(X_{k+1})}[\theta_k] } \eqsp.
    \end{aligned}
  \end{equation}
  Using that the parallel transport associated with the Levi-Civita
  connection is a linear isometry \cite[Proposition 5.5]{lee:2019} and
  $(\parallelTransport_{01}^{\upgamma})^{-1} =
  \parallelTransport_{10}^{\upgamma}$ by uniqueness of parallel
  transport \cite[Theorem 4.32]{lee:2019}, we obtain the following
  decomposition
  \begin{equation}
    \label{eq:w_decompo_geod_proof_markov_ret}
        \expe{-\sum_{k=0}^{n} \upeta_{k+1}\Delta M_k } = -\expe{\sum_{i=1}^5 A_i} \eqsp,
      \end{equation}
      where 
      \begin{equation}
        \label{eq:w_fiveterms_decompose_ret}
        \begin{aligned}
	  A_1  &=\sum_{k=1}^{n} \upeta_{k+1}\psr{\grad V(\theta_{k})}{\hnoise_{\theta_k}(X_{k+1}) - P_{\theta_k} \hnoise_{\theta_k}(X_k)}[\theta_k] \eqsp ,\\
	  A_2  &=\sum_{k=1}^{n} \upeta_{k+1}\psr{\grad V(\theta_{k})}{P_{\theta_k} \hnoise_{\theta_k}(X_k)-\parallelTransport_{01}^{\upgamma^{(k)}}  P_{\theta_{k-1}} \hnoise_{\theta_{k-1}}(X_k)}[\theta_k] \eqsp ,\\
          A_3 
          & = \sum_{k=1}^{n} \upeta_{k+1}\psr{\parallelTransport_{10}^{\upgamma^{(k)}}\grad V(\theta_{k}) - \grad V(\theta_{k-1})}{ P_{\theta_{k-1}} \hnoise_{\theta_{k-1}}(X_k)}[\theta_{k-1}] \eqsp ,\\
	  A_4  &=\sum_{k=1}^{n}( \upeta_{k+1}-\upeta_k)\psr{\grad V(\theta_{k-1})}{P_{\theta_{k-1}} \hnoise_{\theta_{k-1}}(X_k)}[\theta_{k-1}] \eqsp , \\
      A_5  &= \upeta_1\psr{\grad
      V(\theta_0)}{\hnoise_{\theta_0}(X_1)}[\theta_0] 
      - \upeta_{n+1}\psr{\grad
      V(\theta_n)}{\hnoise_{\theta_n}(X_{n+1})}[\theta_n]  \eqsp. 
        \end{aligned}
      \end{equation}
      We now bound each term of this decomposition. As for the proof of \eqref{eq:w_bound_a_1_ret} in the proof of \Cref{lem:poisson_w_markov}, using 
      using~\Cref{ass:w_markov}$(w)$-\ref{ass:item:w_markov:e_bound}, we have
            \begin{equation}
                \label{eq:w_bound_a_1_ret}
        \expe{A_1} = 0 \eqsp. 
      \end{equation}
      Similarly, using the same argument as \eqref{eq:w_bound_a_4_geod}  and \eqref{eq:w_bound_a_5_geod} in the proof of \Cref{lem:poisson_w_markov}, we get 
      \begin{equation}
        \label{eq:w_bound_a_4_5_ret}
\begin{aligned}
& \expe{\abs{A_4}} \leq 2 \cu \hnoise_\infty C_w \upeta_1 + 2 \cu a_2
\sum_{k=1}^n \upeta_k \expe{\normrLigne{h(\theta_{k-1})}[\theta_{k-1}]^2}
\eqsp, \\
& \expe{\abs{A_5}} \leq \cu \hnoise_{\infty} \parenthese{2 + C_w
    \sum_{k=0}^n \upeta_{k+1}^2 \expe{\normrLigne{h(\theta_k)}[\theta_k]^2}
} \eqsp .
\end{aligned}
\end{equation}
It remains to deal with $A_2$ and $A_3$ for which we distinguish the case where \Cref{ass:retraction_first_b} or \Cref{ass:retraction_second} holds.

\textbf{In the case where \Cref{ass:retraction_first_b} holds.}
Using
\Cref{lem:retractions}-\ref{lem_retraction_a},
\Cref{ass:bounded_bias}-\Cref{ass:bounded_vh} and Hölder inequality shows that 
\begin{equation} \label{eq:firstret2_markov}
\begin{aligned}
\normrLigne{ \Delta_{\theta_k, \upeta_{k+1}}(X_{k+1}) }[\theta_k] & \leq \scrl^{(1)} \upeta_{k+1} \normrLigne{ H_{\theta_{k}}(X_{k+1})+b_{\theta_{k}}(X_{k+1}) }[\theta_k]^2 \eqsp, \\
& \leq 3 \scrl^{(1)} \upeta_{k+1} \defEns{ \normrLigne{ h(\theta_k)
}[\theta_k]^2 + b_\infty^2 + \normrLigne{ e_{\theta_k} ( X_{k+1} )
}[\theta_k]^2 }\eqsp. 
\end{aligned}
\end{equation}

            Using  the Cauchy-Schwarz inequality, \eqref{eq:length_upgamma_k_markov_ret}, \Cref{ass:lyap_mean_field}-\ref{ass:lyap_mean_field_a}, the definition of $\noise$ \eqref{eq:markovh}, 
      \Cref{ass:bounded_bias},
      \Cref{ass:w_markov}$(w)$-\ref{ass:item:w_markov:e_bound}-\ref{ass:item:w_markov:poisson}
      and Jensen's inequality, we obtain
      \begin{align}
        \abs{A_3} & \leq L \sum_{k=1}^n  \upeta_{k+1}\ell(\upgamma^{(k)})\normrLigne{P_{\theta_{k-1}}\hnoise_{\theta_{k-1}}
          (X_k)}[\theta_{k-1}]
\\
                   &\leq L \sum_{k=1}^n \upeta_k\upeta_{k+1}\defEns{
              \normrLigne{H_{\theta_{k-1}}(X_k)}[\theta_{k-1}]+
          \normrLigne{b_{\theta_{k-1}}(X_k)}[\theta_{k-1}] }
          \normrLigne{P_{\theta_{k-1}}\hnoise_{\theta_{k-1}}
                     (X_k)}[\theta_{k-1}] \\
        & \qquad + L \sum_{k=1}^n \upeta_k\upeta_{k+1} \normrLigne{ \Delta_{\theta_{k-1}, \upeta_{k}}(X_{k}) }[\theta_{k-1}] 
          \normrLigne{P_{\theta_{k-1}}\hnoise_{\theta_{k-1}}
          (X_k)}[\theta_{k-1}]\\
                 & \leq L \hnoise_{\infty} \sum_{k=1}^n \upeta_k\upeta_{k+1}
                 \parenthese{\noise_{\infty} w^{1/2}(X_k) + b_{\infty}+
                 \normrLigne{h(\theta_{k-1})}[\theta_{k-1}]}
                   P_{\theta_{k-1}} w^{1/2}(X_k)\\
        & \qquad + 3 \scrl^{(1)} L \hnoise_{\infty} \sum_{k=1}^n \upeta_{k}^2 \upeta_{k+1} \defEns{ \normrLigne{ h(\theta_{k-1})
}[\theta_{k-1}]^2 + b_\infty^2 + e_{\infty}^2w ( X_{k} ) }                    P_{\theta_{k-1}} w^{1/2}(X_k) \eqsp .
      \end{align}
      Taking the expectation, using 
      \Cref{ass:w_markov}$(w)$-\ref{ass:item:w_markov:moment}, \Cref{ass:bounded_vh},
      that $(\upeta_k)_{k \in \nsets}$ satisfies
      \eqref{eq:hyp_gamma_k} and the Cauchy-Schwarz inequality brings,
      \begin{multline}
                 \label{eq:w_bound_a_3_ret}
\txts                 \expe{\abs{A_3}}
                  \leq L  \hnoise_{\infty} C_w\defEnsLigne{ \parentheseDeuxLigne{ \noise_{\infty}
                 + b_{\infty} } \Gamma^{(2)}_{n+1} 
                 + \sum_{k=1}^n \upeta_{k}^2
                 \expeLigne{\normrLigne{h(\theta_{k-1})}[\theta_{k-1}]^2}}\\
               + 3 L \hnoise_{\infty} \scrl^{(1)}  C_w [\hinfty^2 + \binfty^2+ e_{\infty}^2 C_w^{(2)} ]\Gamma_{n+1}^{(3)}
                  \eqsp. 
      \end{multline}
      It remains to treat $A_2$ depending on the additional two conditions we consider. We start by proving a general bound which hold .
      Using the Cauchy-Schwarz inequality,  \Cref{ass:w_markov}$(w)$-\ref{ass:item:w_markov:e_bound}-\ref{ass:item:w_markov:poisson}-\ref{ass:item:w_markov:poisson_regularity},  \eqref{eq:length_upgamma_k_markov_ret}
 and \Cref{ass:lyap_mean_field}-\ref{ass:lyap_mean_field_b}, we get
 \begin{align}
   \label{eq:decomp_a_2_ret}
        \abs{A_2} 
                   & \leq A_{2,1} + A_{2,2} \\
        A_{2,1} &= L_{\hnoise}\sum_{k=1}^{n} \upeta_{k+1}
                  \upeta_{k}w^{1/2}(X_k) \normrLigne{\grad
                  V(\theta_{k})}[\theta_k]
                  \defEns{\normrLigne{H_{\theta_{k-1}}(X_k)}[\theta_{k-1}]
                  +\normrLigne{b_{\theta_{k-1}}(X_k)}[\theta_{k-1}]}
                  \eqsp ,\\
                  & \leq L_{\hnoise}\cu\sum_{k=1}^{n} \upeta_{k+1}
                  \upeta_{k} w^{1/2}(X_k)
                  \normrLigne{h(\theta_k)}[\theta_k]\{ \noise_{\infty} w^{1/2}(X_k)
                  + \normrLigne{h(\theta_{k-1})}[\theta_{k-1}]
              +\normrLigne{b_{\theta_{k-1}}(X_k)}[\theta_{k-1}]\}
                    \eqsp,\\
        A_{2,2} &= L_{\hnoise}\sum_{k=1}^{n} \upeta_{k+1}
                  \upeta_{k}w^{1/2}(X_k) \normrLigne{\grad
                  V(\theta_{k})}[\theta_k] \normrLigne{ \Delta_{\theta_{k-1}, \upeta_{k}}(X_{k}) }[\theta_k]\\ &                                                                                                               \leq  3 L_{\hnoise}\scrl^{(1)} \sum_{k=1}^{n} \upeta_{k+1}\upeta_{k}^2
\normrLigne{\grad
                  V(\theta_{k})}[\theta_k]
           \defEns{ \normrLigne{ h(\theta_k)
                                                                                                                 }[\theta_k]^2 + b_\infty^2 + e_{\infty}^2 w(X_k) } \eqsp. 
      \end{align}
      Similarly to the proof of \eqref{eq:wbound_a_2:h} in the proof of \Cref{lem:poisson_w_markov},  we have
      \begin{equation}
          \label{eq:wbound_a_2:h_ret}
              \expe{\abs{A_{1,2}}} \leq L_{\hnoise} \cu \parentheseDeux{
                  (\noise_{\infty} + b_{\infty}) \{ h_\infty^2 a_1 +1\} +
              h_\infty^2 } C_w \Gamma^{(2)}_{n+1} \eqsp.
          \end{equation}
Using \Cref{ass:lyap_mean_field}-\ref{ass:lyap_mean_field_b},  \Cref{ass:w_markov}$(w)$-\ref{ass:item:w_markov:e_bound}-\ref{ass:item:w_markov:moment}, \Cref{ass:bounded_vh}, and $(\upeta_k)_{k \in\nsets}$ satisfies \eqref{eq:hyp_gamma_k}, we get 
\begin{equation}
       \label{eq:wbound_a_2_2:h_ret}
              \expe{\abs{A_{2,2}}} \leq 3 L_{\hnoise} \cu   \scrl^{(1)}\hinfty \{\hinfty^2 + \binfty^2 +C_w e_{\infty}^2\} \Gamma_{n+1}^{(3)} \eqsp. 
            \end{equation}
            Combining     \eqref{eq:w_bound_a_1_ret}-\eqref{eq:w_bound_a_4_5_ret}-\eqref{eq:w_bound_a_3_ret}-\eqref{eq:decomp_a_2_ret} \eqref{eq:wbound_a_2:h_ret} and \eqref{eq:wbound_a_2_2:h_ret} in \eqref{eq:w_decompo_geod_proof_markov_ret} completes the proof.

\textbf{If  \Cref{ass:retraction_second} holds.}
Using
\Cref{lem:retractions}-\ref{lem_retraction_b},
\Cref{ass:bounded_bias}-\Cref{ass:bounded_vh} and Hölder inequality shows that 
\begin{equation} \label{eq:firstret2_markov_r3}
\begin{aligned}
\normrLigne{ \Delta_{\theta_k, \upeta_{k+1}}(X_{k+1}) }[\theta_k] & \leq \scrl^{(2)} \upeta_{k+1}^2 \normrLigne{ H_{\theta_{k}}(X_{k+1})+b_{\theta_{k}}(X_{k+1}) }[\theta_k]^3 \eqsp, \\
& \leq 3^2 \scrl^{(2)} \upeta_{k+1}^2 \defEns{ h_{\infty}^3+ b_\infty^3 + \normrLigne{ e_{\theta_k} ( X_{k+1} )
}[\theta_k]^3 } \eqsp. 
\end{aligned}
\end{equation}

            Using  the Cauchy-Schwarz inequality, \eqref{eq:length_upgamma_k_markov_ret}, \Cref{ass:lyap_mean_field}-\ref{ass:lyap_mean_field_a}, the definition of $\noise$ \eqref{eq:markovh}, 
      \Cref{ass:bounded_bias},
      \Cref{ass:w_markov}$(w)$-\ref{ass:item:w_markov:e_bound}-\ref{ass:item:w_markov:poisson}, 
      and Jensen's inequality, we obtain
      \begin{align}
        \abs{A_3} & \leq L \sum_{k=1}^n  \upeta_{k+1}\ell(\upgamma^{(k)})\normrLigne{P_{\theta_{k-1}}\hnoise_{\theta_{k-1}}
          (X_k)}[\theta_{k-1}]
\\
                   &\leq L \sum_{k=1}^n \upeta_k\upeta_{k+1}\defEns{
              \normrLigne{H_{\theta_{k-1}}(X_k)}[\theta_{k-1}]+
          \normrLigne{b_{\theta_{k-1}}(X_k)}[\theta_{k-1}] }
          \normrLigne{P_{\theta_{k-1}}\hnoise_{\theta_{k-1}}
                     (X_k)}[\theta_{k-1}] \\
        & \qquad + L \sum_{k=1}^n \upeta_k\upeta_{k+1} \normrLigne{ \Delta_{\theta_{k-1}, \upeta_{k}}(X_{k}) }[\theta_{k-1}] 
          \normrLigne{P_{\theta_{k-1}}\hnoise_{\theta_{k-1}}
          (X_k)}[\theta_{k-1}]\\
                 & \leq L \hnoise_{\infty} \sum_{k=1}^n \upeta_k\upeta_{k+1}
                 \parenthese{\noise_{\infty} w^{1/2}(X_k) + b_{\infty}+
                 \normrLigne{h(\theta_{k-1})}[\theta_{k-1}]}
                   P_{\theta_{k-1}} w^{1/2}(X_k)\\
        & \qquad + 3^2 \scrl^{(2)} L \hnoise_{\infty} \sum_{k=1}^n \upeta_{k}^3 \upeta_{k+1} \defEns{ \hinfty^3 + b_\infty^3 + e_{\infty}^3w^{3/2} ( X_{k} ) }                    P_{\theta_{k-1}} w^{1/2}(X_k) \eqsp .
      \end{align}
      Taking the expectation, using 
      \Cref{ass:w_markov}$(w)$-\ref{ass:item:w_markov:moment}, \Cref{ass:bounded_vh}, \Cref{ass:w_moment}, 
      that $(\upeta_k)_{k \in \nsets}$ satisfies
      \eqref{eq:hyp_gamma_k} and the Cauchy-Schwarz inequality brings,
      \begin{multline}
                 \label{eq:w_bound_a_3_ret_r3}
\txts                 \expe{\abs{A_3}}
                  \leq L  \hnoise_{\infty} C_w\defEnsLigne{ \parentheseDeuxLigne{ \noise_{\infty}
                 + b_{\infty} } \Gamma^{(2)}_{n+1} 
                 + \sum_{k=1}^n \upeta_{k}^2
                 \expeLigne{\normrLigne{h(\theta_{k-1})}[\theta_{k-1}]^2}}\\
               + 3^2 L \hnoise_{\infty} \scrl^{(2)}  C_w [\hinfty^3 + \binfty^3+ e_{\infty}^3 C_w^{(3)} ]\Gamma_{n+1}^{(4)}
                  \eqsp. 
      \end{multline}
      It remains to treat $A_2$ depending on the additional two conditions we consider. We start by proving a general bound which hold .
      Using the Cauchy-Schwarz inequality,  \Cref{ass:w_markov}$(w)$-\ref{ass:item:w_markov:e_bound}-\ref{ass:item:w_markov:poisson}-\ref{ass:item:w_markov:poisson_regularity},  \eqref{eq:length_upgamma_k_markov_ret}
 and \Cref{ass:lyap_mean_field}-\ref{ass:lyap_mean_field_b}, we get
 \begin{align}
   \label{eq:decomp_a_2_ret_r3}
        \abs{A_2} 
                   & \leq A_{2,1} + A_{2,2} \\
   A_{2,1}   & \leq L_{\hnoise}\cu\sum_{k=1}^{n} \upeta_{k+1}
               \upeta_{k} w^{1/2}(X_k)
               \normrLigne{h(\theta_k)}[\theta_k]\{ \noise_{\infty} w^{1/2}(X_k)
               + \normrLigne{h(\theta_{k-1})}[\theta_{k-1}]
               +\normrLigne{b_{\theta_{k-1}}(X_k)}[\theta_{k-1}]\}
               \eqsp,\\
   A_{2,2} &= L_{\hnoise}\sum_{k=1}^{n} \upeta_{k+1}
             \upeta_{k}w^{1/2}(X_k) \normrLigne{\grad
             V(\theta_{k})}[\theta_k] \normrLigne{ \Delta_{\theta_{k-1}, \upeta_{k}}(X_{k}) }[\theta_k]\\ &                                                                                                               \leq  3^2 L_{\hnoise}\scrl^{(2)} \sum_{k=1}^{n} \upeta_{k+1}\upeta_{k}^3
                                                                                                            \normrLigne{\grad
                                                                                                            V(\theta_{k})}[\theta_k]
                                                                                                            \defEns{ \hinfty^3 +  b_\infty^3 + e_{\infty} w^{3/2} ( X_{k})                                                                                                      }\eqsp.
      \end{align}
      Similarly to the proof of \eqref{eq:wbound_a_2:h} in the proof of \Cref{lem:poisson_w_markov},  we have
      \begin{equation}
          \label{eq:wbound_a_2:h_ret_r3}
              \expe{\abs{A_{1,2}}} \leq L_{\hnoise} \cu \parentheseDeux{
                  (\noise_{\infty} + b_{\infty}) \{ h_\infty^2 a_1 +1\} +
              h_\infty^2 } C_w \Gamma^{(2)}_{n+1} \eqsp.
          \end{equation}
Using \Cref{ass:lyap_mean_field}-\ref{ass:lyap_mean_field_b},  \Cref{ass:w_markov}$(w)$-\ref{ass:item:w_markov:e_bound}-\ref{ass:item:w_markov:moment}, \Cref{ass:w_moment}, \Cref{ass:bounded_vh}, and $(\upeta_k)_{k \in\nsets}$ satisfies \eqref{eq:hyp_gamma_k}, we get 
\begin{equation}
       \label{eq:wbound_a_2_2:h_ret_r3}
              \expe{\abs{A_{2,2}}} \leq 3^2 L_{\hnoise} \cu   \scrl^{(2)}\hinfty \{\hinfty^3 + \binfty^3 +C_w^{(3/2)} e_{\infty}^3\} \Gamma_{n+1}^{(4)} \eqsp. 
            \end{equation}
            Combining     \eqref{eq:w_bound_a_1_ret}-\eqref{eq:w_bound_a_4_5_ret}-\eqref{eq:w_bound_a_3_ret_r3}-\eqref{eq:decomp_a_2_ret_r3} \eqref{eq:wbound_a_2:h_ret_r3} and \eqref{eq:wbound_a_2_2:h_ret_r3} in \eqref{eq:w_decompo_geod_proof_markov_ret} completes the proof.
\end{proof}

\begin{proof}[Proof of \Cref{th:retmarkov_supp}]
  Taking expectation in the bound provided by \Cref{lem:ret_inequality_lemma} and $\varepsilon = \cl_1/(4\cu^2)$, and using $V$ is non-negative, \Cref{ass:lyap_mean_field}-\ref{ass:lyap_mean_field_b}, \Cref{ass:bounded_bias}, \Cref{ass:w_markov}$(w)$-\ref{ass:item:w_markov:e_bound}-\ref{ass:item:w_markov:moment}, we get
    \begin{align}
      &\textstyle{\sum_{k=0}^{n} \upeta_{k+1} (3\cl_1/4-2L
          \upeta_{k+1}) \PE[
 \normrLigne{h(\theta_k)}[\theta_k]^2]  \leq \PE[V(\theta_0) ]
 +  \PE[\sum_{k=0}^n \upeta_{k+1}\Delta M_k]+\cl_2 \Gamma_{n+1}} \\
      &\qquad \qquad\qquad\qquad\textstyle{
   \qquad   + \cu \hinfty \sum_{k=0}^n \upeta_{k+1} \PE[\normrLigne{\Delta_{\theta_k,
 \upeta_{k+1}}(X_{k+1})}[\theta_k] ]}\\
  &\qquad \qquad\qquad\qquad\textstyle{ \qquad+ 2L e_{\infty}^2 C_w \Gamma_{n+1}^{(2)} + 2L \sum_{k=0}^n \upeta_{k+1}^2 
\PE[      \normrLigne{\Delta_{\theta_k,
    \upeta_{k+1}}(X_{k+1})}[\theta_k]^2]} \\
      \label{eq:th:retmarkov:1}
             &\qquad\qquad\qquad\qquad\textstyle{
  \qquad+ \binfty^2 \{ \cu^2/\cl_1 \Gamma_{n+1} + 2L \Gamma_{n+1}^{(2)}\} }
 \eqsp.
  \end{align}

  A bound on $\PE[\sum_{k=0}^n \upeta_{k+1}\Delta M_k]$ is provided by \Cref{lem:poisson_ret_makov}. It remains therefore to deal with the two terms $\sum_{k=0}^n \upeta_{k+1} \PE[ \normrLigne{\Delta_{\theta_k,
 \upeta_{k+1}}(X_{k+1})}[\theta_k] ]$ and $\sum_{k=0}^n \upeta_{k+1}^2 \PE[
      \normrLigne{\Delta_{\theta_k,
          \upeta_{k+1}}(X_{k+1})}[\theta_k]^2]$. These two terms will treated differently depending if \Cref{ass:retraction_first_b} or \Cref{ass:retraction_second} holds.

\textbf{If \Cref{ass:retraction_first_b} holds.}
Using
\Cref{lem:retractions}-\ref{lem_retraction_a},
\Cref{ass:bounded_bias}, \Cref{ass:bounded_vh}, \Cref{ass:w_markov}$(w)$-\ref{ass:item:w_markov:e_bound} and \Cref{ass:w_moment}$(w)$,   and Hölder inequality shows that 
\begin{equation} \label{eq:firstret2_markov_proof_theo}
\begin{aligned}
\normrLigne{ \Delta_{\theta_k, \upeta_{k+1}}(X_{k+1}) }[\theta_k] & \leq \scrl^{(1)} \upeta_{k+1} \normrLigne{ H_{\theta_{k}}(X_{k+1})+b_{\theta_{k}}(X_{k+1}) }[\theta_k]^2 \eqsp, \\
\PE[\normrLigne{ \Delta_{\theta_k, \upeta_{k+1}}(X_{k+1}) }[\theta_k] ] & \leq 3 \scrl^{(1)} \upeta_{k+1} \defEns{ \normrLigne{ h(\theta_k)
}[\theta_k]^2 + b_\infty^2 + e_{\infty}^2 C_w }\\
\PE[\normrLigne{ \Delta_{\theta_k, \upeta_{k+1}}(X_{k+1}) }[\theta_k]^2 ] &\leq 9 \times 3[ \scrl^{(1)}]^2 \upeta_{k+1}^2 \defEns{\hinfty^4 + b_\infty^4 + e_{\infty}^4 C_w^{(2)}} \eqsp. 
\end{aligned}
\end{equation}
Therefore, plugging these estimates in  \eqref{eq:th:retmarkov:1} and using \Cref{lem:poisson_ret_makov},  we  get that 
\begin{equation}
  \label{proof:ret_theo_markov_3}
  \begin{aligned}
      &\textstyle{\sum_{k=0}^{n} \upeta_{k+1} (3\cl_1/4-2L
          \upeta_{k+1}) 
 \PE[\normrLigne{h(\theta_k)}[\theta_k]^2]  \leq \PE[V(\theta_0)]
 + (\cl_2+\binfty^2 \cu^2/\cl_1) \Gamma_{n+1}} \\
&\qquad \qquad \qquad \qquad \txts + D_{\hnoise} \sum_{k=0}^{n} \upeta_{k+1}^2
\expe{\normrLigne{h(\theta_k)}[\theta_k]^2} + C^0_{\hnoise}
\Gamma_{n+1}^{(2)} + C(\upeta_1) +       E^{(3)}_{\Ret} \Gamma_{n+1}^{(3)} 
\\
&\qquad \qquad\qquad\qquad\textstyle{+ 3 \scrl^{(1)} \cu \hinfty \sum_{k=0}^n \upeta_{k+1}^2 \normrLigne{h(\theta_k)}[\theta_k] + (\binfty^2 + C_w e_{\infty}^2) (3 \cu \hinfty  \scrl^{(1)}+2L) \Gamma_{n+1}^{(2)}} \\
  &\qquad \qquad\qquad\qquad\textstyle{ + 2 3^3 L [\scrl^{(1)}]^2\{\hinfty^4 + \binfty^4 + e_{\infty}^4 C_w^{(2)}\} } \Gamma_{n+1}^{(4)}
 \eqsp.
  \end{aligned}
\end{equation}
Rearranging terms and using $\sup_{k \in\nsets} \upeta_k < \bupeta$, with $\bupeta$ satisfying \eqref{eq:def_bupeta_marting_retract_markov}, we get
\begin{equation}
  \label{proof:ret_theo_markov_3}
  \begin{aligned}
      &\textstyle{(\cl_1/2)\sum_{k=0}^{n} \upeta_{k+1}  
 \PE[\normrLigne{h(\theta_k)}[\theta_k]^2]  \leq \PE[V(\theta_0)]
 + (\cl_2+\binfty^2 \cu^2/\cl_1) \Gamma_{n+1}} \\
&\qquad \qquad \qquad \qquad \txts + [C^0_{\hnoise}+ (\binfty^2 + C_w e_{\infty}^2) (3 \cu \hinfty  \scrl^{(1)}+2L)]
\Gamma_{n+1}^{(2)} + C(\upeta_1) +       E^{(3)}_{\Ret} \Gamma_{n+1}^{(3)} 
\\
  &\qquad \qquad\qquad\qquad\textstyle{ + 2 3^3 L [\scrl^{(1)}]^2\{\hinfty^4 + \binfty^4 + e_{\infty}^4 C_w^{(2)}\} } \Gamma_{n+1}^{(4)}
 \eqsp,
\end{aligned}
\end{equation}
which completes the proof.

and taking $\varepsilon = \cl_1/(4\cu^2)$, we get
\begin{equation}
  \label{proof:ret_theo_martingale_3}
  \begin{aligned}
      &\textstyle{\sum_{k=0}^{n} \upeta_{k+1} (3\cl_1/4-(2L+3\scrl^{(1)}\cu \hinfty)
          \upeta_{k+1}) 
 \PE[\normrLigne{h(\theta_k)}[\theta_k]^2]  \leq \PE[V(\theta_0)]
 + (\cl_2+\binfty^{2}\cu^2/\cl_1) \Gamma_{n+1}} \\
      &\qquad \qquad \qquad+ (2L+3 \scrl^{(1)}  \cu \hinfty )(\binfty^2 + \moment_{(2)}) \Gamma_{n+1}^{(2)}\textstyle{+ 27 \scrl^{(1)}\{\binfty^4 + \hinfty^4 + \moment_{(4)}\} \Gamma_{n+1}^{(4)}}  \eqsp.
  \end{aligned}
\end{equation}

Using $\sup_{k \in\nsets} \upeta_k \leq \bupeta \leq \cl_1[4(2L+3\scrl^{(1)}\cu \hinfty)]^{-1}$ ensures that $\cl_1/4 \geq (2L+3\scrl^{(1)}\cu \hinfty)\upeta_{k+1} $ for any $k \in \nset$ and the proof follows.

\textbf{If  \Cref{ass:retraction_second} holds.}
Using
\Cref{lem:retractions}-\ref{lem_retraction_b}, \Cref{ass:bounded_bias}, \Cref{ass:bounded_vh}, \Cref{ass:w_markov}$(w)$-\ref{ass:item:w_markov:e_bound} and \Cref{ass:w_moment}$(w)$,   and Hölder inequality shows that 
\begin{equation} \label{eq:firstret2_markov_proof_theo_r3}
\begin{aligned}
\normrLigne{ \Delta_{\theta_k, \upeta_{k+1}}(X_{k+1}) }[\theta_k] & \leq \scrl^{(2)} \upeta_{k+1}^2 \normrLigne{ H_{\theta_{k}}(X_{k+1})+b_{\theta_{k}}(X_{k+1}) }[\theta_k]^3 \eqsp, \\
 \PE[\normrLigne{ \Delta_{\theta_k, \upeta_{k+1}}(X_{k+1}) }[\theta_k] ]& \leq 3^2 \scrl^{(2)} \upeta_{k+1}^2 \defEns{ h_{\infty}^3+ b_\infty^3 + e_{\infty}^3 C_w^{(3/2)} }\\
\PE[\normrLigne{ \Delta_{\theta_k, \upeta_{k+1}}(X_{k+1}) }[\theta_k]^2] &\leq 3^5[ \scrl^{(2)}]^2 \upeta_{k+1}^2 \defEns{\hinfty^6 + b_\infty^6 + e_{\infty}^6 C_w^{(3)} } \eqsp. 
\end{aligned}
\end{equation}
Therefore, plugging these estimates in  \eqref{eq:th:retmarkov:1} and using \Cref{lem:poisson_ret_makov},  we  get that 
\begin{equation}
  \label{proof:ret_theo_markov_4}
  \begin{aligned}
      &\textstyle{\sum_{k=0}^{n} \upeta_{k+1} (3\cl_1/4-2L
          \upeta_{k+1}) 
 \PE[\normrLigne{h(\theta_k)}[\theta_k]^2]  \leq \PE[V(\theta_0)]
 + (\cl_2+\binfty^2 \cu^2/\cl_1) \Gamma_{n+1}} \\
&\qquad \qquad \qquad \qquad \txts + D_{\hnoise} \sum_{k=0}^{n} \upeta_{k+1}^2
\expe{\normrLigne{h(\theta_k)}[\theta_k]^2} + C^0_{\hnoise}
\Gamma_{n+1}^{(2)} + C(\upeta_1) +       E^{(4)}_{\Ret} \Gamma_{n+1}^{(4)} 
\\
&\qquad \qquad\qquad\qquad\textstyle{+ 2L (\binfty^2 + C_w e_{\infty}^2) \Gamma_{n+1}^{(2)} + 3^2 \scrl^{(2)} \cu \hinfty \defEns{ h_{\infty}^3+ b_\infty^3 + e_{\infty}^3 C_w^{(3/2)} } \Gamma_{n+1}^{(3)} } \\
  &\qquad \qquad\qquad\qquad\textstyle{ + 2 3^5 L [\scrl^{(2)}]^2\{\hinfty^6 + \binfty^6 + e_{\infty}^6 C_w^{(3)}\} } \Gamma_{n+1}^{(6)}
 \eqsp.
  \end{aligned}
\end{equation}
Rearranging terms and using $\sup_{k \in\nsets} \upeta_k < \bupeta$, with $\bupeta$ satisfying \eqref{eq:def_bupeta_marting_retract_markov}, we get
\begin{equation}
  \label{proof:ret_theo_markov_3}
  \begin{aligned}
      &\textstyle{(\cl_1/2)\sum_{k=0}^{n} \upeta_{k+1} 
 \PE[\normrLigne{h(\theta_k)}[\theta_k]^2]  } \\
& \qquad  \txts\leq \PE[V(\theta_0)]
 + (\cl_2+\binfty^2 \cu^2/\cl_1) \Gamma_{n+1}  + [C^0_{\hnoise}+2L (\binfty^2 + C_w e_{\infty}^2)]
\Gamma_{n+1}^{(2)} + C(\upeta_1) +       E^{(4)}_{\Ret} \Gamma_{n+1}^{(4)}  
\\
&\qquad + 3^2 \scrl^{(2)} \cu \hinfty \defEns{ h_{\infty}^3+ b_\infty^3 + e_{\infty}^3 C_w^{(3/2)} } \Gamma_{n+1}^{(3)}  \textstyle{ + 2 3^5 L [\scrl^{(2)}]^2\{\hinfty^6 + \binfty^6 + e_{\infty}^6 C_w^{(3)}\} } \Gamma_{n+1}^{(6)}
 \eqsp.
  \end{aligned}
\end{equation}

Therefore, plugging these estimates in  \eqref{proof:ret_theo_martingale_1}, using   \Cref{ass:retraction_martingale}$(6)$, we  get that 
\begin{equation}
  \label{proof:ret_theo_martingale_2_R3}
  \begin{aligned}
      &\textstyle{\sum_{k=0}^{n} \upeta_{k+1} (\cl_1-2L
          \upeta_{k+1}-\cu^2\varepsilon) 
 \PE[\normrLigne{h(\theta_k)}[\theta_k]^2]  \leq \PE[V(\theta_0)]
 + \cl_2 \Gamma_{n+1}} \\
      &\qquad \qquad\qquad\qquad\textstyle{
   + 3^2 \scrl^{(2)}  \cu \hinfty   \defEns{ \hinfty^3 + b_\infty^3 + \moment_{(3)}  } \Gamma_{n+1}^{(3)} }+ 2L(\binfty^2 + \moment_{(2)}) \Gamma_{n+1}^{(2)}\\
  & \qquad\qquad\qquad\textstyle{ \qquad+ 3^5 \scrl^{(2)}\{\binfty^6 + \hinfty^6 + \moment_{(6)}\} \Gamma_{n+1}^{(6)}} \textstyle{
+ [\binfty^{2}/(4\varepsilon)] \Gamma_{n+1}}
 \eqsp.
  \end{aligned}
\end{equation}
Rearranging terms and taking $\varepsilon = \cl_1/(4\cu^2)$, we get
\begin{equation}
  \label{proof:ret_theo_martingale_2_R3}
  \begin{aligned}
      &\textstyle{\sum_{k=0}^{n} \upeta_{k+1} (3\cl_1/4-2L
          \upeta_{k+1}) 
 \PE[\normrLigne{h(\theta_k)}[\theta_k]^2]  \leq \PE[V(\theta_0)]
 + (\cl_2+\binfty^{2}\cu^2/\cl_1) \Gamma_{n+1}} \\
 &\qquad \qquad \qquad \qquad\qquad \qquad + 2L(\binfty^2 + \moment_{(2)}) \Gamma_{n+1}^{(2)}\textstyle{
   + 3^2 \scrl^{(2)}  \cu \hinfty   \defEns{ \hinfty^3 + b_\infty^3 + \moment_{(3)}  } \Gamma_{n+1}^{(3)} }\\
 & \qquad \qquad\qquad \qquad\qquad \qquad
 \textstyle{+ 3^5 \scrl^{(2)}\{\binfty^6 + \hinfty^6 + \moment_{(6)}\} \Gamma_{n+1}^{(6)}} \eqsp.
\end{aligned}
\end{equation}
Using that $\sup_{k \in\nsets} \upeta_k \leq \bupeta \leq \cl_1/(8L)$ ensures that $\cl_1/4 \geq 2L \upeta_{k+1}$ for any $k \in \nset$ which completes the proof. 
      
\end{proof}


\section{Proof of \Cref{sec:applications}}

\subsection{Proofs of \Cref{subsec:pca}}
\label{app:pca}

  \label{lem:grass_poisson_lip}
\begin{pproposition}
Consider the setting of \Cref{subsec:pca}  and assume  \Cref{ass:pca:markov}. Consider $e_{[B]}(x) = \{\rmI_d -  B B^{\transpose}\}x x^{\transpose}  B - \{\rmI_d - BB^{\transpose}\}\bfA B$, for any $[B] \in \grassmann_r(\rset^d)$, $x \in \rset^d$. Then,
for              any $[B_0], [B_1]  \in \grassmann_r(\rset^d)$, $x \in \msx$ and
              geodesic curve $\upgamma : \ccint{0,1} \to \grassmann_r(\rset^d)$
              between $[B_0]$ and $ [B_1]$,
              \begin{equation}
                            \label{eq:prop_condition_P_implication_MA_b_grassman}
               \normr{\noise_{[B_1]}(x)
                    -\parallelTransport_{01}^{\upgamma}\noise_{[B_0]}(x)}[{[B_1]}]
                \leq C \ell(\upgamma) (1+\norm{x}^2)  \eqsp.
              \end{equation}
  In particular, for any $x \in \rset^d$, $[B] \mapsto e_{[B]}(x)$ satisfies \eqref{eq:prop_condition_P_implication_MA_b}.
\end{pproposition}

\begin{proof}
  Note that by \cite[Section 2.5.3]{edelman:arias:smith:1998}, for any $[B] \in \grassmann_r(\rset^d)$, $x \in \rset^d$, $e_{[B]}(x) = - \grad f^{x}([B]) +\grad f([B])$ where $\grad$ is the Riemannian gradient on $\grassmann_r(\rset^d)$, $f^x  = - \trace(B^{\transpose} (x x^{\transpose})\bfA B) / 2$ and $f([B] ) = -\trace(B^{\transpose} \bfA B) / 2$. Therefore, to conclude that \eqref{eq:prop_condition_P_implication_MA_b_grassman}, using \Cref{lem:bounded_hessian}, it is sufficient to show that $\Hess f([B]), \Hess f^x([B]):\rmT_{[B]}\grassmann_r(\rset^d) \to \rmT_{[B]}\grassmann_r(\rset^d)$ have  operator norm upper bounded uniformly for  $[B] \in \grassmann_r(\rset^d)$ by $C\{1+\norm{x}^2\}$ for some constant $C$. First, since $f$ does not depend on $x$ and $\grassmann_r(\rset^d)$ is compact, there exists $C_1 \geq 0$ such that for any $[B] \in \grassmann_r(\rset^d)$, $\normr{\Hess f([B])}[{[B]}] \leq C_1$. In addition, by   \cite[Section 2.5.4]{edelman:arias:smith:1998}, $\Hess f^x([B]) : D_1 \in \planT_{[B]} \grassmann_r(\rset^d) \mapsto \{\rmI_d - B B^{\transpose}\}(x x^{\transpose}) D_1 - D_1 B^{\transpose} (x x^{\transpose}) B \in \planT_{[B]} \grassmann_r(\rset^d)$. Therefore, using the definition of the canonical metric on $\grassmann_r(\rset^d)$ given in \Cref{sec:retr-quant-estim}, we get that there exists $C_2 \geq 0$ such that for any $[B] \in \grassmann_r(\rset^d)$, $\normr{\Hess f^x([B])}[{[B]}] \leq C_2 \norm{x}^2$ which completes the proof.
\end{proof}


\subsection{Proofs of Section \ref{subsec:apphuber}}

\label{app:proofshuber}

Recall that we consider the Huber-like dissimilarity measure, $\rhoH:\Theta \times \Theta \rightarrow \mathbb{R}_+$, given for any $\theta_0,\theta_1\in\Theta$ by 
\begin{equation} \label{eq:huberD}
  \rhoH(\theta_0,\theta_1) \,=\, \delta^2\,\left[1+\left\lbrace \rho_\Theta(\theta_0,\theta_1)\middle/\delta\right\rbrace^{2\;}\right]^{\scriptscriptstyle 1/2}\,-\,\delta^2 \eqsp,
\end{equation}
where $\delta > 0$.

We present lemmas to verify \Cref{ass:lyap_mean_field}--\Cref{ass:bounded_bias}, \Cref{ass:0mean_noise} for the robust barycenter problem.   
In addition, we show that the function $f(\theta)$ of \eqref{eq:huberV},
\begin{equation}
  \label{eq:3}
  f(\theta) \,=\, \int_\Theta\,\rhoH(\theta,x)\,\pi(\rmd x) \text{ for $\theta \in\Theta$} \eqsp,
\end{equation}
is strictly g-convex. Thus the the robust barycenter of a probability distribution $\pi$ on $\Theta$ exists and is unique. 
Consequently, as the geodesic SA scheme \eqref{eq:huberscheme} finds a stationary point of \eqref{eq:huberV}, the strict g-convexity of $f$ guarantees that such stationary point is globally optimal and is unique.

\begin{llemma} \label{lem:hessian_comparison}
Let $\Theta$ be a Hadamard manifold with sectional curvature bounded below by $-\kappa^2$,  $x \in \Theta$ and $\delta >0$. Define for any $\theta \in \Theta$, 
\begin{equation}
V_2(\theta)=\distT^2(x,\theta)   \eqsp \text{ and } \eqsp V_1(\theta) = \delta^2 \parentheseDeux{V_2(\theta)/\delta^2+1}^{1/2}-\delta^2 \eqsp.
\end{equation}
Then, $V_1$ is a smooth function and its Riemannian gradient is given for any $\theta \in \Theta$ by
\begin{equation}\label{eq:grad_VHuber}
\grad V_1(\theta) = - \left. \Exp^{-1}_{\theta}(x) \middle/ \parentheseDeux{V_2(\theta)/\delta^2+1}^{1/2}  \right. \eqsp.
\end{equation}
Moreover, for any $\theta \in \Theta$, $v \in \planT_\theta \Theta \setminus \{0\}$, its Hessian satisfies
\begin{equation}
0 < \Hess V_1 (\theta) (v,v) \leq (1+ \delta \kappa) \normr{v}[\theta]^2 \eqsp.
\end{equation}
\end{llemma}
\begin{proof}
  The proof relies heavily on the computation of $\grad V_2$ and the Hessian comparison of $V_2$ done in \cite[Theorem 5.6.1]{jost:2005}.
 Indeed, \cite[Theorem 5.6.1]{jost:2005} shows $V_2$ is smooth and that for any $\theta \in \Theta$, 
\begin{equation}\label{eq:grad_sqdist}
\grad V_2 (\theta) = - 2\Exp_x^{-1}(\theta) \eqsp.
\end{equation}
Hence, \eqref{eq:grad_VHuber} follows by composition since $V_1(\theta) = \varpi \circ V_2(\theta)$ where $\varpi : t \to \delta^2[t/\delta^2+1]^{1/2} - \delta^2 $ for $t \in \rset_+$.

For any $v\in \planT_v\Theta$, recall that $\Hess V_1(\theta)(v,v) = \psrLigne{\nabla_v \grad  V_1(\theta)}{v}[\theta]$, where $\nabla$ is the Levi-Civita connection -- see \Cref{app:grad_hess}. The product rule for the covariant derivative \cite[p. 73]{ghl:2004} for a smooth function $\mathrm{f}$ and vector field $Y$ on $\Theta$ gives $\nabla (\mathrm{f}Y)=\nabla \mathrm{f} \otimes Y + \mathrm{f} \nabla Y$. Applying this result to 
\begin{equation}
\mathrm{f}(\theta)= \parentheseDeux{V_2(\theta)/\delta^2 +1 }^{-1/2} \eqsp, \quad Y(\theta)=-\Exp^{-1}_\theta(x) \eqsp,
\end{equation}
and using $Y=\grad V_2 /2$ gives
\begin{multline}\label{eq:hess_v}
\Hess V_1(\theta) = - \left. \Exp_\theta^{-1}(x)\otimes \Exp_\theta^{-1} (x) \middle/ \parentheseDeux{\delta^2\defEns{1 + V_2(\theta)/\delta^2}^{3/2}} \right. \\
 + \left. \Hess V_2(\theta) \middle/\parentheseDeux{2\defEns{1 + V_2(\theta)/\delta^2}^{1/2}} \right. \eqsp.
\end{multline}

Let $\theta \in \Theta$ and $v \in \planT_\theta \Theta \setminus \{0\}$. On the one hand, we have $\Exp_\theta^{-1}(x) \otimes \Exp_\theta^{-1}(x) (v,v)  = \psrLigne{\Exp_\theta^{-1}(x)}{v}[\theta]^2 $. Therefore, using Cauchy-Schwarz inequality,
\begin{equation}\label{eq:comparison_log}
0\leq \Exp_\theta^{-1}(x) \otimes \Exp_\theta^{-1}(x) (v,v) \leq \normr{\Exp_\theta^{-1}(x)}[\theta]^2\normr{v}[\theta]^2 = \distT^2(\theta,x) \normr{v}[\theta]^2 \eqsp,
\end{equation}
since $\normrLigne{\Exp_\theta^{-1}(x)}[\theta]=\distT(\theta,x)$. On the other hand,  \cite[Theorem 5.6.1]{jost:2005} implies that 
\begin{equation}\label{eq:hess_sqdist}
2\normr{v}[\theta]^2 \leq \Hess V_2 (\theta) (v,v) \leq 2\kappa \distT(\theta,x) \coth\parentheseDeux{\kappa \distT(\theta,x)} \normr{v}[\theta]^2 \eqsp.
\end{equation}

Now, combining  \eqref{eq:comparison_log}-\eqref{eq:hess_sqdist} in \eqref{eq:hess_v}, and using $t\coth(t)\leq 1+t$ for $t\geq0$, it follows that 
\begin{align}
\Hess V_1(\theta)(v,v) &\leq \left. \parentheseDeux{1+ \kappa \distT(\theta,x)} \middle/\defEns{1 + V_2(\theta)/\delta^2}^{1/2} \right. \normr{v}[\theta]^2 \\
& \leq (1 + \kappa \delta) \normr{v}[\theta]^2 \eqsp,
\end{align}
where we have used $1 + V_2(\theta)/\delta^2 \geq \max(1,\distT^2(\theta,x)/\delta^2)$. Similarly, we obtain 
\begin{align}
\Hess V_1(\theta)(v,v) &\geq - \left. \parenthese{\distT^2(\theta,x) \normr{v}[\theta]^2}\middle/ \parenthese{\delta^2 \defEns{1 + V_2(\theta)/\delta^2}^{3/2} } \right.  + \left. \normr{v}[\theta]^2 \middle/ \defEns{1 + V_2(\theta)/\delta^2}^{1/2} \right. \\
& > 0 \eqsp,
\end{align}
which concludes the proof.
\end{proof}

\begin{llemma} \label{lem:huber}
 Let $\Theta$ be a Hadamard manifold with sectional curvature bounded below by $-\kappa^2$. Furthermore, consider the Lyapunov function given by \eqref{eq:huberV} and stochastic approximation scheme \eqref{eq:huberscheme}. Then, \Cref{ass:lyap_mean_field} and \Cref{ass:bounded_bias} are satisfied, with $\cl=\cu = 1$, $L = 1 + \delta\kappa$, and $b_\infty = 0$. Moreover, \Cref{ass:0mean_noise} is satisfied, with $\sigma_0 = \delta^2$ and $\sigma_1 = 0$.
\end{llemma}

\begin{proof}
To show that \Cref{ass:lyap_mean_field} holds with the stated values of $\cl$, $\cu$ and $L$, note that
the scheme \eqref{eq:huberscheme} can be written
\begin{equation}
\theta_{n+1} \,=\, \Exp_{\theta_n}\parenthese{ \upeta_{n+1}\,H_{\theta_n}(X_{n+1}) } \eqsp,
\end{equation}
where the stochastic update $H_{\theta_n}(X_{n+1})$ is given by
\begin{equation} \label{eq:huberH}
  H_\theta(x) \,=\left. -\Exp^{-1}_\theta(x) \middle/ \parentheseDeux{1+\defEns{ \rho_\Theta(\theta,x) \middle/\delta}^2}^{1/2} \right. \eqsp.
\end{equation}
Then, both $\grad V$ and $\Hess V$ can be computed differentiating under the integral in \eqref{eq:huberV}. Using \Cref{lem:hessian_comparison}, we know that for any $x\in \Theta, V_1:\theta \mapsto \delta^2 (1+\distT^2(\theta,x)/\delta^2)^{1/2} - \delta^2$ is smooth. Using $\normrLigne{\Exp^{-1}_\theta(x)}[\theta] = \distT(\theta,x)$, we have that for any $x,\theta \in \Theta,
\normrLigne{\grad V_1(\theta)}[\theta] \leq 1 $
and, for any $x,\theta \in \Theta, v \in \planT_\theta \Theta$,
\begin{equation}
\abs{\Hess V_1 (\theta) (v,v)} \leq (1+\kappa \delta) \normr{v}[\theta]^2 \eqsp.
\end{equation}
Under these domination conditions, using Lebesgue's dominated convergence theorem, we have for any $\theta \in \Theta$,
\begin{equation} \label{eq:gradhuberV}
  \grad f(\theta) = - \int_\Theta \left.\Exp^{-1}_\theta(x) \middle/ \parentheseDeux{1+\defEns{ \distT(\theta,x)\middle/\delta}^2}^{1/2}\,\pi(\rmd x) \right. \eqsp,
\end{equation}
and for any $v \in \planT_\theta \Theta$,
\begin{equation}
\Hess f(\theta)(v,v) = \int_\Theta \Hess V_1(\theta)(v,v) \pi(\rmd x) \eqsp.
\end{equation}
Thus, for any $v\in \planT_\theta \Theta \setminus \{0\}$,
\begin{equation}\label{eq:huberhessianbounds}
0 < \Hess f (\theta) (v,v) \leq (1+\kappa \delta) \normr{v}[\theta]^2 \eqsp.
\end{equation}
This last inequality proves that the operator norm of $\Hess f$ is upper bounded by $1+ \delta \kappa$. Therefore, using \Cref{lem:bounded_hessian}, it follows that \Cref{ass:lyap_mean_field}-\ref{ass:lyap_mean_field_a} holds with $L = 1+\delta\kappa$.

It remains to prove that Assumptions \Cref{ass:0mean_noise} is satisfied with $\sigma_0 = \delta^2$ and $\sigma_1 = 0$. To do so, note first from (\ref{eq:huberH}), that
$$
  \Vert H_{\theta_n}(X_{n+1}) \Vert^2_{\theta_n} =\left. \normr{\Exp^{-1}_{\theta_n}(X_{n+1})}[\theta_n]^2 \middle/ \parenthese{1+\left\lbrace \distT(\theta_n,X_{n+1})\middle/\delta\right\rbrace^{2\;}} \right. \eqsp.
$$
Since $\Vert \Exp^{-1}_{\theta_n}(X_{n+1})\Vert^2_{\theta_n} = \distT^2(\theta_n\,,X_{n+1})$, it follows that $  \Vert H_{\theta_n}(X_{n+1}) \Vert^2_{\theta_n}\leq \delta^2\,$. Thus, in the notation of (\ref{eq:markovh}),
\begin{equation} \label{eq:hubermartingalenoise}
\expe{\normr{\noise_{\theta_n}\parenthese{X_{n+1}}}[\theta_n]^2 \middle| \mcf_n} =
\expe{\normr{H_{\theta_n}\parenthese{X_{n+1}} - h(\theta_n)}[\theta_n]^2 \middle| \mcf_n}
\leq \delta^2
\end{equation}
since the conditional variance is bounded by the mean square. Now, the required values of $\sigma_0$ and $\sigma_1$ can be read from (\ref{eq:hubermartingalenoise}). 
\end{proof}

\begin{pproposition} \label{prop:huberconv}
 Let $\Theta$ be a Hadamard manifold with sectional curvature bounded below, and let $\pi$ be a probability distribution on $\Theta$. Assume there exists some $\tau \in \Theta$ such that 
\begin{equation} \label{eq:huberuniquecondition}
\int_\Theta\,\rho_\Theta(\tau,x)\,\pi(\rmd x)\,<\, +\infty \eqsp.
\end{equation}
Then, the function $f:\Theta\rightarrow \mathbb{R}_+$ given by (\ref{eq:huberV}) is geodesically strictly convex, but not strongly convex, in general. Moreover, $f$ has a unique global minimizer $\theta^* \in \Theta$. In other words, $\pi$ has a unique robust barycenter $\theta^*$.
\end{pproposition}

\begin{proof}
Under condition (\ref{eq:huberuniquecondition}), the function $f$ takes finite values, $f(\theta) < +\infty$ for any $\theta \in \Theta$. Indeed, note the following inequality, which holds for all $x \geq 0$,
$$
\delta^2\,\left[1+\left\lbrace x/\delta\right\rbrace^{2\;}\right]^{\scriptscriptstyle 1/2}\,-\,\delta^2 \,\leq\, \delta\,x
$$
From (\ref{eq:huberD}) and (\ref{eq:huberV}), this inequality implies
\begin{equation} 
  f(\theta) \,\leq\,  \delta\,\int_\Theta\,\distT(\theta,x)\,\pi(\rmd x) \hspace{0.5cm} \text{ for } \theta \in \Theta
\end{equation}
and, furthermore, by the triangle inequality,
\begin{equation} \label{eq:finiteV}
  f(\theta) \,\leq\, \delta\,\int_\Theta\,(\distT(\theta,\tau) + \distT(\tau,x))\,\pi(\rmd x) \,=\, \distT(\theta,\tau) + 
\int_\Theta\,\distT(\tau,x)\,\pi(\rmd x)
\end{equation}
Then, it follows  from (\ref{eq:huberuniquecondition}) and (\ref{eq:finiteV}) that $f(\theta) < +\infty$ for $\theta \in \Theta$.

Further, from (\ref{eq:huberhessianbounds}) in the proof of Lemma \ref{lem:huber}, $\Hess f(\theta) \succ 0$, so $f$ has strictly positive-definite Riemannian Hessian, and is therefore geodesically strictly convex. To see that $f(\theta)$ may fail to be strongly convex, consider the case where $\pi = \delta_\tau$ (here, $\delta_\tau$ is the Dirac distribution, concentrated at $\tau \in \Theta$). By (\ref{eq:huberV}), it then follows
\begin{equation} \label{eq:huberVD}
f(\theta) \,=\, \rhoH(\theta,\tau)
\end{equation}
Then, let $\upgamma(t)$ be a geodesic through $\tau$, given by $\upgamma(t) = \Exp_\tau(t\,u)$ for $t \in \mathbb{R}$, where $u \in T_\tau\Theta$ has $\Vert u \Vert_\tau = 1$. If $f(\theta)$ is given by (\ref{eq:huberVD}), then it follows from (\ref{eq:huberD})
$$
(f \circ \upgamma)(t) \,=\, \delta^2\,\left[1+\left\lbrace t/\delta\right\rbrace^{2\;}\right]^{\scriptscriptstyle 1/2}\,-\,\delta^2 
$$ 
but this is not a strongly convex function of $t \in \mathbb{R}$. Therefore, $f$ is not geodesically strongly convex on $\Theta$~\cite[p.~187]{udriste:1994}. 

It remains to show that $f$ has a unique global minimizer $\theta^*$. Since $f$ is geodesically strictly convex, and bounded below (indeed, $f$ is positive), it is enough to show that $f$ is coercive, in the sense that $f(\theta) \rightarrow +\infty$ when $\distT(\theta,\tau) \rightarrow +\infty$. To do so, note the following inequality holds for all real $x$,
$$
\delta^2\,\left[1+\left\lbrace x/\delta\right\rbrace^{2\;}\right]^{\scriptscriptstyle 1/2}\,-\,\delta^2 \,\geq\, 
\frac{\delta}{\sqrt{2}}(x-\delta)
$$
Then, using (\ref{eq:huberD}) and (\ref{eq:huberV}), this inequality implies
\begin{equation} 
  f(\theta) \,\geq\,  
\frac{\delta}{\sqrt{2}}\;\int_\Theta\,\left( \distT(\theta,x) - \delta\right)\,\pi(\rmd x)
\end{equation}
or, by the triangle inequality,
\begin{equation} 
  f(\theta) \,\geq\,  
\frac{\delta}{\sqrt{2}}\;\int_\Theta\,\left( |\distT(\theta,\tau)-\distT(\tau,x)| - \delta\right)\,\pi(\rmd x) \hspace{0.5cm} \text{ for } \theta \in \Theta
\end{equation}
However, this directly yields
$$
 f(\theta) \,\geq\, 
\frac{\delta}{\sqrt{2}}\;\distT(\theta,\tau) - \frac{\delta}{\sqrt{2}}\;\int_\Theta\,\left( \distT(\tau,x) + \delta\right)\,\pi(\rmd x) 
$$
Clearly, the right-hand side increases to $+\infty$ when $\distT(\theta,\tau) \rightarrow +\infty$. Thus, $f$ is indeed coercive.

Finally, to show that $f$ has a unique global minimizer $\theta^*$, let $f_* = \inf \lbrace f(\theta)\,; \theta \in \Theta \rbrace$ and $(\theta_n\,;n=1,2,\ldots)$ a sequence of points in $\Theta$ such that $\lim f(\theta_n) = f_*\,$. since $f$ takes finite values, $f_* < + \infty$. Therefore, there exists some $R > 0$ such that $\distT(\theta_n\,,\tau) < R$ for all $n$. This is because, otherwise, $\distT(\theta_n\,,\tau) \rightarrow +\infty$,  and thus $f(\theta_n) \rightarrow +\infty$, since $f$ is coercive. Because $\Theta$ is a complete Riemannian manifold, the metric ball $B(\tau,R)$ has compact closure (a consequence of the Hopf-Rinow theorem~\cite{lee:2019}). This implies that the sequence $\theta_n$ has a convergent subsequence, whose limit $\theta^*$ belongs to the closure of $B(\tau,R)$. By continuity of $f$, it is clear that $f(\theta^*) = \lim f(\theta_n) = f_*\,$, so $\theta^*$ is indeed a global minimizer of $f$. This global minimizer is unique because $f$ is geodesically strictly convex.
\end{proof}


\subsection{Details and specification for \Cref{subsec:numerical}}
\label{sec:app_numerics}

The entries of matrix $\mathbf{M}$ are chosen such that
\begin{enumerate*}[label=(\alph*)]
\item for any $i \in \{1,\ldots,50\}$, $\Mbf_{i,i}$ is a sample from the uniform distribution on $\ccint{0.1,0.7}$;
\item for any $k \in\{1,\ldots,23\}$ and  $i \in \{1,d-k\}$, $\Mbf_{i,k}= U_{i,k}(0.1+0.8(i-1)/27)$, where $U_{i,k}$ is a sample from the uniform distribution on $\ccint{0.05,0.1}$;
\item $\norm{\Mbf} < 1$.
\end{enumerate*}
By a straightforward induction, we have since for any $k \in
 \nset$,
$X_{k+1} = \mathbf{M} X_k + \sigma \varepsilon_{k+1}$, where
$\sigma >0$,
\begin{equation}
  X_{k} = \mathbf{M}^k X_0 + \sigma \sum_{i=0}^{k-1} \Mbf^{k-1-i} \varepsilon_{i+1} \eqsp. 
\end{equation}
Therefore, for any $k \in \nsets$, $X_k$ follows a Gaussian distribution with mean
$\mathbf{M}^k X_0$ and covariance matrix
$\sigma^2 \sum_{i=0}^{k-1} \Mbf^{k-1-i} [\Mbf^{k-1-i}]^{\transpose}$. Denote by $P$ the Markov kernel associated with $(X_k)_{k \in\nset}$
Then, it is straightforward to verify that it is strongly aperiodic
and irreducible with the Lebesgue measure as an irreducibility
measure. In addition, the zero-mean Gaussian distribution with
covariance matrix
$\Sigmabf_{\pi} = \sigma^2 \sum_{i=0}^{\infty} \Mbf^{i}
[\Mbf^{i}]^{\transpose}$ is the unique stationary distribution of $P$. Note that
$\Sigmabf_{\pi}$ is well defined since $\norm{\Mbf} < 1$. In addition,
using that $\norm{\Mbf} < 1$ again, an easy computation shows that 
there exists $c >0$ small enough such that setting
$w(x) = \exp(c \normr{x}^2)$, there exists $\lambda \in \ooint{0,1}$
and $\mathrm{b}_w \in \rset_+$ such that $P w(x) \leq \lambda w(x) +\mathrm{b}_w$ for any $x \in \rset^d$. Therefore \Cref{ass:pca:markov} holds for this example. 
In addition, $\Sigmabf_{\pi}$ satisfies the discrete Riccati equation: $\Sigmabf_{\pi} = \Mbf \Sigmabf_{\pi} \Mbf^{\transpose} + \sigma^2 \mathrm{I}_{50}$ which it is solved using the function \verb|dare| of the Python library \verb|control|\footnote{\url{https://python-control.readthedocs.io/en/0.8.3/intro.html}}. 



\section{Preliminaries on Riemannian Geometry}\label{app:bg_rie}
This section reviews some basic concepts about Riemannian geometry. These concepts are essential to develop our main results on convergence of Riemannian SA.

\subsection{Metric Tensor and Distance} \label{app:metricdistance} A smooth manifold (at least $C^2$) $\Theta$ is equipped with a smooth metric tensor field $g\in T^{2}\rmT^* \Theta$, see \cite[Proposition 2.4]{lee:2019}. To each $\theta \in \Theta$, this associates a scalar product $g_\theta$ on the tangent space $\rmT_\theta\Theta$. When there is no confusion, we denote  that \cite[Chapter 2, pages 11-12]{lee:2019},
\begin{equation} \label{eq:metric}
g_\theta(u,v) = \psr{u}{v}[\theta]\eqsp, \quad u,v \in \rmT_\theta\Theta
\end{equation}
and the corresponding norm on $\rmT_\theta\Theta$  is called the Riemannian norm, $\normr{u}[\theta]=\psr{u}{u}[\theta]^{1/2}$\,. 

With the metric tensor, it is possible to define the notion of length of a curve. If $c:I\to \Theta$ is a differentiable curve, defined on some interval $I \subset \rset$, with velocity $\dot{c}$, then its length is~\cite[page 34]{lee:2019}
\begin{equation} \label{eq:length}
  \ell(c) \,=\, \int_I\,\normr{\dot{c}(t)}[c(t)]\,\rmd t \eqsp.
\end{equation}
The length $\ell(c)$ is invariant by reparametrization\,: $\ell(c\,\circ\, \phi) = \ell(c)$ for any diffeomorphism $\phi:J\rightarrow I$, from an interval $J$ onto $I$. Thus, without loss of generality, we consider only curves that are restricted to $c:[0,1] \to \Theta$. 

This can be used to turn $\Theta$ into a metric space. Indeed, if $\Theta$ is connected, the following is a well-defined distance function; satisfying the axioms of a metric space~\cite[Theorem 2.55]{lee:2019},
\begin{equation} \label{eq:distance}
  \distT(\theta,\theta') \,=\, \inf \left\lbrace \ell(c) \, | \, c:[0,1]\rightarrow \Theta\,;\,c(0) = \theta\,,c(1) = \theta'\,\right\rbrace \eqsp, \quad \theta,\theta' \in \Theta \eqsp. 
\end{equation}
This is called the Riemannian distance induced by the metric tensor $g$. 

The infimum in \eqref{eq:distance} is always attained for any $\theta,\theta' \in \Theta$, provided that the distance $\distT(\cdot,\cdot)$ turns $\Theta$ into a complete metric space. This is a corollary of the Hopf-Rinow theorem, a fundamental theorem in Riemannian  geometry~\cite[Corollary 6.21]{lee:2019}.

\subsection{Levi-Civita Connection} \label{app:levicivita} 
In the Riemannian mainfold, the curve $\upgamma$ which attains the infimum in \eqref{eq:distance} is called a geodesic. Intuitively, a geodesic is a $C^2$ curve which has zero acceleration. This intuition can be formalized by introducing an affine connection $\nabla$ \cite[page 89]{lee:2019}, compatible with the metric tensor $g$, called the Levi-Civita connection, or just Riemannian connection (to be precise, $\nabla = \nabla^g$, depends on the choice of $g$). 

To each vector $u \in \rmT_\theta\Theta$ and smooth vector field $X$ on $\Theta$, the connection $\nabla$ associates a vector $\nabla_uX \in \rmT_\theta\Theta$. This vector is called the covariant derivative of $X$ in the direction of $u$. This is bilinear in $u$ and $X$, and satisfies the product rule
\begin{equation} \label{eq:product}
   \nabla_u(fX) = (uf)\,X(\theta) + f(\theta)\,\nabla_uX \eqsp,
\end{equation}
for any differentiable function $f:\Theta\rightarrow\rset$, where $uf$ denotes the derivative of $f$ along $u$. Moreover, one has the following,
\begin{equation} \label{eq:metriccon}
   u\,\psr{X}{Y} \,=\, \psr{\nabla_u X}{Y}[\theta] + \psr{X}{\nabla_u Y}[\theta]\eqsp,
\end{equation}
\begin{equation} \label{eq:torsion}
  \nabla_XY - \nabla_YX = [X,Y]\eqsp,
\end{equation}
for any differentiable vector fields $X,Y$ on $\Theta$, where $[X,Y]$ is the Lie bracket of the vector fields $X$ and $Y$, itself a vector field. Here, \eqref{eq:metriccon} states that $\nabla$ is compatible with the metric, and \eqref{eq:torsion} states that $\nabla$ is a connection with zero torsion.

The Levi-Civita connection~\cite[Theorem 5.10]{lee:2019} is defined as the unique affine connection $\nabla$ which satisfies \eqref{eq:metriccon} and \eqref{eq:torsion}. Note that the uniqueness of this connection can be guaranteed by the Koszul's theorem, also known as the fundamental theorem of Riemannian geometry. 

\subsection{Geodesic Equation} \label{app:geodesic} 
It can be proved using \eqref{eq:product} that $\nabla_uX$ depends only on the values of $X$ along a curve tangent to the vector $u$ \cite[Proposition 4.26]{lee:2019}. This motivates the following definition. Consider $c:I\rightarrow \Theta$ as a smooth curve on $\Theta$ and $X$ is an extendible vector field along $c$, this means that $X:I\to \rmT\Theta$ satisfies $X(t) \in \rmT_{c(t)}\Theta$ for any $t \in I$, see \cite[pages 100-101]{lee:2019}. The covariant derivative of $X$ along $c$ is defined by
\begin{equation} \label{eq:curvecovariant}
\rmD_{t}X = \nabla_{\dot{c}}\tilde{X}\circ c \eqsp,
\end{equation}
where $\tilde{X}$ is a vector field on $\Theta$ satisfying, for any $t \in I$, $X(t)=\tilde{X}\circ c (t)$. The reader should not confuse the index $t$ in $\rmD_t$, which is just a notation, with an actual real number $t \in I$.

A \emph{geodesic} is thus a smooth curve $\upgamma:I\rightarrow \Theta$, whose velocity $\dot{\upgamma}$ is parallel along $\upgamma$. If $\rmD_t$ is the covariant derivative along $\upgamma$, then $\upgamma$ satisfies the geodesic equation~\cite[page 103]{lee:2019},
\begin{equation} \label{eq:geodesicequation}
  \rmD_{t}\dot{\upgamma}(t) = 0\eqsp, \quad t \in I \eqsp.
\end{equation}
The left-hand side of this equation is precisely the acceleration of the curve $\upgamma$.

The geodesic equation is a non-linear ordinary differential equation of second order. For given initial conditions $\upgamma(0) = \theta$ and $\dot{\upgamma}(0) = u$, it has a unique solution $\upgamma : (-\varepsilon,\varepsilon) \to \Theta$, for some $\varepsilon > 0$ \cite[Theorem 4.27]{lee:2019}. If this solution can always be extended to a curve $\upgamma:\rset \to \Theta$, then $\Theta$ is called a complete Riemannian manifold. The Hopf-Rinow theorem states that this is equivalent to $\Theta$ being a complete metric space, with the distance function \eqref{eq:distance}~\cite[Theorem 6.19]{lee:2019}.

\subsection{Parallel Transport and Parallel Frames} \label{app:parallel}
Recall $\rmD_t$ the covariant derivative, associated with the Levi-Civita connection, along a curve $c:[0,1] \to \Theta$, given in \eqref{eq:curvecovariant}. Then, a vector field $X$ is said to be parallel along $c$ if it satisfies the parallel transport equation
\begin{equation}
\rmD_t X(t) = 0 \eqsp.
\end{equation}
This is a first-order linear ordinary differential equation (ODE). Say $c(0)=\theta_0$, then, for a given initial condition $X(0) = u$, where $u \in \rmT_{\theta_0}\Theta$, it follows that $u \mapsto X(t)$ is a linear mapping from  $\rmT_{\theta_0}\Theta$ to $\rmT_{c(t)}\Theta$ \cite[Theorem 4.32]{lee:2019}. This is denoted $\parallelTransport_{0t}^{c}$, and by uniqueness of the solution to the ODE, $\parallelTransport_{t0}^c$ is its linear inverse \cite[Equation (4.22)]{lee:2019}.  

It is useful to derive an equivalent condition to \eqref{eq:metriccon}, which holds for vector fields $X$ and $Y$ along $c$ \cite[Proposition 5.5]{lee:2019}.
\begin{equation} \label{eq:metricconbis}
 \frac{\rmd}{\rmd t}\,\psr{X}{Y}[c(t)] = \psr{\rmD_t X}{Y}[c(t)] + \psr{X}{\rmD_t Y}[c(t)] \eqsp.
\end{equation}
This equation yields that $t \mapsto \psrLigne{X}{Y}[c(t)]$ is constant if $X$ and $Y$ are parallel vector fields along $c$. Thus $\parallelTransport_{0t}^{c\,}$ preserves scalar products. In particular, if $(\bfb_i\,;i = 1,\ldots,d)$ is an orthonormal basis of $\rmT_{\theta_0}\Theta$, then the vector fields along $c$, defined by
\begin{equation}
e_i(t) = \parallelTransport_{0t}^{c}\,\bfb_i \eqsp,
\end{equation}
form an orthonormal basis of the tangent space $\rmT_{c(t)}\Theta$, for each $t \in I$. This is called a parallel orthonormal frame along $c$~\cite[Equation (4.23)]{lee:2019}. By linearity of $\parallelTransport_{0t}^{c\,}$, if $u \in T_{\theta_0}\Theta$ is written $u = \sum^d_{i=1}\,u^i\,\bfb_i$, then 
\begin{equation} \label{eq:parallelformula}
\parallelTransport_{0t}^{c\,} u \,=\, \sum^d_{i=1}\,u^i\,e_i(t) \eqsp.
\end{equation}
In other words, parallel transport is obtained by simply propagating the cooordinates $u^i$ of the vector $u$ along a parallel orthonormal frame.

\subsection{Riemannian Exponential Map and Cut Locus} \label{app:cutlocus} 
From now on, let us assume  that $\Theta$ is a complete Riemannian manifold. A curve that attains the infimum in \eqref{eq:distance} is called a length-minimizing curve. While this curve is not always unique, it is always a geodesic \cite[Theorem 6.4]{lee:2019}; in other words, it is a twice differentiable solution of the geodesic equation \eqref{eq:geodesicequation}. On the other hand, it is very important to keep in mind that a geodesic is not always a length-minimizing curve.

To give a concrete example, consider the geodesics of a sphere with its usual round metric \cite[Example 2.13]{lee:2019} which are simply its great circles, i.e., intersections of the sphere with planes passing through the origin. Clearly, a portion of a great circle whose length is greater than $\pi$ is not length-minimizing. 
Therefore, geodesics which start at some point $\theta$ on a sphere, are length-minimizing until they reach the opposite point $-\theta$. One says that the cut locus of the point $\theta$ on the sphere is the set $\lbrace -\theta \rbrace$. 

Since $\Theta$ is complete, for $u \in \rmT_\theta \Theta$, there exists a unique geodesic $\upgamma_u:\rset \to \Theta$ with $\upgamma_u(0) = \theta$ and $\dot{\upgamma}_u(0)=u$. Then, \cite[page 128]{lee:2019} define the \emph{exponential map} $\Exp:\rmT\Theta \to \Theta$ as:
\begin{equation} \label{eq:exp}
  \Exp_\theta(u) = \upgamma_u(1) \eqsp,
\end{equation}
which is a smooth map.
For $u \in \rmT_\theta \Theta$, with $\normr{u}[\theta] = 1$, let $c(u) > 0$ be the largest positive number $t$ such that $\upgamma_u$ is length-minimizing when restricted to the interval $[0,t]$ \cite[page 307]{lee:2019}:
\begin{equation} \label{eq:cu1}
c(u) = \sup\left\lbrace t\geq 0 \, : \, \distT(\theta,\upgamma_u(t)) = t\right\rbrace \eqsp.
\end{equation}
The tangent cut locus of $\theta$ is then defined by
\begin{equation} \label{eq:tangentcut}
\rmT\Cut(\theta) \,=\, \left\lbrace c(u)u\, : \, u\in \rmT_\theta\Theta\,;\,\normr{u}[\theta] = 1\right\rbrace \eqsp.
\end{equation}
Finally, the cut locus of $\theta$ is given by~\cite[page 308]{lee:2019}
\begin{equation} \label{eq:cut}
  \Cut(\theta) \,=\, \Exp_\theta\defEns{\rmT\Cut(\theta)} \eqsp.
\end{equation}
For example, if $\Theta$ is a unit sphere, then $\rmT\Cut(\theta)$ is the set of tangent vectors $u \in \rmT_\theta\Theta$ such that $\normr{u}[\theta] = \pi$. On the other hand, $\Cut(\theta) = \lbrace -\theta \rbrace$, since $\normr{u}[\theta] = \pi$ implies $\Exp_\theta(u) = -\theta$.

\subsection{Injectivity Domain} \label{app:injectivity} 
The cut locus of a point in a complete Riemannian manifold gives valuable information regarding the topology of the manifold. Indeed, if $\Theta$ is a complete Riemannian manifold and $\theta$ is any point in $\Theta$, then $\Theta$ can be decomposed into the disjoint union of two sets
\begin{equation} \label{eq:DCut}
   \Theta = \mathrm{D}(\theta) \,\cup\,\Cut(\theta) \eqsp,
\end{equation} 
where $\mathrm{D}(\theta)$ is the injectivity domain,
\begin{equation}
\mathrm{D}(\theta) \,=\, \Exp_\theta\defEns{\rmT \mathrm{D}(\theta)} \eqsp, \quad \text{where } \rmT\mathrm{D}(\theta) = \defEns{ tu \, : \, u \in \rmT_\theta \Theta\,;\,\normr{u}[\theta] = 1 \text{ and } 0\leq t < c(u)} \eqsp ,
\end{equation}
where $c(u)$ is given by \eqref{eq:cu1}. We observe:
\begin{pproposition} \label{prop:injectivity} \cite[Theorem 10.34]{lee:2019}
The Riemannian exponential map $\Exp_\theta$ is a diffeomorphism of $\rmT\mathrm{D}(\theta)$ onto $\mathrm{D}(\theta)$. Therefore, the inverse of the Riemannian exponential $\Exp^{-1}_\theta$ is well-defined, and a diffeomorphism, on $\mathrm{D}(\theta) = \Theta - \Cut(\theta)$.
\end{pproposition} 
In fact, $\rmT\mathrm{D}(\theta)$ is an open, star-shaped subset of the tangent space $\rmT_\theta\Theta$, so it has the topology of an open ball. Thus, \eqref{eq:DCut} states that the topology of $\Theta$ is completely determined by $\Cut(\theta)$. This theorem also ensures that $\Cut(\theta)$ is a closed set of measure zero.

\subsection{Isometries and Homogeneous Spaces} \label{app:isometry} 
An isometry $g$ on the Riemannian manifold $\Theta$ is a diffeomorphism $g:\Theta\to \Theta$ which preserves the Riemannian metric. To express this, let $g\cdot \theta = g(\theta)$, and $g\cdot u = \rmD g_\theta(u)$ for each $\theta \in \Theta$ and $u \in \rmT_\theta\Theta$.  Here, $g(\theta) \in \Theta$ is simply the image of $\theta$ under the map $g$, and $\rmD g_\theta$ denotes the derivative of $g$ at $\theta$, so $Dg_\theta : \rmT_\theta\Theta \to \rmT_{g\cdot \theta}\Theta$. We say that $g$ is an isometry if \cite[page 12]{lee:2019}
\begin{equation} \label{eq:isometry}
  \psr{g\cdot u}{g\cdot v}[g\cdot\theta] = \psr{u}{v}[\theta] \hspace{1cm} u\,,\,v \in \rmT_\theta\Theta \eqsp.
\end{equation}
In other words, the linear map $u \mapsto g\cdot u$ from $\rmT_\theta\Theta$ to $\rmT_{g\cdot\theta}\Theta$ preserves scalar products. In particular, it also preserves norms, so $\normr{g\cdot u}[g\cdot \theta] = \normr{u}[\theta]$.

Isometries also preserve objects derived from the Riemannian metric such as distance, geodesics, among others. In particular, if $\upgamma:I\rightarrow \Theta$ is a geodesic, and $g:\Theta \rightarrow \Theta$ is an isometry, then $\upgamma' = g\circ \upgamma : I\rightarrow \Theta$ is also a geodesic~\cite[Corollary 5.14]{lee:2019}. Now, if $\upgamma(0) = \theta$ and $\dot{\upgamma}(0) = u$, then $\upgamma'(0) = g\cdot \theta$ and $\dot{\upgamma}'(0) = g\cdot u$. From the definition of the Riemannian exponential \eqref{eq:exp}, it is seen that
\begin{equation} \label{eq:isomexp}
  g\cdot \Exp_\theta(u) = \Exp_{g\cdot\theta}(g\cdot u) \eqsp.
\end{equation}
The set $\msg$ of all isometries of a Riemannian manifold $\Theta$ forms a group under composition. A deep theorem, called Myers-Steenrod theorem, states that $\msg$ can always be given the structure of a Lie group, such that for each $\theta \in \Theta$, the group action $g\mapsto g\cdot \theta$ is a differentiable map~\cite[page 66]{ghl:2004}.
 
One calls $\Theta$ a Riemannian homogeneous space if its group of isometries $\msg$ acts transitively. Transitive action means that for any $\theta,\theta' \in \Theta$ there exists $g \in \msg$ such that $g\cdot\theta = \theta'$. When $\Theta$ is a Riemannian homogeneous space, knowing the metric of $\Theta$ at just one point, $o \in \Theta$, is enough to know this metric anywhere~\cite[page 67]{ghl:2004}. 

\subsection{Riemannian Gradient, Hessian, and Taylor Formula} \label{app:grad_hess} 
The metric tensor $\psr{\cdot}{\cdot}[\cdot]$ and Levi-Civita connection $\nabla$, on the Riemannian manifold $\Theta$,  can be used to generalize classical objects from analysis, like the gradient and Hessian of a $C^2$ function $V:\Theta\rightarrow \rset$, as we introduce next.

The Riemannian gradient of $V$ is a vector field $\grad V$ on $\Theta$, uniquely defined by the property~\cite[Equation 2.14]{lee:2019}
\begin{equation} \label{eq:gradient}
  \psr{\grad V}{u}[\theta] \,=\, \rmD V(\theta)(u) \eqsp, \quad u \in T_\theta\Theta \eqsp,
\end{equation}
where $\rmD V(\theta) : \rmT_\theta\Theta\to \rset$ is the differential of the function $V$ at $\theta$. As $V$ is a real-valued function, it is useful to know that differentials, directional derivatives and covariant derivatives coincide $\rmD V(\theta) (u)=u \, V (\theta) = \nabla_u V (\theta) $. This definition makes it clear that the Riemannian gradient $\grad V$ depends on the choice of metric on the manifold $\Theta$, and does not arise from the manifold structure of $\Theta$, in itself.

The Riemannian Hessian of $V$, denoted $\Hess V$ is defined using the Levi-Civita connection. Precisely, it is the covariant derivative of the gradient $\grad V$. For $\theta \in \Theta$, this gives the Hessian $\Hess V(\theta):\rmT_\theta\Theta \to \rmT_\theta\Theta$,
\begin{equation} \label{eq:hessian}
  \Hess V(\theta)\, u \,=\, \nabla_u \, \grad V(\theta) \eqsp.
\end{equation}
The Riemannian definition of the Hessian coincides with the covariant Hessian obtained from the Levi-Civita connection \cite[Example 4.22]{lee:2019}. The covariant characterization gives, for any vector fields $X,Y$ on $\Theta$,
\begin{equation}\label{eq:cov_hessian}
\psr{\Hess V X}{Y} = \nabla_X \nabla_Y V - \nabla_{\nabla_X Y} V = X(Y \,V) - (\nabla_X Y) V\eqsp,
\end{equation}
where the last equality comes from the remark regarding directional derivatives. One can see that \eqref{eq:torsion} yields that the linear operator $\Hess V(\theta):\rmT_\theta\Theta \to \rmT_\theta\Theta$ is self-adjoint with respect to the Riemannian scalar product $\psr{\cdot}{\cdot}[\theta]$. Therefore, and not without a slight abuse of notation, we will also call $\Hess$ the resulting symmetric bilinear form
\begin{equation}
\Hess V(\theta) (u,v)= \psr{\Hess V(\theta) u}{v}[\theta] \eqsp, \quad u,v\in \rmT_\theta \Theta.
\end{equation}

Using the gradient \eqref{eq:gradient} and the Hessian \eqref{eq:hessian}, one can derive the following Taylor formula for the function $V$. If $\upgamma:[0,1]\rightarrow \Theta$ is a geodesic such that $\upgamma(0) = \theta_0$ and $\upgamma(1) = \theta_1$; which we simply call a geodesic between $\theta_0$ and $\theta_1$; then we have
\begin{equation} \label{eq:taylor}
 V(\theta_1) - V(\theta_0) = \psr{\grad V(\theta_0)}{\dot{\upgamma}(0)}[\theta_0] + \Hess V(\upgamma(t_*))\,(\dot{\upgamma},\dot{\upgamma}) /2 \eqsp,
\end{equation} 
for some $t_* \in (0,1)$.
\subsection{Bounded Hessian Implies Lipschitz Gradient} \label{app:assumptionlemma} 
We can now state a result that can be very useful when the Riemannian Hessian is bounded.
\begin{llemma} \label{lem:bounded_hessian}
 If $V$ has a continuous Riemannian Hessian $\Hess V(\theta):\rmT_\theta\Theta \to \rmT_\theta\Theta$, with operator norm upper bounded uniformly for  $\theta \in \Theta$ by $C\geq 0$; then the Riemannian gradient $\grad V$ satisfies the Lipschitz property \eqref{eq:llipschitz} with Lipschitz constant $L = C$. 
\end{llemma}
\begin{proof}
 Let $(e_i;i=1,\ldots,d)$ be a parallel orthonormal frame along $\upgamma$. Define $\nabla V^i:[0,1]\to \rmT_\theta\Theta$ by
\begin{equation} \label{eq:vii}
\nabla V^i(t) = \psr{\grad V(\upgamma(t))}{e_i(t)}[\upgamma(t)]\eqsp, \quad t \in [0,1] \eqsp.
\end{equation}
Also, note from \eqref{eq:parallelformula}, applied to $c = \upgamma$ and $u = \grad V(\theta_0)$, that
\begin{equation}
\parallelTransport_{01}^{\upgamma} \grad V(\theta_0) \,=\, \sum^d_{i=1}\,\nabla V^i(0)\,e_i(1) \eqsp.
\end{equation}
Then, since $(e_i(1);i=1,\ldots,d)$ is an orthonormal basis of $\rmT_{\theta_1}\Theta$,
\begin{equation} \label{eq:normtofinitedif}
\normr{\grad V(\theta_1) - \parallelTransport_{01}^{\upgamma} \grad V(\theta_0)}[\theta_1]^2 \,=\,
\sum^d_{i=1}\,(\nabla V^i(1) - \nabla V^i(0))^2
\end{equation}
But, by applying \eqref{eq:metricconbis} to $c = \upgamma$ with $X = \grad V \circ \upgamma$ and $Y = e_i\,$, it follows from \eqref{eq:curvecovariant}, \eqref{eq:hessian} and \eqref{eq:vii},
\begin{equation} \label{eq:dvdt}
\frac{\rmd}{\rmd t}\nabla V^i(t) \,=\, \psr{\Hess V(\upgamma(t)) \dot{\upgamma}(t)}{e_i(t)}[\upgamma(t)] + 
\psr{\grad V(\upgamma(t))}{\rmD_te_i(t)}[\upgamma(t)] \eqsp.
\end{equation}
Then, since each $e_i(t)$ is parallel along $\upgamma$, $\rmD_t e_i(t) = 0$. Plugging \eqref{eq:dvdt} into \eqref{eq:normtofinitedif} yields, by the mean-value theorem, followed by Jensen's inequality,
\begin{align}\label{eq:jensen}
\normr{\grad V(\theta_1) - \parallelTransport_{01}^{\upgamma} \grad V(\theta_0)}[\theta_1] &= \parenthese{\sum_{i=1}^d \parentheseDeux{\int_0^1 \psr{\Hess V(\upgamma(t)) \dot{\upgamma}(t)}{e_i(t)}[\upgamma(t)] \rmd t}^2}^{1/2} \eqsp, \\
& \leq \int_0^1 \parenthese{\sum_{i=1}^d \parentheseDeux{\psr{\Hess V(\upgamma(t)) \dot{\upgamma}(t)}{e_i(t)}[\upgamma(t)] }^2 }^{1/2}   \rmd t \eqsp, \\
& = \int_0^1 \normr{\Hess V (\upgamma(t))\dot{\upgamma}(t)}[\upgamma(t)] \rmd t \eqsp.
\end{align}
The last equality comes from the fact that $(e_i(t)\, ;\, i=1,\ldots,d)$ is an orthonormal basis of $\rmT_{\upgamma(t)} \Theta$. Finally, using the definition of the operator norm $\Vert \cdot \Vert_{op,\upgamma(t)}$ on $\rmT_{\upgamma(t)} \Theta$, the bound in the assumption and \eqref{eq:length},
\begin{equation}\label{eq:opnormhess}
\normr{\grad V(\theta_1) - \parallelTransport_{01}^{\upgamma} \grad V(\theta_0)}[\theta_1] \leq \int_0^1 \Vert \Hess V (\upgamma(t)) \Vert_{op,\upgamma(t)} \normr{\dot{\upgamma}(t)}[\upgamma(t)] \rmd t \leq N \ell(\upgamma) \eqsp .
\end{equation}
This concludes the proof.
\end{proof}







\end{document}